\documentclass{elsarticle}

\usepackage{amssymb,amsmath,amsthm}
\usepackage{booktabs}
\usepackage{bbm}
\usepackage{graphicx}

\newlength{\minipagewidth}
\newcommand{\bookbox}[1]{\small
\par\medskip\noindent
\framebox[\columnwidth]{
\begin{minipage}{\minipagewidth} {#1} \end{minipage} } \par\medskip }

\newtheorem{theo}{Theorem}[section]
\newtheorem{prop}[theo]{Proposition}
\newtheorem{lemm}[theo]{Lemma}
\newtheorem{defi}{Definition}
\newtheorem{coro}[theo]{Corollary}
\newcommand{\Pro}{\mathbb{P}}
\newcommand{\N}{\mathbb{N}}
\newcommand{\R}{\mathbb{R}}
\newcommand{\tth}{\tilde{\theta}}
\newcommand{\tnu}{\tilde{\nu}}
\newcommand{\tmu}{\tilde{\mu}}

\newcommand\set[2]{\{#1,\dots,#2\}}

\newcommand{\IND}{\mathbbm{1}}
\newcommand{\hX}{\hat{X}}
\newcommand{\hR}{\hat{R}}

\renewcommand{\P}{\mathbb{P}}
\newcommand\E{\mathbb{E}}
\newcommand{\ra}{\rightarrow}
\newcommand\tC{\tilde{C}}

\newcommand\fracl[2]{{(#1)}/{#2}}
\newcommand\fracc[2]{{#1}/{#2}}
\newcommand\fracr[2]{{#1}/{(#2)}}

\newcommand{\argmax}{\mathop{\mathrm{argmax}}}

\def\beglab{\begin{equation} \label}
\def\endlab{\end{equation}}

\newcommand\und[2]{\underset{#2}{#1}\;}

\newcommand\F{\mathcal{F}}

\newsavebox{\fmbox}

\begin{document}

\begin{frontmatter}

\title{Robustness of anytime bandit policies}

\author[imagine]{Antoine Salomon}
\ead{salomona@imagine.enpc.fr}

\author[imagine,sierra]{Jean-Yves Audibert}
\ead{audibert@imagine.enpc.fr}

\address[imagine]{Imagine, LIGM \\
\'Ecole des Ponts ParisTech\\
Universit\'e Paris Est}
\address[sierra]{Sierra, CNRS/ENS/INRIA, Paris, France}

\begin{abstract}
This paper studies the deviations of the regret in a stochastic multi-armed bandit problem.
When the total number of plays $n$ is known beforehand by the agent, \citet{audibert2009exploration} exhibit
a policy such that with probability at least $1-1/n$, the regret of the policy is 
of order $\log n$. They have also shown that such a property is not shared by the popular {\sc ucb1} policy
of \citet{AuCeFi02}. This work first answers an open question: it extends this negative result to any anytime policy.
Another contribution of this paper is to design anytime robust policies for specific multi-armed bandit problems in which
some restrictions are put on the set of possible distributions of the different arms. We also show that, for any policy (i.e. when the number of plays is known), the regret is of order $\log n$ with probability at least $1-1/n$, so that the policy of \citeauthor{audibert2009exploration} has the best possible deviation properties. 
\end{abstract}


\end{frontmatter}

\section{Introduction}

Bandit problems illustrate the fundamental difficulty of sequential decision making in the
face of uncertainty: a decision maker must choose between following what seems
to be the best choice in view of the past (``exploitation'') or testing (``exploration'')
some alternative, hoping to discover a choice that beats the current empirical best choice.
More precisely, in the stochastic multi-armed bandit problem, at each stage, an agent (or
decision maker) chooses one action (or arm), and receives a reward from it. The
agent aims at maximizing his rewards. Since he does not know the process generating
the rewards, he does not know the best arm, that is the one having the highest expected reward. 
He thus incurs a regret, that is the difference between the cumulative reward he would have got by always drawing the best arm
and the cumulative reward he actually got. The name ``bandit'' comes from imagining a gambler in a casino playing with $K$ slot machines, where at each round, the gambler pulls the arm of any of the machines and gets a payoff as a result.

The multi-armed bandit problem is the simplest setting where one encounters the 
exploration-exploitation dilemma. It has a wide range of applications
including advertisement \citep{Baba09,Dev09}, economics \citep{BerVal08,LamPagTar04}, games \citep{Gel06} and
optimization \citep{Kle05,Coq07,Kle08,Bub08}.
It can be a central building block of larger systems, like in evolutionary programming \citep{Hol92} and reinforcement learning \citep{Sut98}, in particular in large state space Markovian Decision Problems \citep{KocSze06}. Most of these applications require
that the policy of the forecaster works well {\it for any time}. For instance, in tree search using bandit policies at each node,
the number of times the bandit policy will be applied at each node is not known beforehand (except for the root node in some cases),
and the bandit policy should thus provide consistently low regret whatever the total number of rounds is.

Most previous works on the stochastic multi-armed bandit 
\cite[][among others]{ro52,LaiRo85,Agr95,AuCeFi02} focused on the expected regret, and showed that after $n$ rounds, 
the expected regret is of order $\log n$.
So far, the analysis of the upper tail of the regret was only addressed in \citet{audibert2009exploration}.
The two main results there about the deviations of the regret are the following. First, after $n$ rounds, for large enough constant $C>0$, the probability that the regret of {\sc ucb1} (and also its variant taking into account the empirical variance) exceeds $C\log n$
is upper bounded by $1/(\log n)^{C'}$ for some constant $C'$ depending on the distributions of the arms and on $C$ (but not on $n$). Besides, for most bandit problems, this upper bound is tight to the extent that the probability is also lower bounded by a quantity of the same form.
Second, a new upper confidence bound policy was proposed: 
it requires to know the total number of rounds in advance and uses this knowledge to design
a policy which essentially explores in the first rounds and then exploits the information gathered in the exploration phase.
Its regret has the advantage of being more concentrated to the extent that
with probability at least $1-1/n$, the regret is of order $\log n$. 
The problem left open by \cite{audibert2009exploration} is whether it is possible to design an anytime robust policy, that is a policy such that for any $n$, with probability at least $1-1/n$, the regret is of order $\log n$.
In this paper, we answer negatively to this question when the reward distributions of all arms are just assumed to be uniformly bounded, say all rewards are in $[0,1]$ for instance (Corollary \ref{cor:main}). We then study which kind of restrictions on the set of probabilities defining the bandit problem allows to answer positively. One of our positive results is the following: if the agent knows the value of the expected reward of the best arm (but does not know which arm is the best one), the agent can use this information to design an anytime robust policy (Theorem \ref{th:mu*}). We also show that it is not possible to design a policy such that the regret is of order $\log n$ with a probability that would significantly greater than $1-1/n$, even if the agent knows the total number of rounds in advance (Corollary \ref{cor:hor}).\\
The paper is organised as follows: in the first section, we formally describe the problem we address and give the corresponding definitions and properties. Next we present our main impossibility result. In the third section, we provide restrictions under which it is possible to design anytime robust policies. In the fourth section, we study the robustness of policies that can use the knowledge of the total number of rounds. Then we provide experiments to compare our robust policy to the classical {\sc UCB} algorithms. The last section is devoted to the proofs of our results.


\section{Problem setup and definitions}

In the stochastic multi-armed bandit problem with $K\ge 2$ arms, at each time step $t=1,2,\dots$, an
agent has to choose an arm $I_t$ in the set $\{1,\dots,K\}$ and obtains a reward drawn from $\nu_{I_t}$ independently from the past (actions
and observations). The environment is thus parameterized by a $K$-tuple of probability distributions
$\theta=(\nu_1,\dots,\nu_K)$.
The agent aims at maximizing his rewards. He does not know $\theta$ but knows that it belongs to some set $\Theta$.
We assume for simplicity that $\Theta \subset \bar\Theta$, where $\bar\Theta$ denotes the set of all $K$-tuple of probability distributions on $[0,1]$. We thus assume that the rewards are in $[0,1]$. 

For each arm $k$ and all times $t \geq 1$, let $T_{k}(t)=\sum_{s=1}^t \IND_{I_s=k}$ denote the number of times arm $k$ was pulled from round $1$ to round $t$, and $X_{k,1}, X_{k,2},\ldots, X_{k,T_{k}(t)}$ the sequence of associated rewards. 
For an environment parameterized by $\theta=\big(\nu_1,\dots,\nu_K)$, let $\P_\theta$ denote the distribution on the probability space such that
for any $k\in\set{1}{K}$, the random variables $X_{k,1}, X_{k,2}, \dots$ are i.i.d. realizations of $\nu_k$, and such that these
$K$ infinite sequence of random variables are independent. Let $\E_\theta$ denote the associated expectation.

Let $\mu_k=\int x d\nu_k(x)$ be the mean reward of arm $k$. Introduce $\mu^*=\max_k \mu_k$ and fix an arm $k^*\in\argmax_{k\in\set{1}{K}} \mu_k$, that is $k^*$ has the best expected reward.
The suboptimality of arm $k$ is measured by $\Delta_k = \mu^*-\mu_k$. The agent aims at minimizing its regret defined as the difference between the cumulative reward he would have got by always drawing the best arm
and the cumulative reward he actually got. At time $n\ge 1$, its regret is thus
\begin{equation}    
    \hR_n=\sum_{t=1}^{n} X_{k^*,t}-\sum_{t=1}^{n} X_{I_t,T_{I_t}(t)}. \label{def:defreg}
\end{equation}
The expectation of this regret has a simple expression in terms of the suboptimalities of the arms and the expected sampling times of the arms at time $n$. Precisely, we have
    \begin{eqnarray} 
    \E_\theta \hR_n &=& n \mu^* - \sum_{t=1}^n \E_{\theta}(\mu_{I_t}) =  n \mu^* - \E_{\theta}\bigg( \sum_{k=1}^K T_k(n) \mu_k\bigg) \nonumber \\ &=& \mu^* \sum_{k=1}^K \E_{\theta}[T_k(n)] - \sum_{k=1}^K \mu_k \E_{\theta}[T_k(n)]
     =\sum_{k=1}^{K} \Delta_k \E_{\theta}[T_k(n)]. \nonumber
    \end{eqnarray}
Other notions of regret exists in the literature: the quantity $\sum_{k=1}^{K} \Delta_kT_k(n)$ is called the pseudo regret and may be more practical to study, and the quantity $\max_k\sum_{t=1}^{n} X_{k,t}-\sum_{t=1}^{n} X_{I_t,T_{I_t}(t)}$ defines the regret in adverserial settings. Results and ideas we want to convey here are more suited to definition \eqref{def:defreg}, and taking another definition of the regret would only bring some more technical intricacies.\\

Our main interest is the study of the \emph{deviations} of the regret $\hR_n$, i.e. the value of $\Pro_{\theta}(\hR_n\geq x)$ when $x$ is larger and of order of $\E_\theta \hR_n$. If a policy has small deviations, it means that the regret is small with high probability and in particular, if the policy is used on some real data, it is very likely to be small on this specific dataset.
Naturally, small deviations imply small expected regret since we have
$$\E_\theta \hR_n\le\E_\theta \max(\hR_n,0)=\int_0^{+\infty}\Pro_{\theta}\left(\hR_n\geq x\right)dx.$$
To a lesser extent it is also interesting to study the deviations of the sampling times $T_n(k)$, as this shows the ability of a policy to match the best arm. Moreover our analysis is mostly based on results on the deviations of the sampling times, which then enables to
derive results on the regret. We thus define below the notion of being $f$-upper tailed for both quantities.\\
Define $\R_+^*=\{x\in\R:x>0\}$, and let $\Delta= \min_{k\neq k^*} \Delta_k$ be the gap between the best arm and second best arm.
\begin{defi}[$f$-$\mathcal T$ and $f$-$\mathcal R$] 
Consider a mapping $f:\R\ra\R_+^*$. 
A policy has $f$-upper tailed sampling Times (in short, we will say that the policy is $f$-$\mathcal T$) if and only if
\begin{align*}
 \exists C,\tilde{C}>0, \ \forall \theta\in\Theta \textit{ such that } & \Delta\neq 0, \\
 & \forall n\geq 2, \ \forall k\neq k^*, \ \Pro_{\theta}\left(T_k(n)\geq C\frac{\log n}{\Delta_k^2}\right)\leq \frac{\tilde{C}}{f(n)}.
\end{align*}
A policy has $f$-upper tailed Regret (in short, $f$-$\mathcal R$) if and only if
\begin{eqnarray*}
 \exists C,\tilde{C}>0, \ \forall \theta\in\Theta \textit{ such that } \Delta\neq 0, \ \forall n\geq 2, \ \Pro_{\theta}\left(\hR_n\geq C\frac{\log n}{\Delta}\right)\leq \frac{\tilde{C}}{f(n)}.
\end{eqnarray*}
\end{defi}

We will sometimes prefer to denote $f(n)$-$\mathcal T$ (resp. $f(n)$-$\mathcal R$) instead of $f$-$\mathcal T$ (resp. $f$-$\mathcal R$) for readability. Note also that, for sake of simplicity, we leave aside the degenerated case of $\Delta$ being null (i.e. when there are at least two optimal arms).


In this definition, we considered that the number $K$ of arms is fixed, meaning that $C$ and $\tC$ may depend on $K$.
The thresholds considered on $T_k(n)$ and $\hR_n$ directly come from known tight upper bounds 
on the expectation of these quantities for several policies. To illustrate this, let us recall the definition and
properties of the popular {\sc ucb1} policy. 
Let $\hat{X}_{k,s}=\frac{1}{s} \sum_{t=1}^s X_{k,t}$ be the empirical mean of arm $k$ after $s$ pulls.
In {\sc ucb1}, the agent plays each arm once, and then (from $t\ge K+1$), he plays 
\begin{equation}\label{eq:ucb}
I_t\in\argmax_{k\in\set{1}{K}} \Bigg\{ \hat{X}_{k,T_k(t-1)} + \sqrt{\frac{2\log t}{T_k(t-1)}} \Bigg\}.
\end{equation}
While the first term in the bracket ensures the exploitation of the knowledge gathered during steps $1$ to $t-1$,
the second one ensures the exploration of the less sampled arms. For this policy, \citet{AuCeFi02} proved:
    $$
    \forall n \ge 3, \quad \E[T_k(n)]\le 12 \frac{\log n}{\Delta_k^2} \quad \text{and} \quad 
    \E_{\theta} \hR_n \le 12 \sum_{k=1}^K \frac{\log n}{\Delta_k} \le 12 K \frac{\log n}{\Delta}.
    $$
\citet{LaiRo85} showed that these results cannot be improved up to numerical constants. 
\citet{audibert2009exploration} proved that {\sc ucb1} is $\log^3$-$\mathcal T$ and $\log^3$-$\mathcal R$ where $\log^3$ is the function $x\mapsto[\log(x)]^3$.
Besides, they also study the case when $2 \log t$ is replaced by $\rho \log t$ in \eqref{eq:ucb} with $\rho>0$,
and proved that this modified {\sc ucb1} is $\log^{2 \rho -1}$-$\mathcal T$ and $\log^{2 \rho -1}$-$\mathcal R$ for $\rho>1/2$,
and that $\rho=\frac12$ is actually a critical value. Indeed, for $\rho<\fracc12$ the policy does
not even have a logarithmic regret guarantee in expectation.
Another variant of {\sc ucb1} proposed by \citeauthor{audibert2009exploration} is to replace 
$2 \log t$ by $2 \log n$ in \eqref{eq:ucb} when we want to have low and concentrated regret at a fixed given time $n$.
We refer to it as {\sc ucb-h} as its implementation requires the knowledge of the \underline{h}orizon $n$ of the game.
The behaviour of {\sc ucb-h} on the time interval $[1,n]$ is significantly different to the one of {\sc ucb1}, as 
{\sc ucb-h} will explore much more at the beginning of the interval, and thus avoids exploiting the suboptimal arms
on the early rounds. \citeauthor{audibert2009exploration} showed that {\sc ucb-h} is $n$-$\mathcal T$ and $n$-$\mathcal R$ (as it will be recalled in Theorem \ref{th:ucbh}). As it will be confirmed by our results, whether a policy knows in advance the horizon $n$ or not matters a lot, that is why we introduce the following terms.

\begin{defi}
 A policy that uses the knowledge of the horizon $n$ (e.g. {\sc ucb-h}) is a horizon policy. A policy that does not use the knowledge of $n$ (e.g. {\sc ucb1}) is an anytime policy.
\end{defi}

We now introduce the weak notion of $f$-upper tailed as 
this notion will be used to get our strongest impossibility results.

\begin{defi}[$f$-w$\mathcal T$ and $f$-w$\mathcal R$] 
Consider a mapping $f:\R\ra\R_+^*$. 
A policy has weak $f$-upper tailed sampling Times (in short, we will say that the policy is $f$-w$\mathcal T$) if and only if
\begin{align*}
 \forall \theta\in\Theta \textit{ such that } & \Delta\neq 0, \\
 & \exists C,\tilde{C}>0, \ \forall n\geq 2, \ \forall k\neq k^*, \ \Pro_{\theta}\left(T_k(n)\geq C\frac{\log n}{\Delta_k^2}\right)\leq \frac{\tilde{C}}{f(n)}.
\end{align*}
A policy has weak $f$-upper tailed Regret (in short, $f$-w$\mathcal R$) if and only if
\begin{eqnarray*}
 \forall \theta\in\Theta \textit{ such that } \Delta\neq 0, \ \exists C,\tilde{C}>0, \ \forall n\geq 2, \ \Pro_{\theta}\left(\hR_n\geq C\frac{\log n}{\Delta}\right)\leq \frac{\tilde{C}}{f(n)}.
\end{eqnarray*}
\end{defi}

The only difference between $f$-$\mathcal T$ and $f$-w$\mathcal T$ (and between $f$-$\mathcal R$ and $f$-w$\mathcal R$) is
the interchange of ``$\forall \theta$'' and ``$\exists C,\tilde{C}$''.
Consequently, a policy that is $f$-$\mathcal T$ (respectively $f$-$\mathcal R$) is $f$-$\mathcal T$ (respectively $f$-w$\mathcal R$).
Let us detail the links between the $f$-$\mathcal T$, $f$-$\mathcal R$, $f$-w$\mathcal T$ and $f$-w$\mathcal R$. \\

\begin{prop}\label{prop:eqRT}
Assume that there exists $\alpha,\beta>0$ such that $f(n)\leq \alpha n^{\beta}$ for any $n\geq 2$. We have
    $$
    f\text{-}\mathcal T \Rightarrow f\text{-}\mathcal R \Rightarrow f\text{-w}\mathcal R \Leftrightarrow f\text{-w}\mathcal T.
    $$
\end{prop}

The proof of this proposition is technical but rather straightforward. Note that we do not have $f\text{-}\mathcal R \Rightarrow f\text{-}\mathcal T$, because the agent may not regret having pulled a suboptimal arm if the latter has delivered good rewards. Note also that $f$ is required to be at most polynomial: if not some rare events such as unlikely deviations of rewards towards their actual mean can not be neglected, and none of the implications hold in general (except, of course, $f\text{-}\mathcal R \Rightarrow f\text{-w}\mathcal R$ and $f\text{-}\mathcal T \Rightarrow f\text{-w}\mathcal T$).



\section{Impossibility result} 

Here and in section \ref{sec:pos} we mostly deal with anytime policies, and the word policy (or algorithm) implicitly refers to anytime policy.\\
In the previous section, we have mentioned that for any $\rho>1/2$, there is a variant of {\sc ucb1} (obtained by changing $2 \log t$ into $\rho\log t$ in \eqref{eq:ucb}) which is $\log^{2\rho-1}$-$\mathcal T$. This means that, for any $\alpha>0$, there exists a $\log^{\alpha}$-$\mathcal T$ policy, and a hence $\log^{\alpha}$-$\mathcal R$ policy. The following result shows that it is impossible to find an algorithm that would have better deviation properties than these {\sc ucb} policies. For many usual settings (e.g., when $\Theta$ is the set $\bar\Theta$ of all $K$-tuples of measures on $[0,1]$), with not so small probability, the agent gets stuck drawing a suboptimal arm he believes best. Precisely, this situation arises when simultaneously:
\begin{itemize}
\item[(a)] an arm $k$ delivers payoffs according to a same distribution $\nu_k$ in two distinct environments $\theta$ and $\tth$,
\item[(b)] arm $k$ is optimal in $\theta$ but suboptimal in $\tth$,
\item[(c)] in environment $\tth$, other arms may behave as in environment $\theta$, i.e. with positive probability other arms deliver payoffs that are likely in both environments.
\end{itemize}

If the agent suspects that arm $k$ delivers payoffs according to $\nu_k$, he does not know if he has to pull arm $k$ again (in case the environment is $\theta$) or to pull the optimal arm of $\tth$. The other arms can help to point out the difference between $\theta$ and $\tth$, but then they have to be chosen often enough.
 This is in fact this kind of situation that has to be taken into account when balancing a policy between exploitation and exploration.\\

Our main result is the formalization of the leads given above. In particular, we give a rigorous description of conditions (a), (b) and (c). Let us first recall the following results, which are needed in the formalization of condition (c). One may look at \cite{rudin}, p.121 for details (among others). Those who are not familiar with measure theory can skip to the non-formal explanation just after the results.
\begin{theo}[Lebesgue-Radon-Nikodym theorem]
Let $\mu_1$ and $\mu_2$ be $\sigma$-finite measures on a given measurable space. There exists a $\mu_2$-integrable function $\frac{d\mu_1}{d\mu_2}$ and a $\sigma$-finite measure $m$ such that $m$ and $\mu_2$ are singular\footnote{Two measures $m_1$ and $m_2$ on a measurable space $(\Omega,\F)$ are singular if and only if there exists two disjoint measurable sets $A_1$ and $A_2$ such that $A_1\cup A_2= \Omega$, $m_1(A_2)=0$ and $m_2(A_1)=0$.} and 
$$
\mu_1=\frac{d\mu_1}{d\mu_2}\cdot\mu_2+m.$$
The density $\frac{d\mu_1}{d\mu_2}$ is unique up to a $\mu_2$-negligible event.
\end{theo}
We adopt the convention that $\frac{d\mu_1}{d\mu_2}=+\infty$ on the complementary of the support of $\mu_2$. 
\begin{lemm} \label{le:mus}
We have 
\begin{itemize}
\item $\mu_1\big(\frac{d\mu_1}{d\mu_2}=0\big)=0$.
\item $\mu_2\big(\frac{d\mu_1}{d\mu_2}>0\big)>0 \Leftrightarrow \mu_1\big(\frac{d\mu_2}{d\mu_1}>0\big)>0$.
\end{itemize}
\end{lemm}
\begin{proof}
The first point is a clear consequence of the decomposition $\mu_1=\frac{d\mu_1}{d\mu_2}\cdot\mu_2+m$ and of the convention mentioned above. For the second point, one can write by uniqueness of the decomposition:
$$
\mu_2\bigg(\frac{d\mu_1}{d\mu_2}>0\bigg)=0\Leftrightarrow \frac{d\mu_1}{d\mu_2}=0 \ \mu_2-a.s. \Leftrightarrow \mu_1=m
\Leftrightarrow \mu_1 {\rm \ and \ } \mu_2 {\rm \ are \ singular}.
$$
And by symmetry of the roles of $\mu_1$ and $\mu_2$:
$$
\mu_2\bigg(\frac{d\mu_1}{d\mu_2}>0\bigg)>0 \Leftrightarrow \mu_1 {\rm \ and \ } \mu_2 {\rm \ are \ not \ singular}
\Leftrightarrow \mu_1\bigg(\frac{d\mu_2}{d\mu_1}>0\bigg)>0.
$$
\end{proof}

Let us explain what these results have to do with condition (c).\\
One may be able to distinguish environment $\theta$ from $\tth$ if a certain arm $\ell$ delivers a payoff that is infinitely more likely in $\tth$ than in $\theta$. This is for instance the case if $X_{\ell,t}$ is in the support of $\tnu_\ell$ and not in the support of $\nu_\ell$, but our condition is more general. If the agent observes a payoff $x$ from arm $\ell$, the quantity $\frac{d\nu_\ell}{d\tnu_\ell}(x)$ represents how much the observation of $x$ is more likely in environment $\theta$ than in $\tth$. If $\nu_k$ and $\tnu_k$ admit density functions (say, respectively, $f$ and $\tilde f$) with respect to a common measure, then $\frac{d\nu_\ell}{d\tnu_\ell}(x)=\frac{f(x)}{\tilde{f}(x)}$. Thus the agent will almost never make a mistake if he removes $\theta$ from possible environments when $\frac{d\nu_\ell}{d\tnu_\ell}(x)=0$. This may happen even if $x$ is in both supports of $\nu_\ell$ and $\tnu_\ell$, for example if $x$ is an atom of $\tnu_\ell$ and not of $\nu_\ell$ (i.e. $\tnu_\ell(x)>0$ and $\nu_\ell(x)$=0). On the contrary, if $\frac{d\nu_\ell}{d\tnu_\ell}(x)>0$ both environments $\theta$ and $\tth$ are likely and arm $\ell$'s behaviour is both consistent with $\theta$ and $\tth$.\\

Now let us state the impossibility result. Here and throughout the paper we find it more convenient to denote $f\gg_{+\infty}g$ rather than the usual notation $g=o(f)$, which has the following meaning: $$\forall \varepsilon> 0, \ \exists N\geq 0, \ \forall n\geq N, \ g(n)\leq \varepsilon f(n).$$

\begin{theo} \label{th:main}
Let $f:\N\to\R_+^*$ be greater than order $\log^\alpha$, that is for any $\alpha>0, \ f\gg_{+\infty}\log^\alpha$.\\
Assume that there exists $\theta$, $\tilde{\theta}\in\Theta$, and $k\in\{1,\ldots,K\}$ such that:
\begin{itemize}
\item[(a)] $\nu_k=\tilde{\nu}_k,$
\item[(b)] $k$ is the index of the best arm in $\theta$ but not in $\tilde{\theta}$,
\item[(c)] $\forall \ell\neq k, \ \P_{\tilde{\theta}}\big(\frac{d\nu_\ell}{d\tilde{\nu}_\ell}(X_{\ell,1})>0\big)>0$.
\end{itemize}
Then there is no $f$-w$\mathcal T$ anytime policy, and hence no $f$-$\mathcal R$ anytime policy.
\end{theo}

Let us give some hints of the proof (see Section \ref{proofs} for details).
The main idea is to consider a policy that would be $f$-w$\mathcal T$, and in particular 
that would ``work well'' in environment $\theta$ in the sense given by the definition of $f$-w$\mathcal T$.
The proof exhibits a time $N$ at which arm $k$, optimal in environment $\theta$ and thus often drawn 
with high $\P_\theta$-probability, is drawn too many times (more than the logarithmic threshold $C\frac{\log(N)}{\Delta_k^2}$) with not so small $\P_{\tth}$-probability, which shows the nonexistence of such a policy.
More precisely, let $n$ be large enough and consider a time $N$ of order $\log n$ and above the threshold.
If the policy is $f$-w$\mathcal T$, at time $N$, sampling times of suboptimal arms are of order $\log N$ at most, with $\P_\theta$-probability at least $1-\fracc{\tC}{f(N)}$.
In this case, at time $N$, the draws are concentrated on arm $k$. So $T_k(N)$ is of order
$N$, which is more than the threshold. This event holds with high $\P_\theta$-probability.
Now, from (a) and (c), we exhibit constants that are characteristic of the ability of arms $\ell\neq k$ to ``behave as if in $\theta$'': for some $0<a,\eta<1$, there is a subset $\xi$ of this event such that $\P_\theta(\xi) \ge a^T$ for
$T=\sum_{\ell\neq k} T_\ell(N)$ and for which $\frac{d\P_\theta}{d\P_{\tth}}$ is lower bounded by $\eta^T$. The event $\xi$ on which the arm $k$ is sampled $N$ times at least has therefore a $\P_{\tth}$-probability
of order $(\eta a)^T$ at least. This concludes this sketchy proof since
$T$ is of order $\log N$, thus $(\eta a)^T$ is of order $\log^{\log(\eta a)} n$ at least.\\

Note that the conditions given in Theorem \ref{th:main} are not very restrictive. The impossibility holds for very basic settings, and may hold even if the agent has great knowledge of the possible environments. For instance, the setting $$K=2 \ {\rm and} \ \Theta=\left\{\left(Ber\Big(\frac{1}{4}\Big),\delta_{\frac{1}{2}}\right),\left(Ber\Big(\frac{3}{4}\Big),\delta_{\frac{1}{2}}\right)\right\},$$
where $Ber(p)$ denotes the Bernoulli distribution of parameter $p$ and $\delta_x$ the Dirac measure on $x$,
 satisfies the three conditions of the theorem.\\
Nevertheless, the main interest of the result regarding the previous literature is the following corollary.

\begin{coro}\label{cor:main}
If $\Theta$ is the whole set $\bar\Theta$ of all $K$-tuples of measures on $[0,1]$, then there is no $f$-$\mathcal R$ anytime policy, where $f$ is any function such that $f\gg_{+\infty}\log^\alpha$ for all $\alpha>0$.
\end{coro}

This corollary should be read in conjunction with the following result for {\sc ucb-h} which, for a given~$n$, plays at time $t\ge K+1$, 
\begin{equation*}
I_t\in\argmax_{k\in\set{1}{K}} \Bigg\{ \hat{X}_{k,T_k(t-1)} + \sqrt{\frac{2\log n}{T_k(t-1)}} \Bigg\}.
\end{equation*}

\begin{theo} \label{th:ucbh}
For any $\beta>0$, {\sc ucb-h} is $n^\beta$-$\mathcal R$.
\end{theo}
For $\rho >1$, Theorem \ref{th:ucbh} can easily be extended to the policy {\sc ucb-h}($\rho$) which
starts by drawing each arm once, and then at time $t\ge K+1$, plays
\begin{equation} \label{ucbhrho}
I_t\in\argmax_{k\in\set{1}{K}} \Bigg\{ \hat{X}_{k,T_k(t-1)} + \sqrt{\frac{\rho \log n}{T_k(t-1)}} \Bigg\}.
\end{equation}

Naturally, we have $n^{\beta}\gg_{n\to+\infty}\log^\alpha(n)$ for all $\alpha, \beta>0$ but this does not contradict our theorem, since {\sc ucb-h}($\rho$) is not an {\it anytime} policy. {\sc ucb-h} will work fine if the horizon $n$ is known in advance, but may perform poorly at other rounds.\\

Corollary \ref{cor:main} should also be read in conjunction with the following result for the policy {\sc ucb1}($\rho$) which
starts by drawing each arm once, and then at time $t\ge K+1$, plays
\begin{equation} \label{ucb1rho}
I_t\in\argmax_{k\in\set{1}{K}} \Bigg\{ \hat{X}_{k,T_k(t-1)} + \sqrt{\frac{\rho \log t}{T_k(t-1)}} \Bigg\}.
\end{equation}
\begin{theo}
For any $\rho>1/2$, {\sc ucb1}($\rho$) is $\log^{2\rho-1}$-$\mathcal R$.
\end{theo}

Thus, any improvements of existing algorithms which would for instance involve estimations of variance (see \cite{audibert2009exploration}), of $\Delta_k$, or of many characteristics of the distributions cannot beat the variants of {\sc ucb1} regarding deviations.

\section{Positive results} \label{sec:pos}

The intuition behind Theorem \ref{th:main} suggests that, if one of the three conditions (a), (b), (c) does not hold, a robust policy would consist in the following: at each round and for each arm $k$, compute a distance between the empirical distribution of arm $k$ and the set of distribution $\nu_k$ that makes arm $k$ optimal in a given environment $\theta$. 
As this distance decreases with our belief that $k$ is the optimal arm, the policy consists in taking the $k$ minimizing the distance.
Thus, the agent chooses an arm that fits better a winning distribution $\nu_k$. He cannot get stuck pulling a suboptimal arm because there are no environments $\tth$  with $\nu_k=\tnu_k$ in which $k$ would be suboptimal. More precisely, if there exists such an environment $\tth$, the agent is able to distinguish $\theta$ from $\tth$: during the first rounds, he pulls every arm and at least one of them will never behave as if in $\theta$ if the current environment is $\tth$. Thus, in $\tth$, he is able to remove $\theta$ from the set of possible environments $\Theta$ (remember that $\Theta$ is a parameter of the problem which is known by the agent).\\

Nevertheless such a policy cannot work in general, notably because of the three following limitations:
\begin{itemize}

\item If $\tth$ is the current environment and even if the agent has identified $\theta$ as impossible (i.e. $\frac{d\nu_k}{d\tilde{\nu}_k}(X_{k,1})=0$), there still could be other environments $\theta'$ that are arbitrary close to $\theta$ in which arm $k$ is optimal and which the agent is not able to distinguish from $\tth$. This means that the agent may pull arm $k$ too often because distribution $\tilde{\nu}_k=\nu_k$ is too close to a distribution $\nu'_k$ that makes arm $k$ the optimal arm.

\item The ability to identify environments as impossible relies on the fact that the event $\frac{d\nu_k}{d\tnu_k}(X_{k,1})>0$ is almost sure under $\Pro_{\theta}$ (see Lemma \ref{le:mus}). If the set of all environments $\Theta$ is uncountable, such a criterion can lead to exclude the actual environment. For instance, assume an agent has to distinguish a distribution among all Dirac measures $\delta_x$ ($x\in[0,1]$) and the uniform probability $\lambda$ over $[0,1]$. Whatever the payoff $x$ observed by the agent, he will always exclude $\lambda$ from the possible distributions, as $x$ is always infinitely more likely under $\delta_x$ than under $\lambda$: $$\forall x\in [0,1], \ \frac{d\lambda}{d\delta_x}(x)=0.$$

\item On the other hand, the agent could legitimately consider an environment $\theta$ as unlikely if, for $\varepsilon>0$ small enough, there exists $\tth$ such that $\frac{d\nu_k}{d\tnu_k}(X_{k,1})\leq \varepsilon$. Criterion (c) only considers as unlikely an environment $\theta$ when there exists $\tth$ such that $\frac{d\nu_k}{d\tnu_k}(X_{k,1})=0$.

\end{itemize}

Despite these limitations, we give in this section sufficient conditions on $\Theta$ for such a policy to be robust. This is equivalent to finding conditions on $\Theta$ under which the converse of Theorem \ref{th:main} holds, i.e. under which the fact one of the conditions (a), (b) or (c) does not hold implies the existence of a robust policy. This can also be expressed as finding which kind of knowledge of the environment enables to design anytime robust policies.\\

 We estimate distributions of each arm by means of their empirical cumulative distribution functions, and distance between two c.d.f. is measured by the norm $\Vert.\Vert_{\infty}$, defined by  $\Vert f\Vert_{\infty}=\sup_{x\in[0,1]}|f(x)|$ where $f$ is any function $[0,1]\to\R$. The empirical c.d.f of arm $k$ after having been pulled $t$ times is denoted $\hat F_{k,t}$. The way we choose an arm at each round is based on confidence areas around $\hat F_{k,T_k(n-1)}$. We choose the greater confidence level ({\sc gcl}) such that there is still an arm $k$ and a winning distribution $\nu_k$ such that $F_{\nu_k}$, the c.d.f. of $\nu_k$, is in the area of $\hat F_{k,T_k(n-1)}$. We then select the corresponding arm $k$. By means of Massart's inequality (\citeyear{massart1990tight}), this leads to the c.d.f. based algorithm described in Figure \ref{fig:game}. $\Theta_k$ denotes the set $\{\theta\in\Theta| k {\rm \ is \ the \ optimal \ arm \ in} \ \theta  \}$, i.e. the set of environments that makes $k$ the index of the optimal arm.

\begin{figure} 
\bookbox{

Proceed as follows:
\begin{itemize}
\item Draw each arm once.
\item Remove each $\theta\in\Theta$ such that there exists $\tilde{\theta} \in \Theta$ and $\ell\in\{1,\ldots,K\}$ with $\frac{d\nu_\ell}{d\tilde{\nu}_\ell}(X_{\ell,1})=0$.
\item Then at each round $t$, play an arm
 $$I_t \in \smash{\mathop{{\rm argmin}}\limits_{k\in\{1,\dots,K\}}} T_k(t-1) \inf_{\theta\in\Theta_k}\big\Vert\hat{F}_{k,T_k(t-1)}-F_{\nu_k}\big\Vert_{\infty}^2.$$
\end{itemize}
}
\vspace{-0.15in}
\caption{A c.d.f.-based algorithm: {\sc gcl}.} \label{fig:game}
\vspace{-0.0in}
\end{figure}

\subsection{$\Theta$ is finite}
When $\Theta$ is finite the limitations presented above do not really matter, so that the converse of Theorem \ref{th:main} is true and our algorithm is robust.
\begin{theo}\label{th:mainpositive}
Assume that $\Theta$ is finite and that for all $\theta=(\nu_1,\ldots,\nu_K)$, $\tilde{\theta}=(\tilde{\nu}_1,\ldots,\tilde{\nu}_K)\in\Theta$, and all $k\in\{1,\ldots,K\}$, at least one of the following holds:
\begin{itemize}
\item $\nu_k\neq\tilde{\nu}_k,$
\item $k$ is suboptimal in $\theta$, or is optimal in $\tilde{\theta}$.
\item $\exists \ell\neq k, \ \Pro_{\tilde{\theta}}\left(\frac{d\nu_\ell}{d\tilde{\nu}_\ell}(X_{\ell,1})>0\right)=0$.
\end{itemize}

Then {\sc gcl} is $n^\beta$-$\mathcal T$ (and hence $n^\beta$-$\mathcal R$) for all $\beta>0$.
\end{theo}

\subsection{Bernoulli laws}\label{bernoulli}
We assume that any $\nu_k$ ($k\in\{1,\hdots,K\}$, $\theta\in\Theta$) is a Bernoulli law, and denote by $\mu_k$ its parameter. We also assume that there exists $\gamma\in(0,1)$ such that $\mu_k\in[\gamma,1]$ for all $k$ and all $\theta$.\footnote{The result also holds if all parameters $\mu_k$ are in a given interval $[0,\gamma]$, $\gamma\in(0,1)$.} Moreover we may denote arbitrary environments $\theta, \tth$ by $\theta=(\mu_1,\hdots,\mu_K)$ and $\tth=(\tilde{\mu}_1,\hdots,\tilde{\mu}_K)$.\\
In this case $\frac{d\nu_\ell}{d\tilde{\nu}_\ell}(1)=\frac{\mu_l}{\tmu_l}>0$, so that for any $\theta$, $\tth\in\Theta$ and any $l\in\set{1}{K}$ one has $$\P_{\tilde{\theta}}\left(\frac{d\nu_\ell}{d\tilde{\nu}_\ell}(X_{\ell,1})>0\right)\geq \P_{\tilde{\theta}}(X_{\ell,1}=1)=\tmu_l>0.$$
Therefore condition (c) of Theorem \ref{th:main} holds, and the impossibility result only relies on conditions (a) and (b). Our algorithm can be made simpler: there is no need to try to exclude unlikely environments, and computing the empirical c.d.f. is equivalent to computing the empirical mean (see Figure \ref{fig:ber}). The theorem and its converse are expressed as follows. We will refer to our policy as {\sc gcl-b} as it looks for the environment matching the observations at the Greatest Confidence Level, in the case of Bernoulli distributions.

\begin{figure} 
\bookbox{

Proceed as follows:
\begin{itemize}
\item Draw each arm once.
\item Then at each round $t$, play an arm
 $$I_t \in \smash{\mathop{{\rm argmin}}\limits_{k\in\{1,\dots,K\}}} T_k(t-1) \inf_{\theta\in\Theta_k}\left( \mu_k-\hat X_{k,T_k(t-1)} \right)^2.$$
\end{itemize}

}
\vspace{-0.15in}
\caption{A c.d.f.-based algorithm in case of Bernoulli laws: {\sc gcl-b}.} \label{fig:ber}
\vspace{-0.0in}
\end{figure}

\begin{theo} \label{th:bern}
For any $\theta\in\Theta$ and any $k\in\{1,\dots,K\}$, let us set
$$d_k=\inf_{\tiny \tth\in\Theta_k }|\mu_k-\tilde{\mu}_k|.$$
{\sc gcl-b} is such that:
\[
\forall \beta>0, \ \exists C,\tilde{C}>0, \forall \theta\in\Theta, \ \forall n\geq 1, \ \forall k\in\{1,\dots,K\}, \ \Pro_{\theta}\left(T_k(n)\geq \frac{C\log n}{d_k^2}\right)\leq \frac{\tilde{C}}{n^{\beta}}.
\]

\bigskip

Let $f:\N^*\to\R_+^*$ be greater than order $\log^{\alpha}$: $\forall \alpha>0, \ f\gg_{+\infty}\log^\alpha$.\\
If there exists $k$ such that 
\begin{itemize}
\item[(a')] $\displaystyle\inf_{\theta\in\Theta\setminus\Theta_k}d_k=\inf_{\tiny \begin{array}{c}
\theta\in\Theta_k \\ \tth\in\Theta\setminus\Theta_k\end{array}}|\mu_k-\tilde{\mu}_k|=0,$
\end{itemize}
 then there is no anytime policy such that:
\[
\exists C,\tilde{C}>0, \forall \theta\in\Theta, \ \forall n\geq 2, \ \forall k\neq k^*, \ \Pro_{\theta}\left(T_k(n)\geq C\log n\right)\leq \frac{\tilde{C}}{f(n)}.
\]

\end{theo}
Note that we do not adopt the former definitions of robustness ($f$-$\mathcal R$ and $f$-$\mathcal T$), because the significant term here is $d_k$ (and not $\Delta_k$)\footnote{There is no need to leave aside the case of $d_k=0$: with the convention $\frac{1}0=+\infty$, the corresponding event has zero probability.}, which represents the distance between $\Theta_k$ and $\Theta\smallsetminus\Theta_k$. Indeed robustness lies on the ability to distinguish environments, and this ability is all the more stronger as the distance between the parameters of these environments is greater. Provided that the density $\frac{d\nu}{d\tnu}$ is uniformly bounded away from zero, the theorem holds for any parametric model, with $d_k$ being defined with a norm on the space of parameters (instead of $|.|$).\\
Note also that the second part of the theorem is a bit weaker than Theorem \ref{th:main}, because of the interchange of ``$\forall \theta$'' and ``$\exists C,\tilde{C}$''. The reason for this is that condition (a) is replaced by a weaker assumption: $\nu_k$ does not equal $\tnu_k$, but condition (a') means that such $\nu_k$ and $\tnu_k$ can be chosen arbitrarily close.

\subsection{$\mu^*$ is known}

This section shows that the impossibility result also breaks down if
$\mu^*$ is known by the agent. This situation is formalized as $\mu^*$ being constant over $\Theta$. Conditions (a) and (b) of Theorem \ref{th:main} do not hold: if a distribution $\nu_k$ makes arm $k$ optimal in an environment $\theta$, it is still optimal in any environment $\tth$ such that $\tnu_k=\nu_k$.\\
In this case, our algorithm can be made simpler (see Figure \ref{fig:mu*}). At each round we choose the greatest confidence level such that at least one empirical mean $\hX_{k,T_k(t-1)}$ has $\mu^*$ in its confidence interval, and select the corresponding arm $k$. This is similar to the previous algorithm, deviations being evaluated according to Hoeffding's inequality instead of Massart's one. There is one more refinement: the level confidence of arm $k$ at time step $t$ can be defined as $T_k(t-1)( \mu^*-\hat X_{k,T_k(t-1)} )_+^2$ (where, for any $x\in\R$, $x_+$ denotes $\max(0,x)$) instead of \mbox{$T_k(t-1)( \mu^*-\hat X_{k,T_k(t-1)} )^2$}. Indeed, there is no need to penalize an arm for his empirical mean reward being too much greater than $\mu^*$. We will refer to this policy as {\sc gcl$^*$}.

\begin{figure} 
\bookbox{
Proceed as follows:
\begin{itemize}
\item Draw each arm once.
\item Then at each round $t$, play an arm
 $$I_t \in \smash{\mathop{{\rm argmin}}\limits_{k\in\{1,\dots,K\}}} T_k(t-1)\left( \mu^*-\hat X_{k,T_k(t-1)} \right)_+^2.$$
\end{itemize}
}
\vspace{-0.15in}
\caption{{\sc gcl$^*$}: a variant of c.d.f.-based algorithm when $\mu^*$ is known.} \label{fig:mu*}
\vspace{-0.0in}
\end{figure}

\begin{theo}\label{th:mu*}
When $\mu^*$ is known, {\sc gcl$^*$} is $n^\beta$-T (and hence $n^\beta$-R) for all $\beta>0$.
\end{theo}

{\sc gcl$^*$} relies on the use of Hoeffding's inequality. It is now well-established that in general,
the Hoeffding inequality does not lead to the best factor in front of the $\log n$ in the expected regret bound.
The minimax factor has been identified in the works of \citet{LaiRo85,BurKat96} for specific families of probability distributions.
This result has been strengthened in \citet{HonTak10} to deal with the whole set of probability distributions on $[0,1]$.
Getting the best factor in front of the $\log n$ term in the expected regret bound appeared there to be tightly linked with the use of Sanov's inequality. The recent work of \citet{maillard2011finite} builds on a non-asymptotic version of Sanov's inequality to get tight non-asymptotic bounds for probability distributions with finite support. \citet{garivier2011kl} adopts a different starting point: the Chernoff inequality. This inequality states that for i.i.d. random variables $V,V_1,\dots,V_T$, taking their values in
{$[0,1]$}, for any {$\tau<\E V$} we have
    \begin{equation}
    {\P\bigg(\frac1T\sum_{i=1}^T V_i\le \tau\bigg) \le \exp\big(-T\ \mathcal{K}(\tau,\E V)\big)},
    \end{equation}
where $\mathcal{K}(p,q)$ denotes the Kullback-Leibler divergence between Bernoulli distributions of respective parameter $p$ and $q$. It is known to be tight for Bernoulli random variables (as discussed e.g. in \cite{garivier2011kl}).
A Chernoff version of {\sc GCL*} would consist in the following:
 \begin{equation} 
 I_t \in \smash{\mathop{{\rm argmin}}\limits_{k\in\{1,\dots,K\}}} T_k(t-1)\mathcal{K}\big(\min(\hat X_{k,T_k(t-1)},\mu^*),\mu^*\big).
 \end{equation}
 
At the expense of a more refined analysis, it is easy to prove that Theorem \ref{th:mu*} still holds for
this algorithm. Getting a small constant in front of the logarithmic term being an orthogonal discussion to the main topic of this paper, we do not detail further this point.

\section{Horizon policies}

We now study regret deviation properties of horizon policies. Again, we prove that {\sc ucb} policies are optimal. Indeed, deviations of {\sc ucb-h} are of order $1/n^{\alpha}$ (for all $\alpha>0$) and our result shows that this cannot be improved in general.\\
This second impossibility result holds for many settings, that is the one for which there exists $\theta, \tth\in\Theta$ such that:
\begin{itemize}
\item[(b)] an arm $k$ is optimal in $\theta$ but not in $\tth$,
\item[(c')] in environment $\tth$, all arms may behave as if in $\theta$.
\end{itemize}

Indeed, draws have to be concentrated on arm $k$ in environment $\theta$. In particular, with large $\Pro_\theta$-probability, the number of draws of arm $k$ (and only of arm $k$) exceed the logarithmic threshold $\frac{C\log n}{\Delta^2}$ at step $N=\left\lceil K\frac{C\log n}{\Delta^2} \right\rceil$. Such an event only affects a small (logarithmic) number of pulls, so that in environment $\tth$ arms may easily behave as in $\theta$, and   arm $k$ is pulled too often with not so small $\Pro_{\tth}$-probability. More precisely, this event 
happens with at least $\Pro_{\theta}$-probability $1-\frac{(K-1)\tC}{f(n)}$ for a $f$-wT policy. Because arms under environment $\tth$ are able to behave as in $\theta$, there exist constants $0<a,\eta<1$ and a subset $\xi$ of this event such that $\P_\theta(\xi) \ge a^N$ and for which $\frac{d\P_\theta}{d\P_{\tth}}$ is lower bounded by $\eta^N$. The event $\xi$ has then $\P_{\tth}$-probability of order $(\eta a)^N$ at least. As $N$ is of order $\log n$, the probability of arm $k$ being pulled too often in $\tth$ is therefore at least of order $1/n$ to the power of a constant. Hence the following result.

\begin{theo}\label{th:hor}
Let $f:\N\to\R_+^*$ be greater than order $n^\alpha$, that is for any $\alpha>0, \ f(n)\gg_{n\to+\infty}n^\alpha$.\\
Assume that there exists $\theta$, $\tilde{\theta}\in\Theta$, and $k\in\{1,\ldots,K\}$ such that:
\begin{itemize}
\item[(b)] $k$ is the index of the best arm in $\theta$ but not in $\tilde{\theta}$,
\item[(c')] $\forall \ell\in\{1,...,K\}, \ \P_{\tilde{\theta}}\big(\frac{d\nu_\ell}{d\tilde{\nu}_\ell}(X_{\ell,1})>0\big)>0$.
\end{itemize}
Then there is no $f$-w$\mathcal T$ horizon policy, and hence no $f$-$\mathcal T$ horizon policy.
\end{theo}

Note that the conditions under which the impossibility holds are far less restrictive than in Theorem \ref{th:main}. Indeed, conditions (b) and (c') are equivalent to:
\begin{itemize}
\item[(a'')] $\P_{\tilde{\theta}}\big(\frac{d\nu_k}{d\tilde{\nu}_k}(X_{\ell,1})>0\big)>0,$
\item[(b)] $k$ is the index of the best arm in $\theta$ but not in $\tilde{\theta}$,
\item[(c)] $\forall \ell\neq k, \ \P_{\tilde{\theta}}\big(\frac{d\nu_\ell}{d\tilde{\nu}_\ell}(X_{\ell,1})>0\big)>0$.
\end{itemize}

These are the same conditions as in Theorem \ref{th:main}, except for the first one, (a''), which is weaker than condition (a).\\
As a consequence, corollary \ref{cor:main} can also be written in the context of horizon policies.

\begin{coro}\label{cor:hor}
If $\Theta$ is the whole set $\bar\Theta$ of all $K$-tuples of measures on $[0,1]$, then there is no $f$-$\mathcal T$ horizon policy, where $f$ is any function such that $f(n)\gg_{n\to+\infty}n^\alpha$ for all $\alpha>0$.
\end{coro}

Moreover, the impossibility also holds for many basic settings, such as the ones described in section \ref{bernoulli}. This shows that {\sc gcl-b} is not only better in terms of deviations than {\sc ucb} anytime algorithms, but it is also optimal and, despite being an anytime policy , it is at least as good as any horizon policy. In fact, in most settings suitable for a {\sc gcl} algorithm, {\sc gcl} is optimal and is as good as {\sc ucb-h} without using the knowledge of the horizon $n$.\\

Nevertheless, the impossibility is not strong enough to avoid the existence of $f$-$\mathcal R$ horizon policies, with $f(n)\gg_{n\to+\infty}n^\alpha$ and any $\alpha>0$. Proposition \ref{prop:eqRT} does not enable to deduce the non-existence of $f$-$\mathcal R$ policy from the non-existence of $f$-w$\mathcal T$ policy because it needs $f$ to be less than a function of the form $\alpha n^\beta$. We believe that, in general, the impossibility still holds for $f$-$\mathcal R$ horizon policies, but the corresponding conditions will not be easy to write and the analysis will not be as clear as our previous results. Basically, the impossibility would require the existence of a pair of environments $\theta,\tth$ such that
\begin{itemize}
\item an arm $k$ is optimal in $\theta$ but not in $\tth$,
\item in environment $\tth$, all arms may behave as if in $\theta$ in such a way that best arm in $\tth$ would have actually given greater rewards than the other arms if it had been pulled more often.
\end{itemize}

Finally, as in section \ref{sec:pos} one can wonder if there exists a converse to our result. Again, such an analysis would be tougher to perform and we only give some basic hints.\\ If $\Theta$ is such that for any $\theta,\tth\in\Theta$ either (b) or (c') does not hold, then one could actually perform very well. In this degenerated case, only one pull of each arm may make it possible to distinguish tricky pairs of environments $\theta, \tth$, and thus to learn the best arm $k^*$. The agent then keeps on pulling arm $k^*$, and its regret is almost surely less than $K$ at any time step. The tricky part is that, as the distinction relies on the fact that the event $\frac{d\nu_k}{d\tnu_k}(X_{k,1})>0$ is almost sure under $\Pro_{\theta}$ (see Lemma \ref{le:mus}), this may not work if $\Theta$ is uncountable.

\section{Experiments}

Our goal is to compare anytime {\sc ucb} policies, more precisely {\sc ucb1}($\rho$) for $\rho\ge0$ defined by \eqref{ucb1rho}, to the low-deviation 
policies {\sc ucb-h}($\rho$) for $\rho\ge0$, defined by \eqref{ucbhrho}, and {\sc gcl$^*$} introduced in Section \ref{sec:pos} (see Figure \ref{fig:mu*}).
Most bandit policies contain a parameter allowing to tune the exploration-exploitation trade-off. To do a fair comparison with anytime {\sc ucb} policies, we consider the full range of possible exploration parameters.

We estimate the distribution of the regret of a policy by running $100000$ simulations. In parti\-cular, this implies that the confidence interval for the expected regret of a policy is smaller than the size of the markers for $n=100$, and smaller than the linewidth for $n\ge 500$.
The reward distribution of the arms are here either the uniform (Unif), or the Bernoulli (Ber) or the Dirac distributions.



\subsection{Tuning the exploration parameter in {\sc ucb} policies for low expected regret} \label{sec:expA}

In {\sc ucb} policies, the exploration parameter can be interpreted as the confidence level at which a deviation inequality is applied (neglecting union bounds issues).
For instance, the popular {\sc ucb1} uses an exploration term $\sqrt{\fracl{2 \log t}{T_k(t)}}=\sqrt{\fracr{\log(t^4)}{2T_k(t)}}$ corresponding to a $1/t^4$
confidence level in view of Hoeffding's inequality.
Several studies \citep{LaiRo85,Agr95,BurKat96,audibert2009exploration,HonTak10} have shown that 
the critical confidence level is $1/t$. In particular, \cite{audibert2009exploration} 
have considered the policy {\sc ucb1}($\rho$) having the exploration term $\sqrt{\fracl{\rho \log t}{T_k(t)}}$, and shown that this policy have polynomial regrets as soon as $\rho < 1/2$ (and have logarithmic regret for $\rho>1/2$). Precisely, for $\rho <1/2$, the regret of the policy can be lower bounded by $n^\gamma$ with $0<\gamma<1$ which is all the smaller as $\rho$ is close to $1/2$.

The first experiments, reported in Figures \ref{fig:1} and \ref{fig:2}, show that for $n \le 10^8$, taking $\rho$ in $[0.2,0.5)$ generally leads to better performance than 
taking the critical $\rho=0.5$. There is not really a contradiction with the previous results as for such $\rho$, the 
exponent $\gamma$ is so small than there is no great difference between $\log n$ and $n^\gamma$. 
For $n \le 10^8$, the polynomial regret will appear for smaller values of $\rho$ (i.e. $\rho\approx 0.1$ in our experiments).

The numerical simulations exhibit two different types of bandit problems: 
in simple bandit problems (which contain the case where the optimal arm is a Dirac distribution, or the case when the smallest reward that the optimal arm can give
is greater than the largest reward than the other arms can give), the performance of UCB policies is all the better as the exploration is reduced, that is
the expected regret is an increasing function of the exploration parameter. 
In difficult bandit problems (which contain in particular the case when the smaller reward that the optimal arm is smaller than the mean of the second best arm),
there is a real trade-off between exploration and exploitation: the expected regret of {\sc ucb1}($\rho$) decreases with $\rho$ for small $\rho$ and then increases.
Both types of problems are illustrated in Figure \ref{fig:2}.

\begin{figure}
\begin{center}
\scalebox{0.29}{\includegraphics{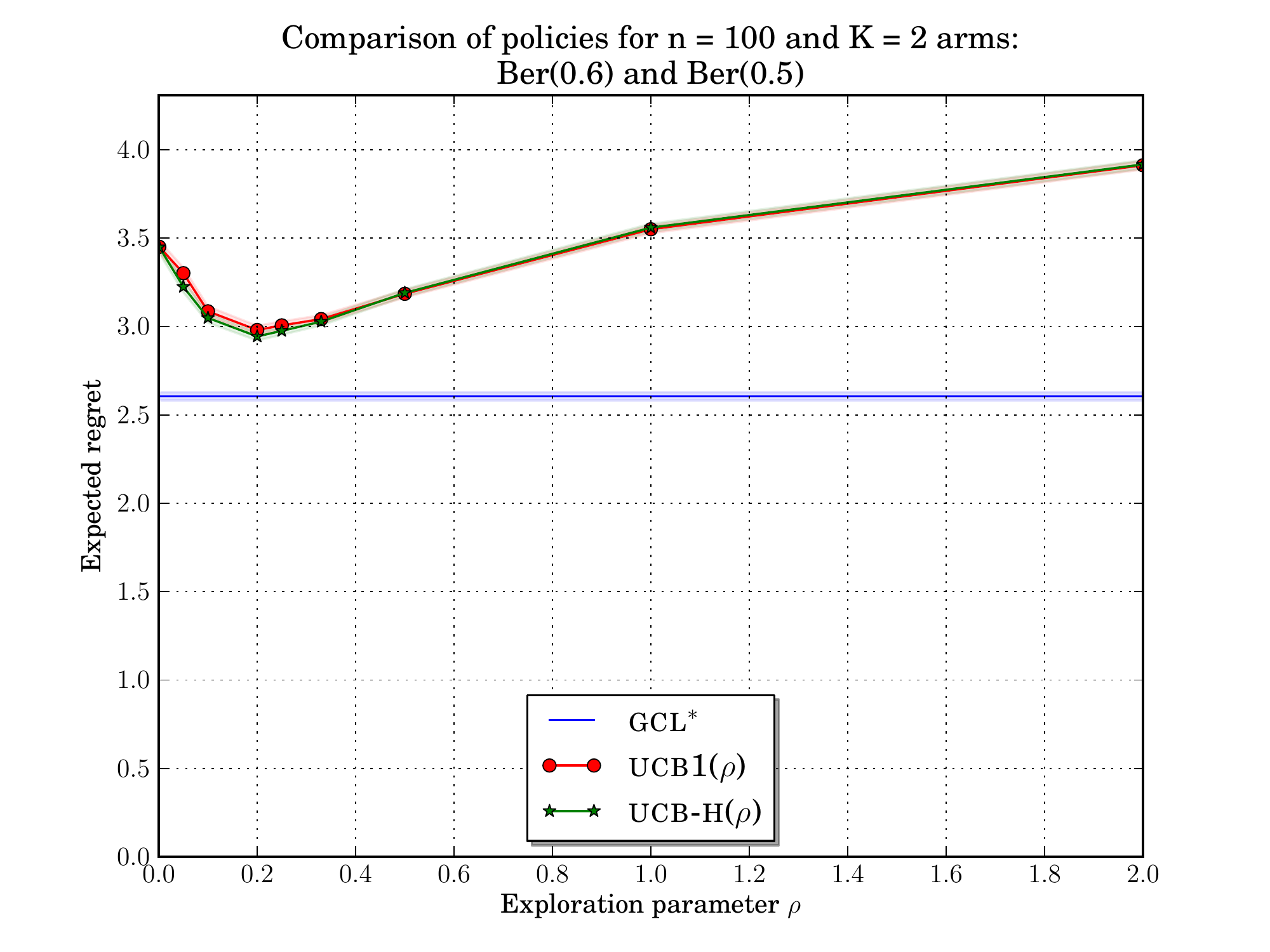}}
\scalebox{0.29}{\includegraphics{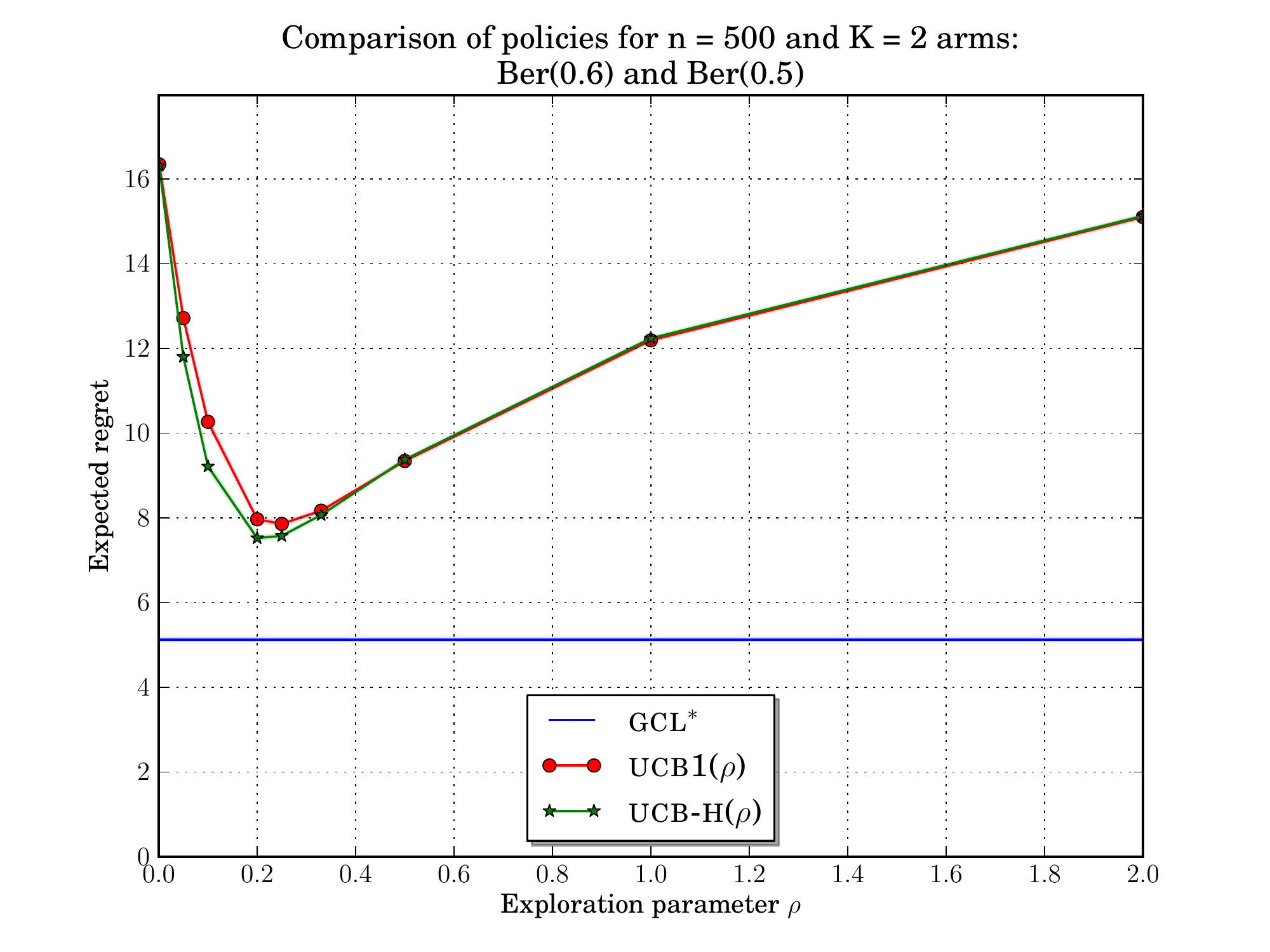}}
\\
\scalebox{0.29}{\includegraphics{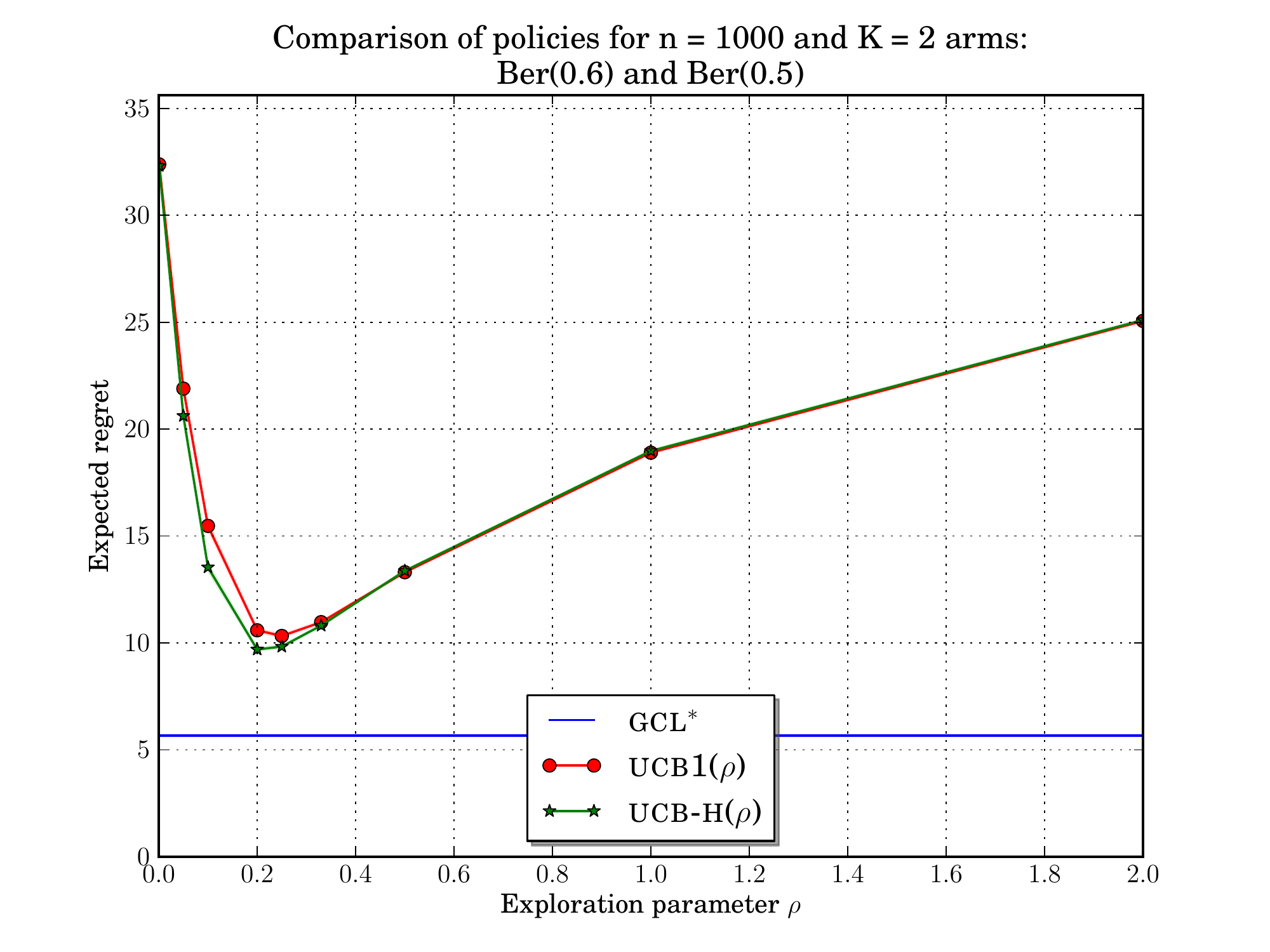}}
\scalebox{0.29}{\includegraphics{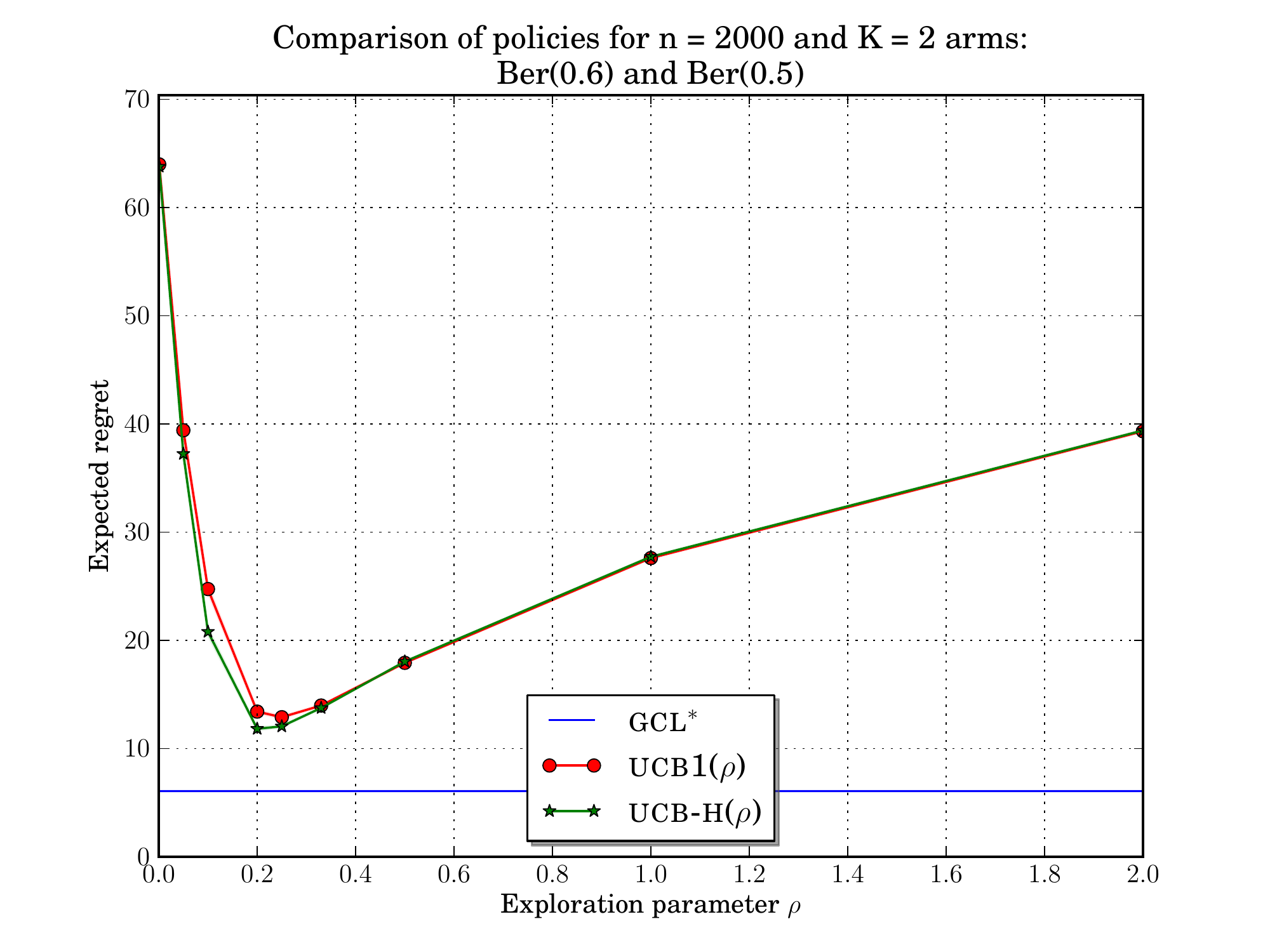}}
\\
\scalebox{0.29}{\includegraphics{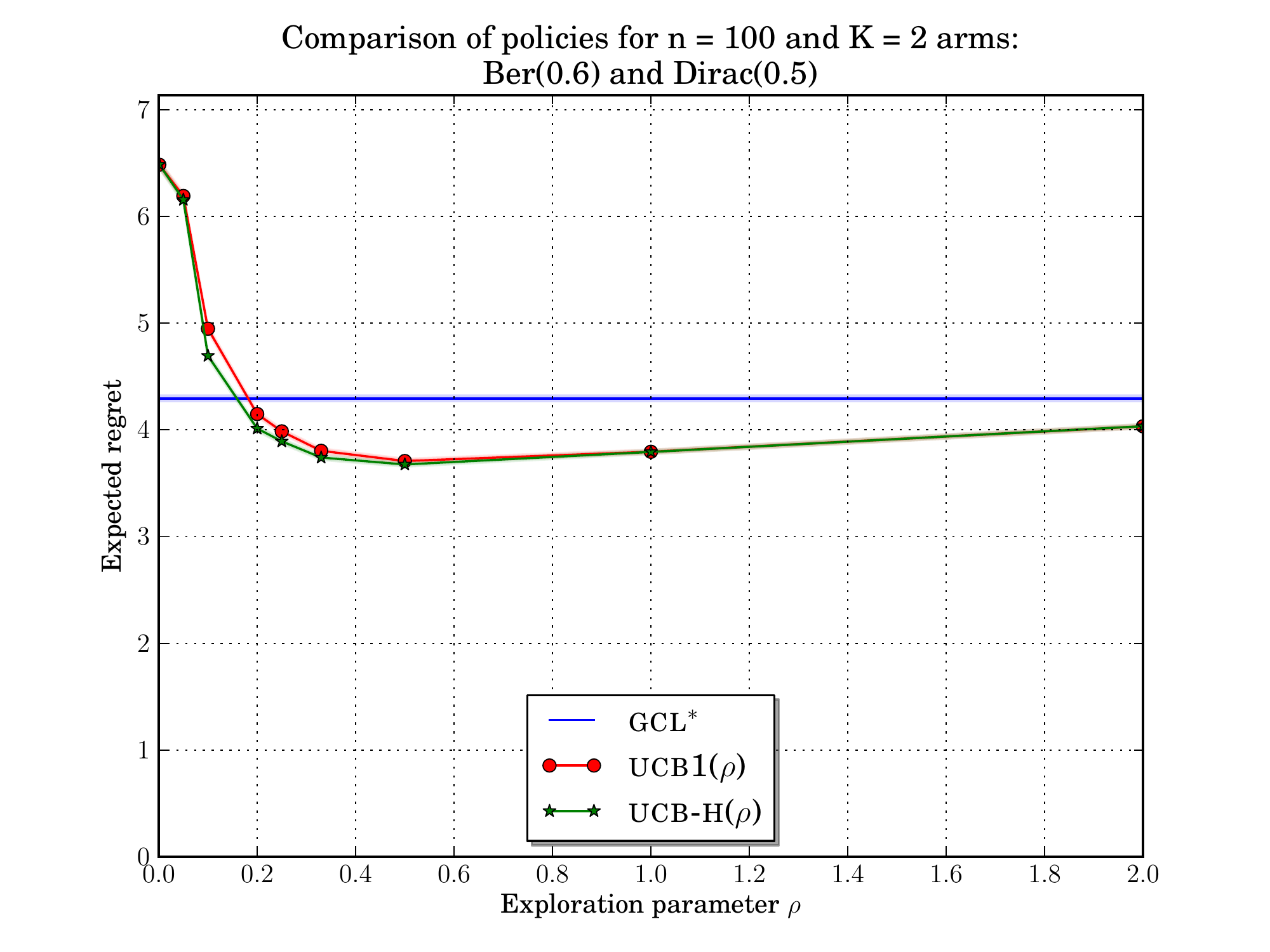}}
\scalebox{0.29}{\includegraphics{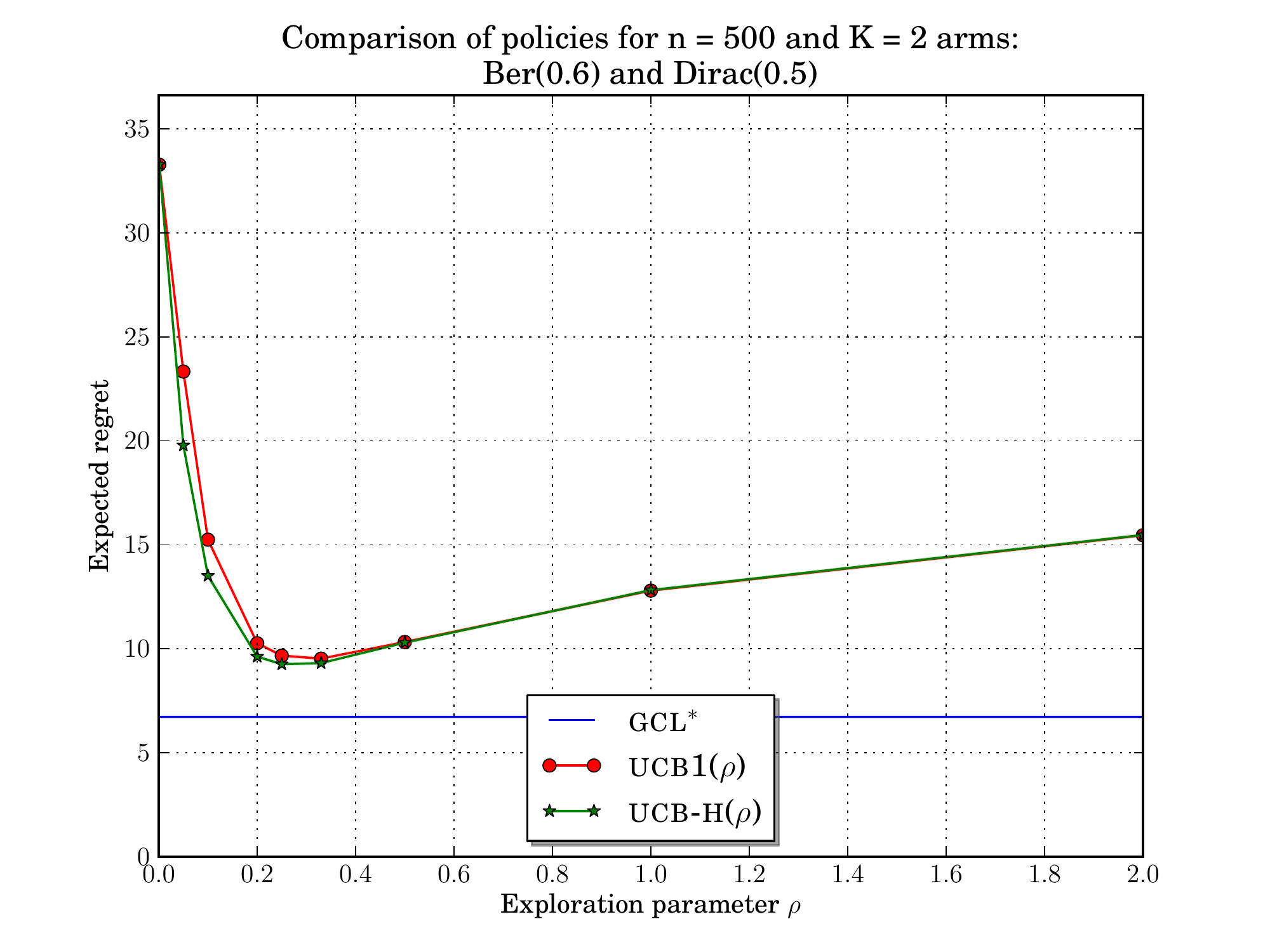}}
\\
\scalebox{0.29}{\includegraphics{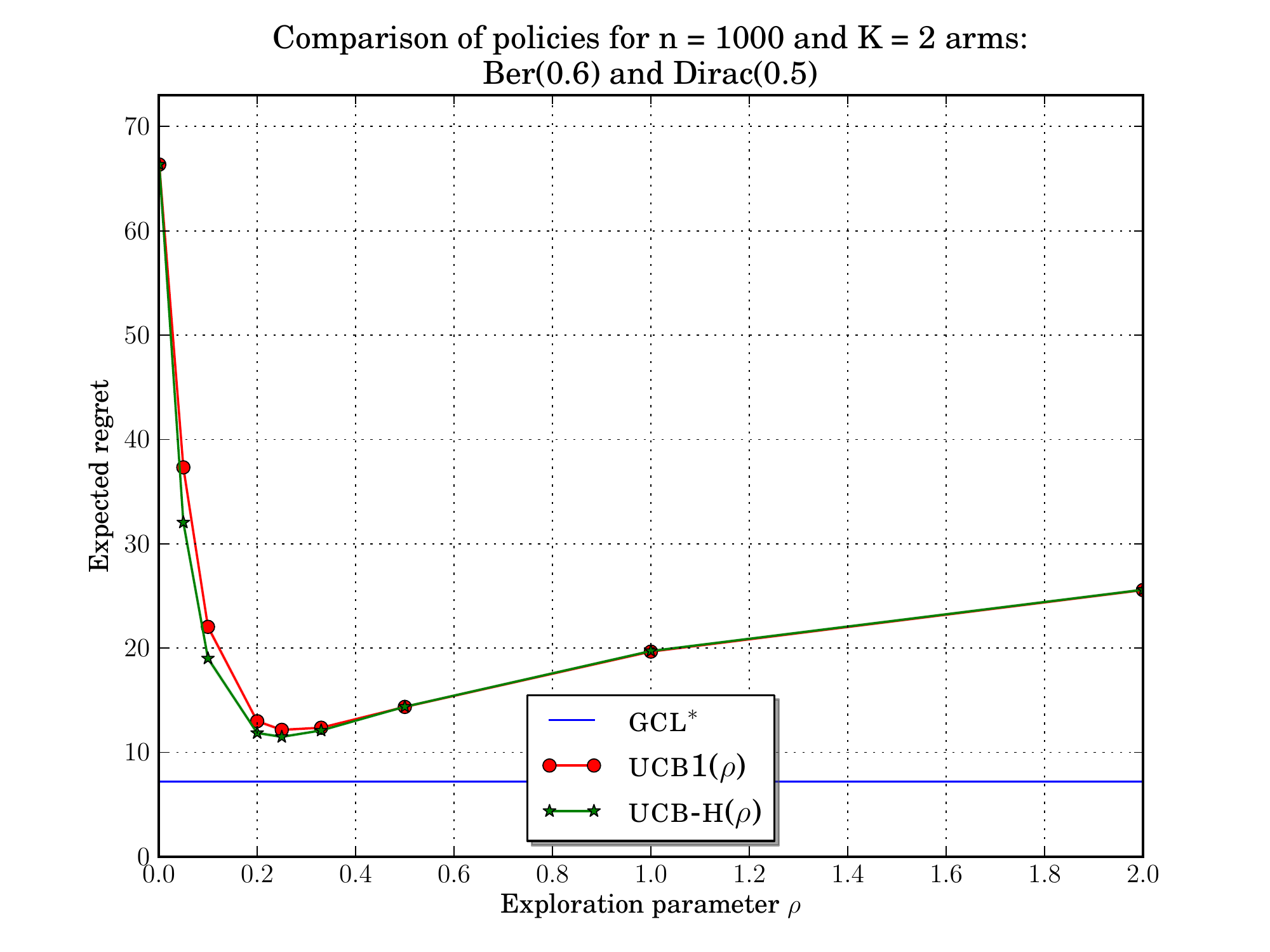}}
\scalebox{0.29}{\includegraphics{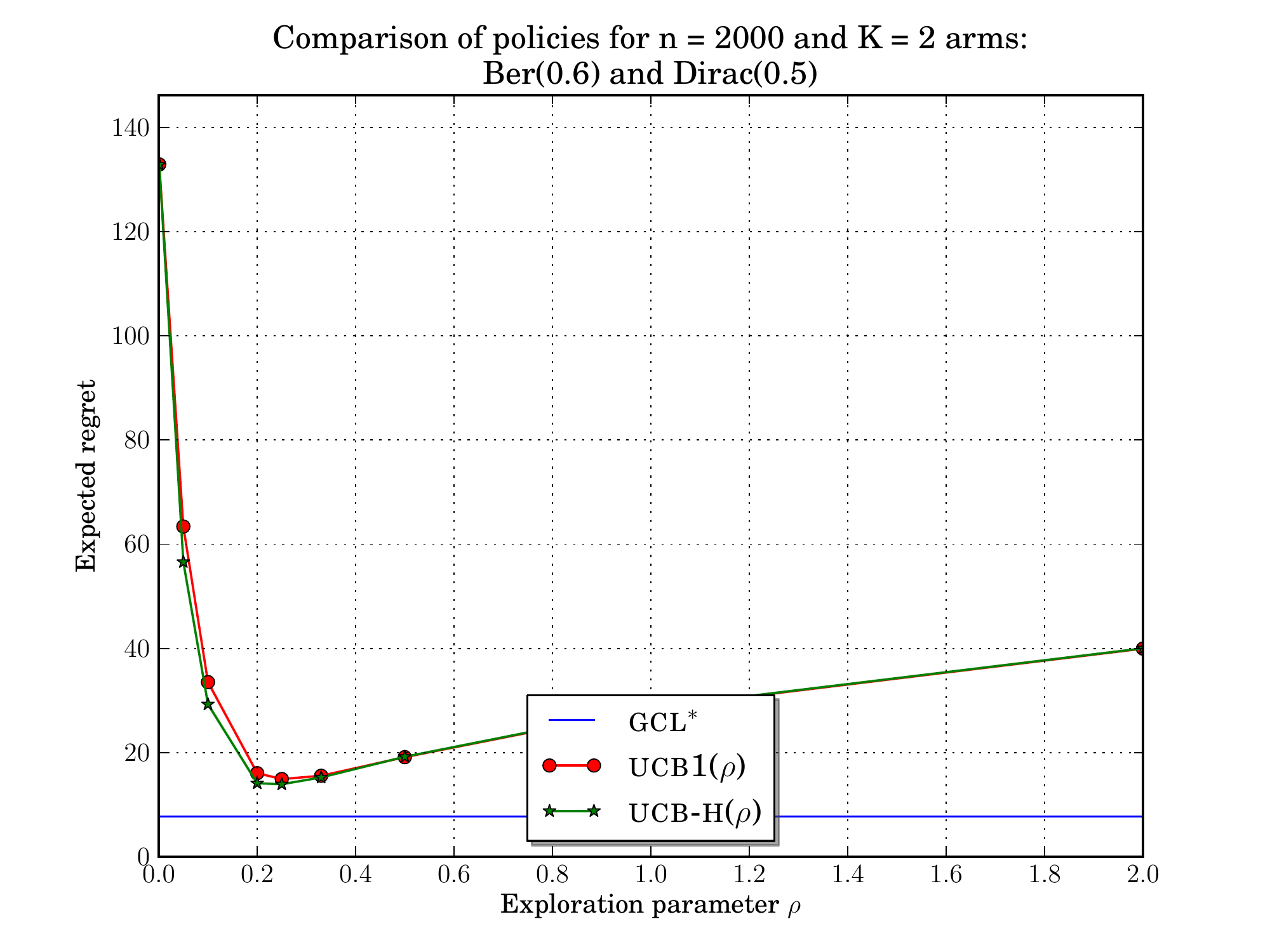}}
\end{center}
\caption{Expected regret of {\sc ucb1}($\rho$), {\sc ucb-h}($\rho$) and {\sc gcl}$^*$
for various bandit settings (1/2). }
\label{fig:1}
\end{figure}

\begin{figure}
\begin{center}
\scalebox{0.29}{\includegraphics{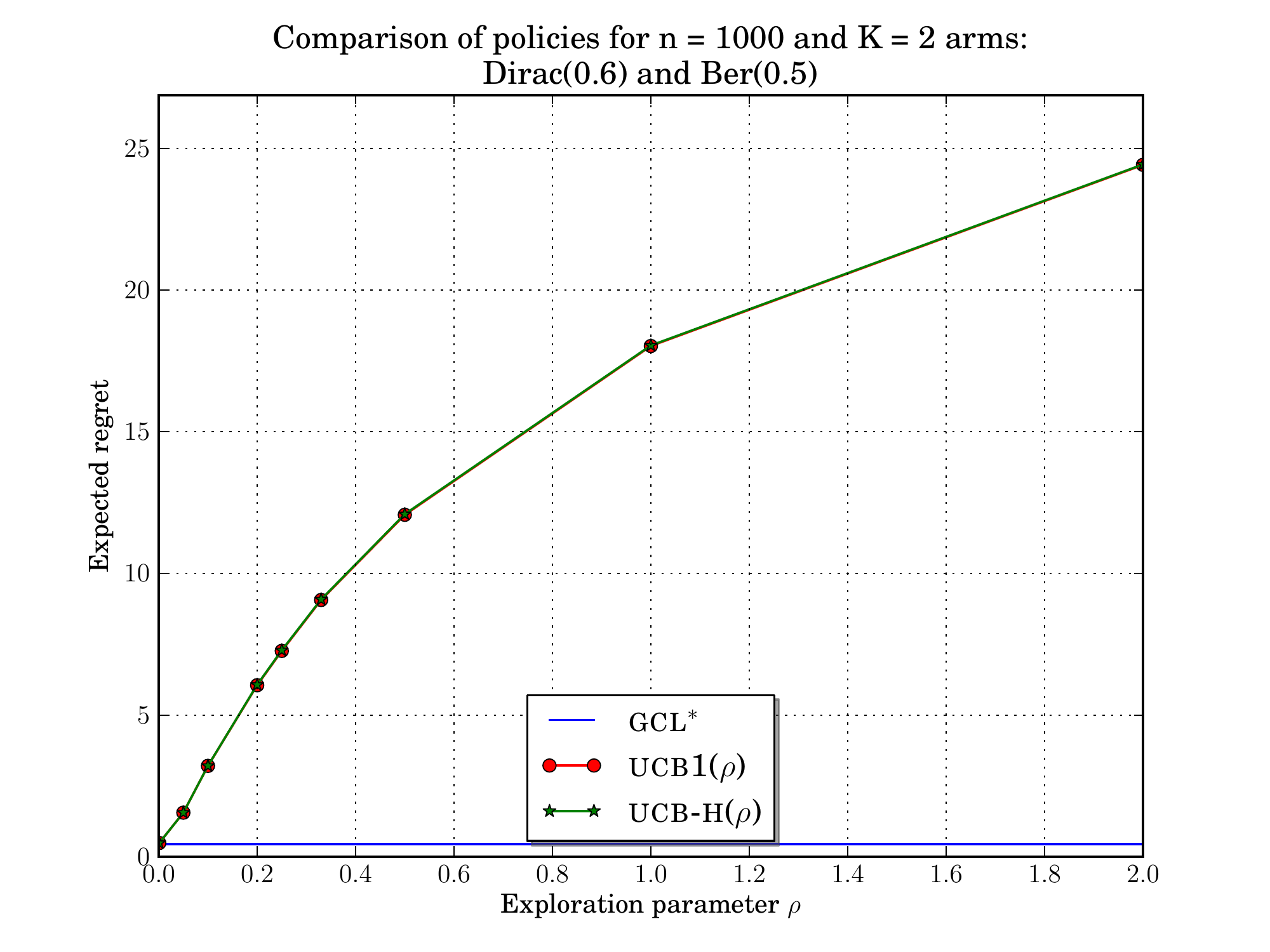}}
\scalebox{0.29}{\includegraphics{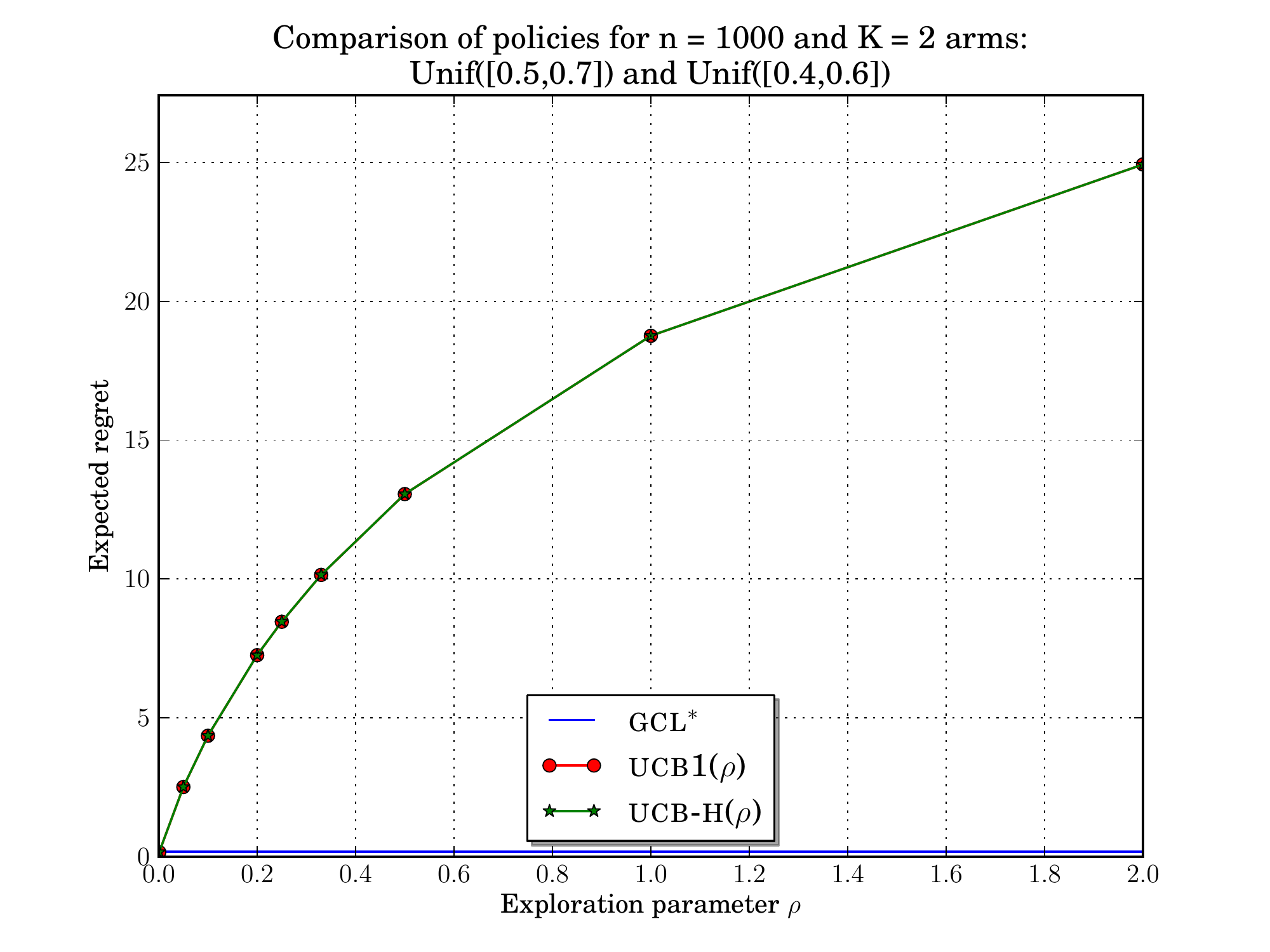}}
\\
\scalebox{0.29}{\includegraphics{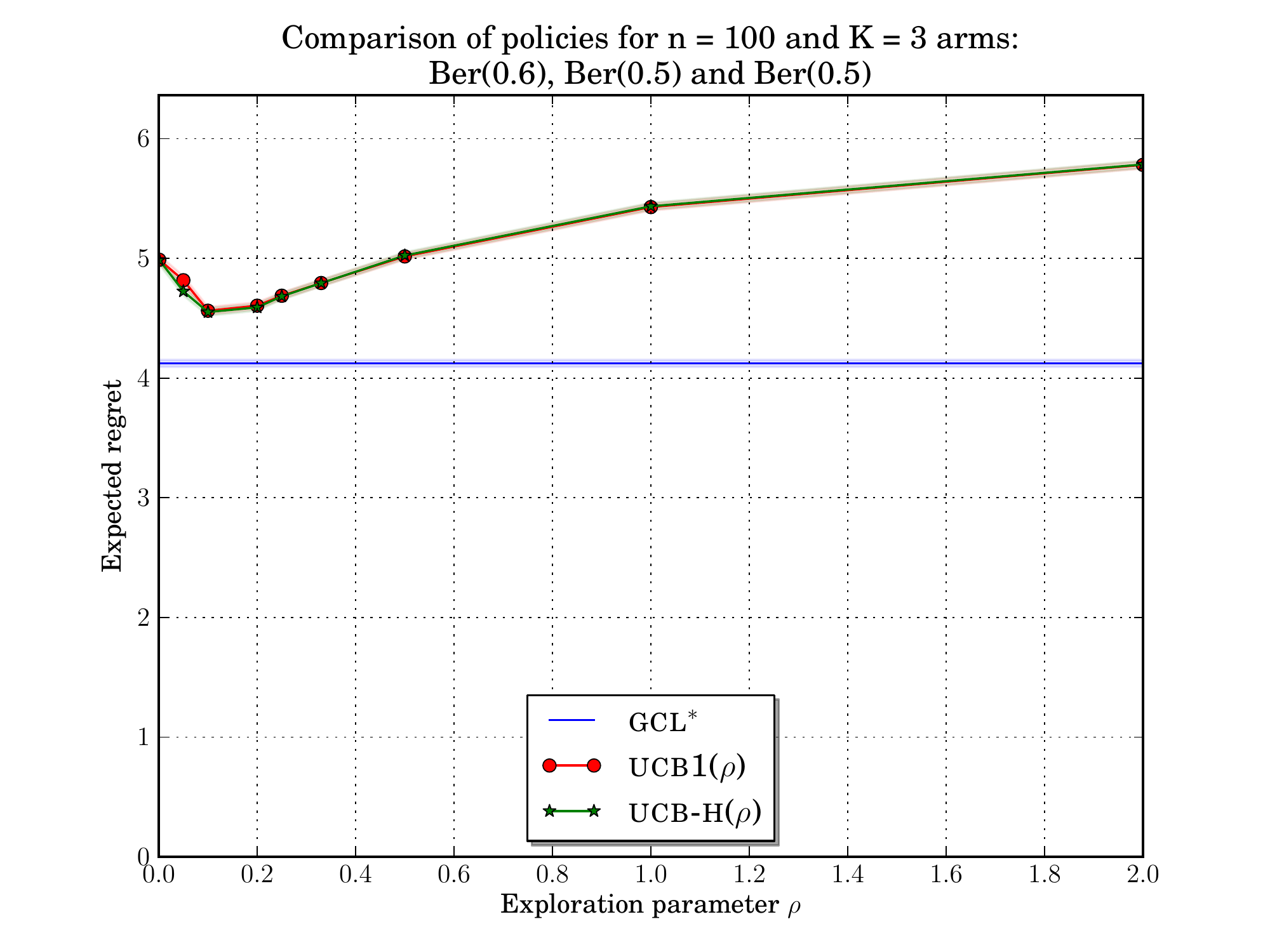}}
\scalebox{0.29}{\includegraphics{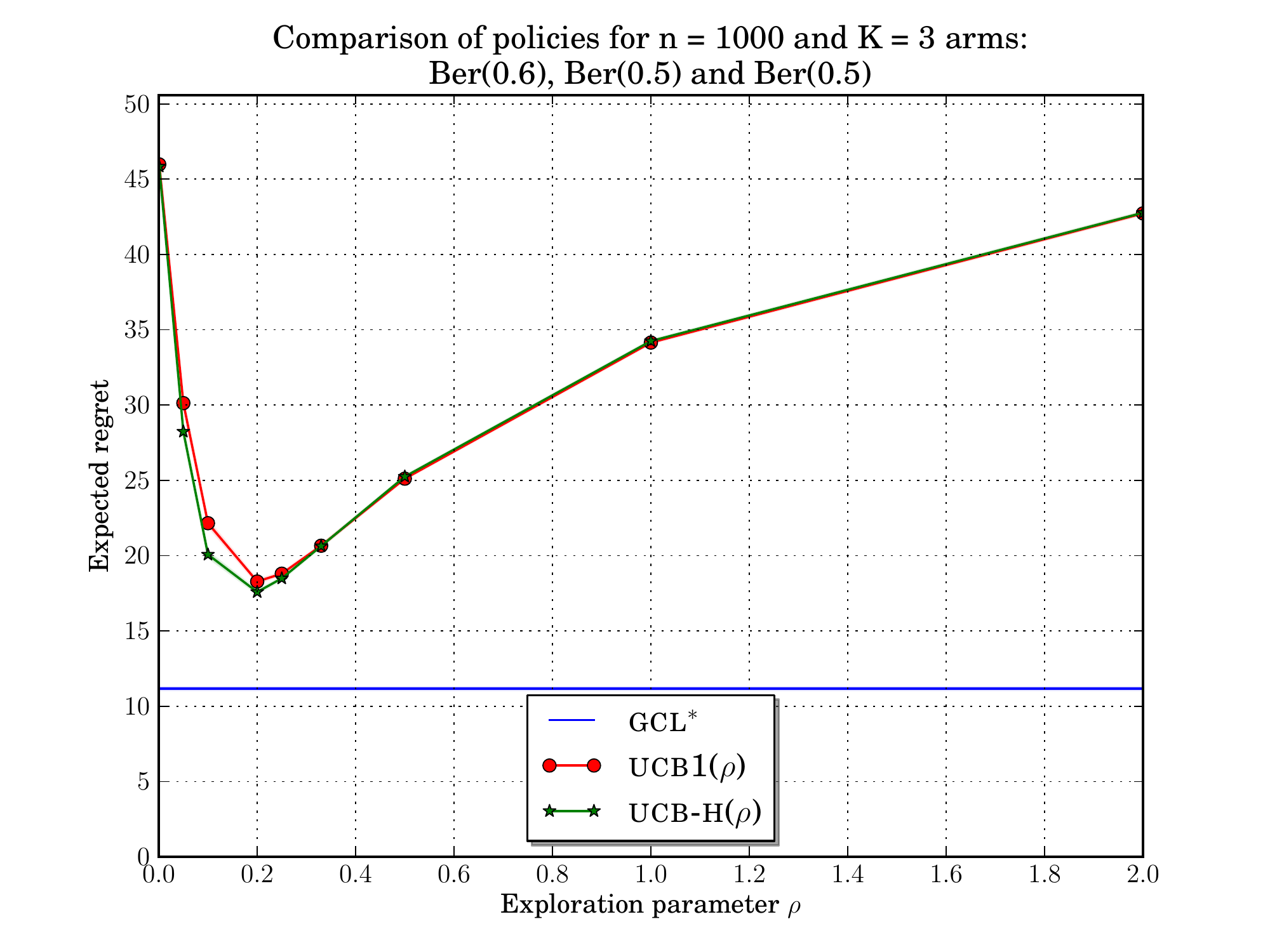}}
\\
\scalebox{0.29}{\includegraphics{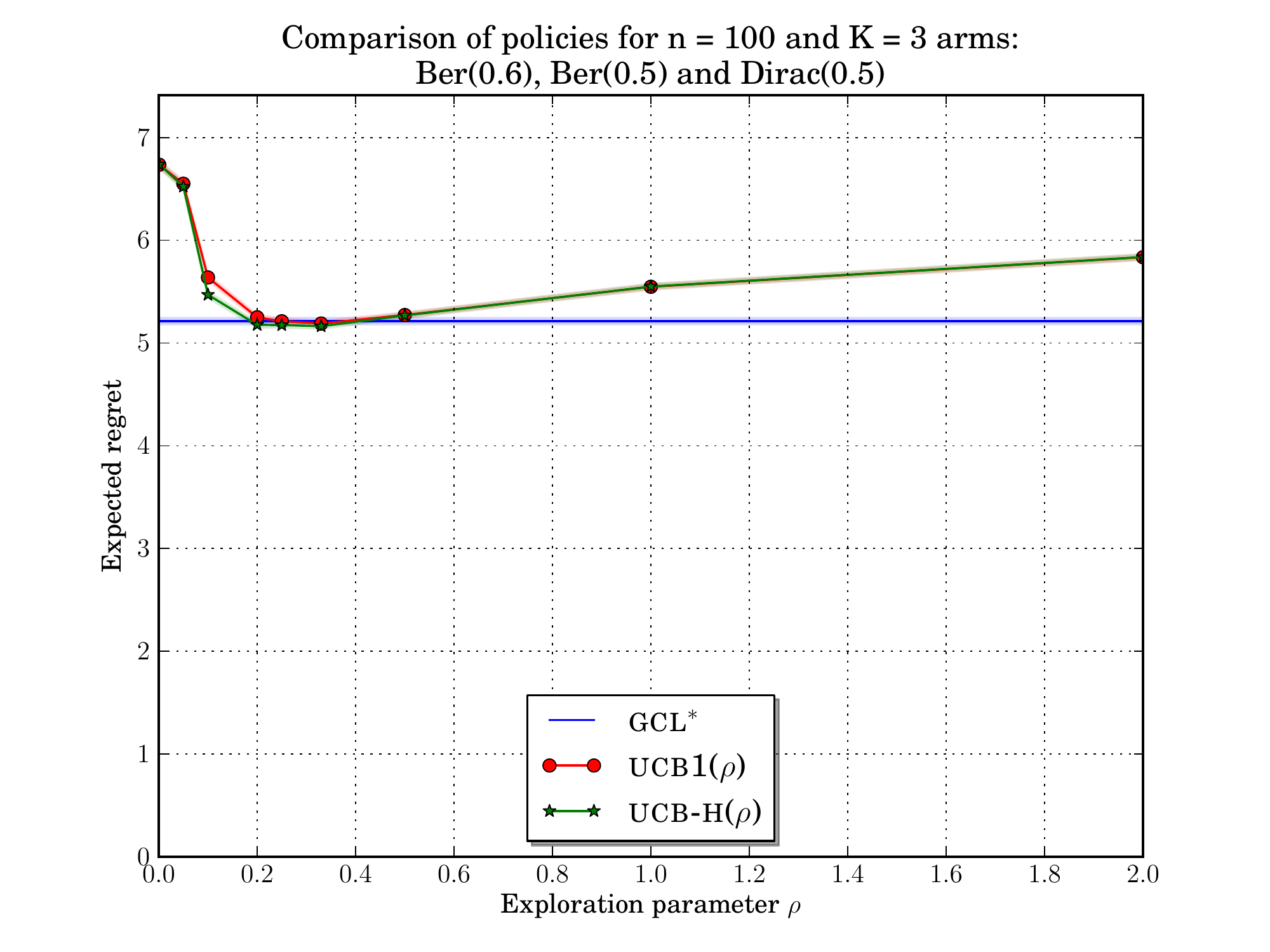}}
\scalebox{0.29}{\includegraphics{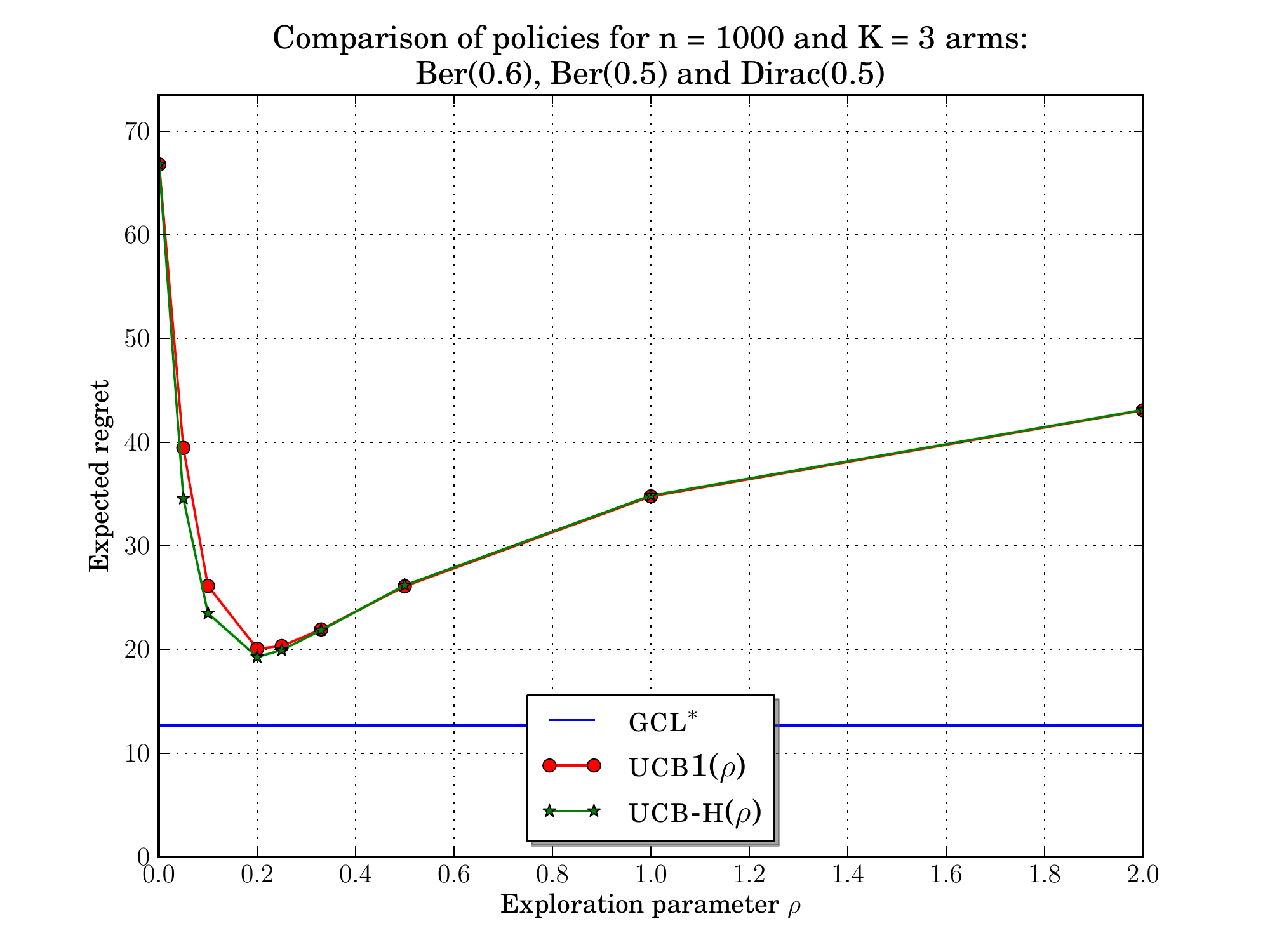}}
\\
\scalebox{0.29}{\includegraphics{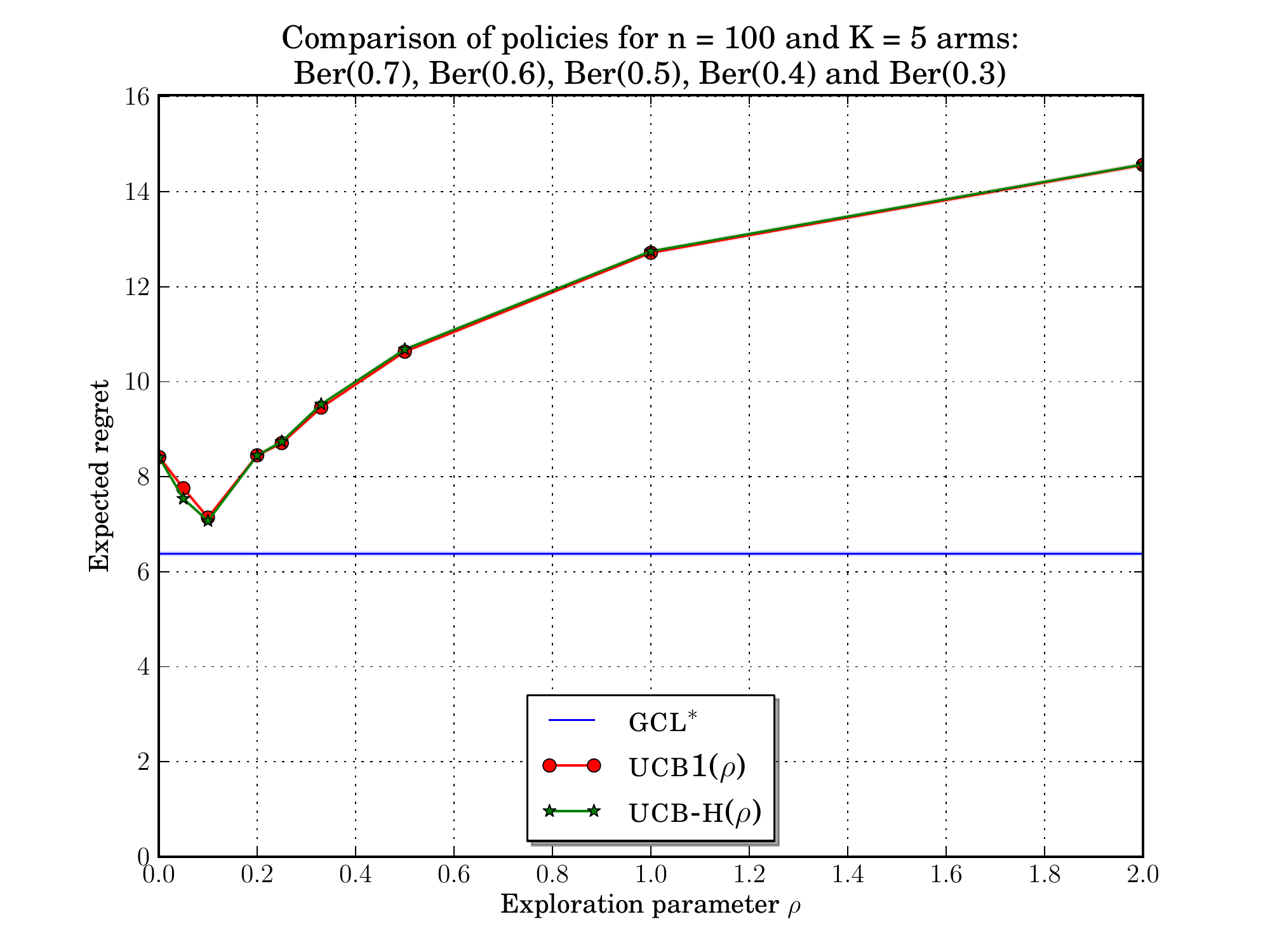}}
\scalebox{0.29}{\includegraphics{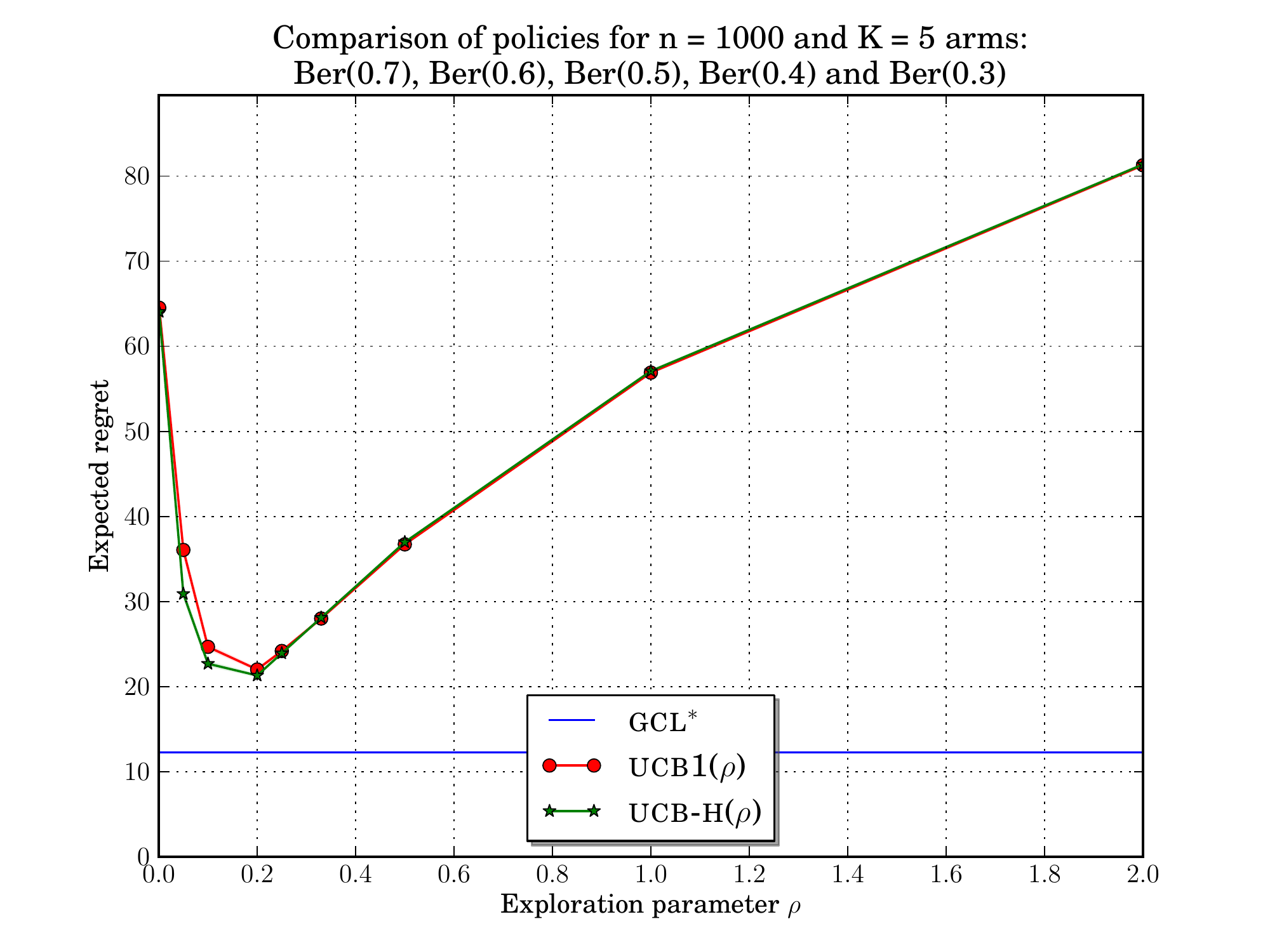}}
\end{center}
\caption{Expected regret of {\sc ucb1}($\rho$), {\sc ucb-h}($\rho$) and {\sc gcl}$^*$
for various bandit settings (2/2)} \label{fig:2}
\end{figure}



\subsection{The gain of knowing the horizon}

There is consistently a slight gain in using {\sc ucb-h}($\rho$) instead of {\sc ucb1}($\rho$) both in terms of expected regret 
(see Figures \ref{fig:1} and \ref{fig:2}) and in terms of deviations (see Figures \ref{fig:devbeg} to \ref{fig:devend} in pages 
\pageref{fig:devbeg} to \pageref{fig:devend}). 

The latter figures also show the following. If the agent's target is not to minimize its expected regret, but to minimize a quantile function at a given confidence level, increasing the exploration parameter $\rho$ (for instance taking $\rho=0.5$ instead of $\rho=0.2$) can lead to a large improvement in difficult bandit problems, but also a large decrement simple bandit problems. Besides, for large values of $\rho$ or for simple bandit problems, {\sc ucb1}($\rho$) and {\sc ucb-h}($\rho$) behave similarly and thus, there is not much gain in using {\sc ucb-h} policies instead of {\sc ucb1} policies.

\subsection{The gain of knowing the mean reward of the optimal arm}

When the mean reward $\mu^*$ of the optimal arm is known, there is a strong gain in using this information to design the policy. In all our experiments comparing the expected regret of policies, summarized in Figures \ref{fig:1} and \ref{fig:2}, the parameter-free and anytime policy 
{\sc gcl$^*$} performs a lot better than {\sc ucb1}($\rho$) and {\sc ucb-h}($\rho$), even for the best $\rho$, except in one simulation (for $n=100$, and $K=2$ arms: a Bernoulli distribution of parameter $0.6$ and a Dirac distribution at $0.5$). 
In terms of thinness of the tail distribution of the regret, {\sc gcl$^*$} outperforms all policies in simple bandit problems, while in difficult bandit problems, it generally outperforms {\sc ucb1}($0.2$) and {\sc ucb-h}($0.2$) and performs similarly to {\sc ucb1}($0.5$) and {\sc ucb-h}($0.5$) (see Figures \ref{fig:devbeg} to \ref{fig:devend} in pages 
\pageref{fig:devbeg} to \pageref{fig:devend}). 

The gain of knowing $\mu^*$ is more important than the gain of knowing the horizon.
It is not clear to us that we can have a significant gain in knowing both $\mu^*$ and the horizon $n$ compared to just knowing $\mu^*$ .

\begin{figure}
\begin{center}
\scalebox{0.18}{\includegraphics{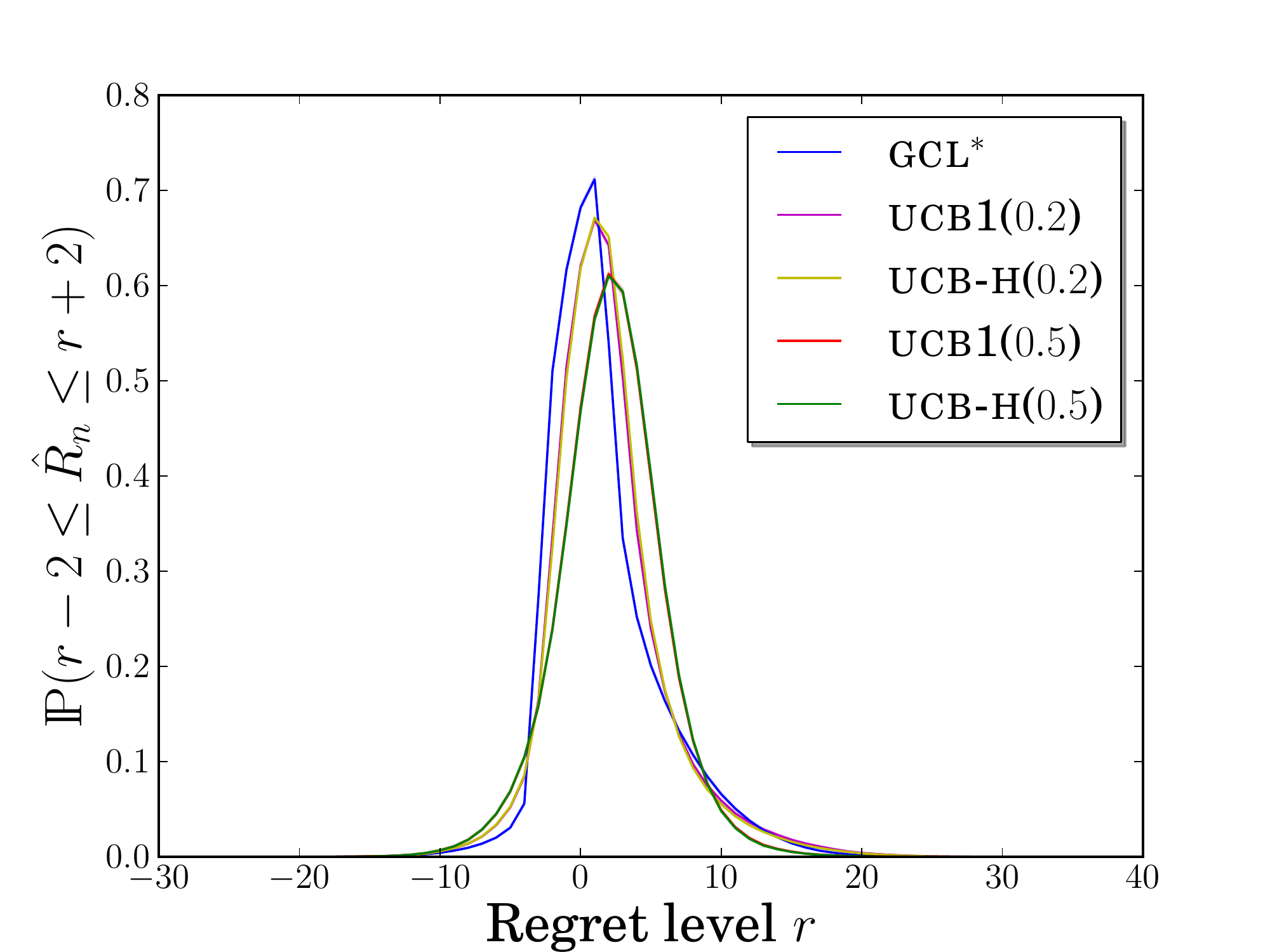}}
\scalebox{0.18}{\includegraphics{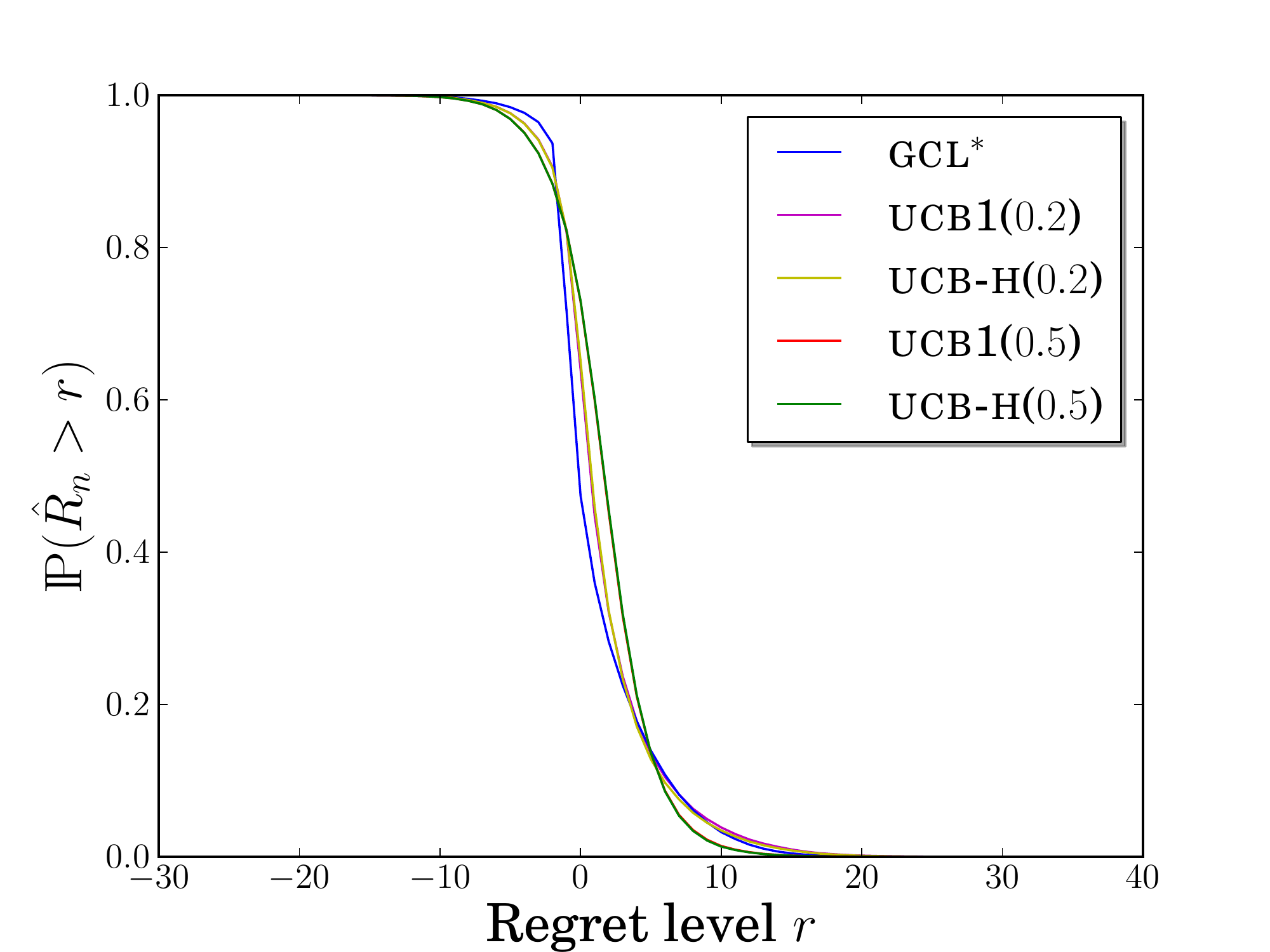}}
\scalebox{0.18}{\includegraphics{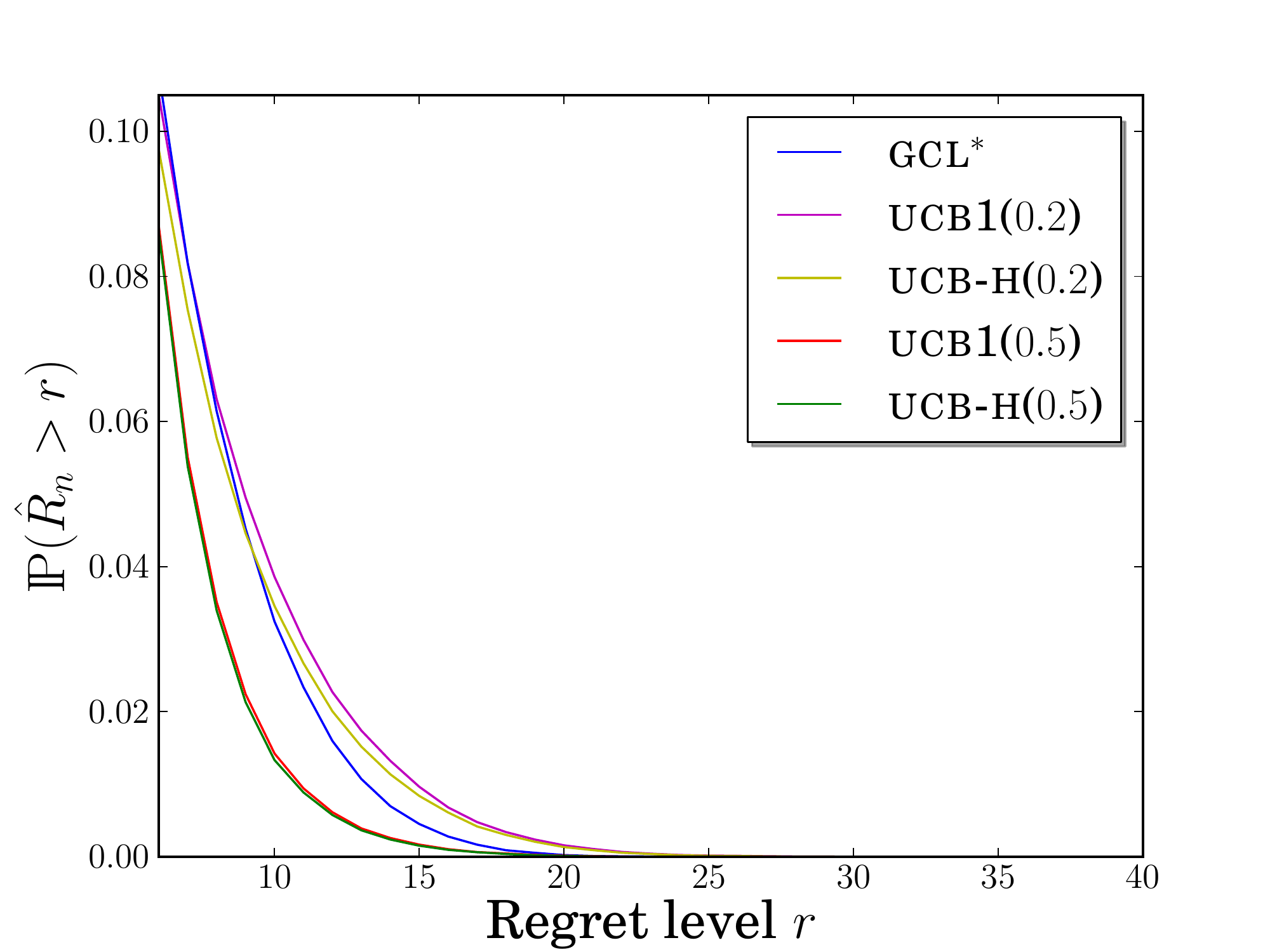}}
\end{center}
\caption{Comparison of policies for n = 100 and K = 2 arms:
Ber(0.6) and Ber(0.5). Left: smoothed probability mass function. Center and right: tail distribution of the regret.}
\label{fig:devbeg}
\end{figure}

\begin{figure}
\begin{center}
\scalebox{0.18}{\includegraphics{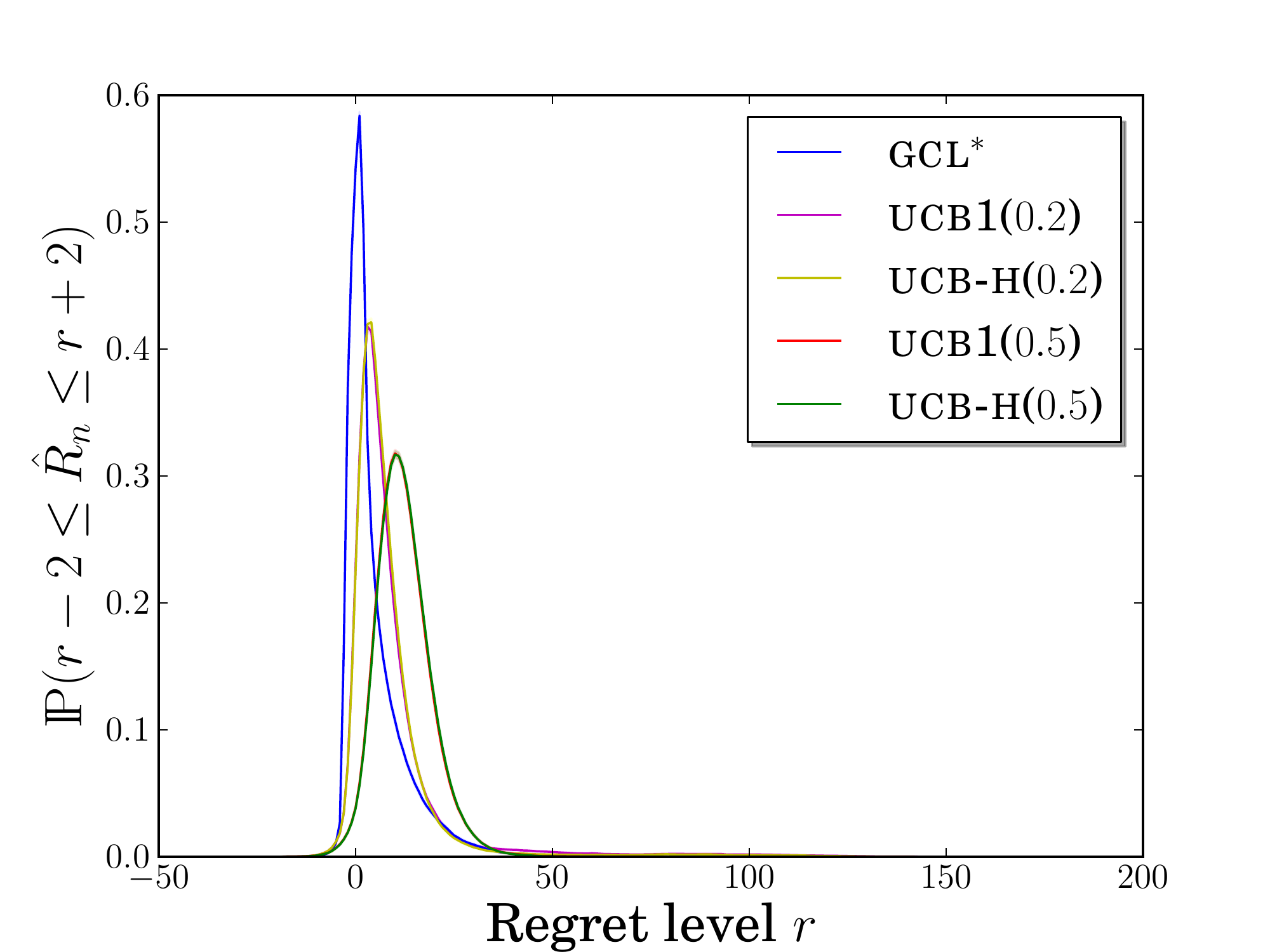}}
\scalebox{0.18}{\includegraphics{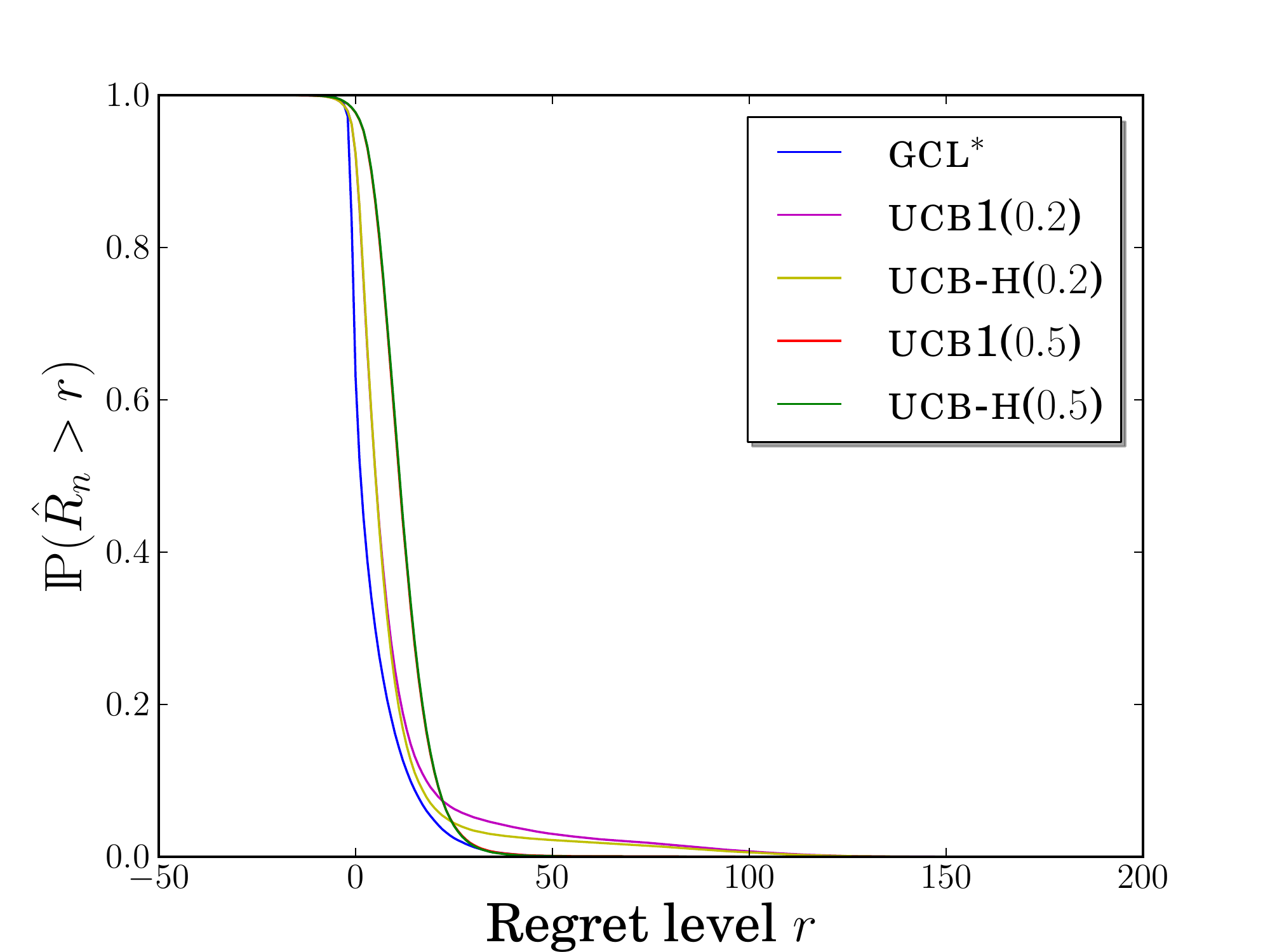}}
\scalebox{0.18}{\includegraphics{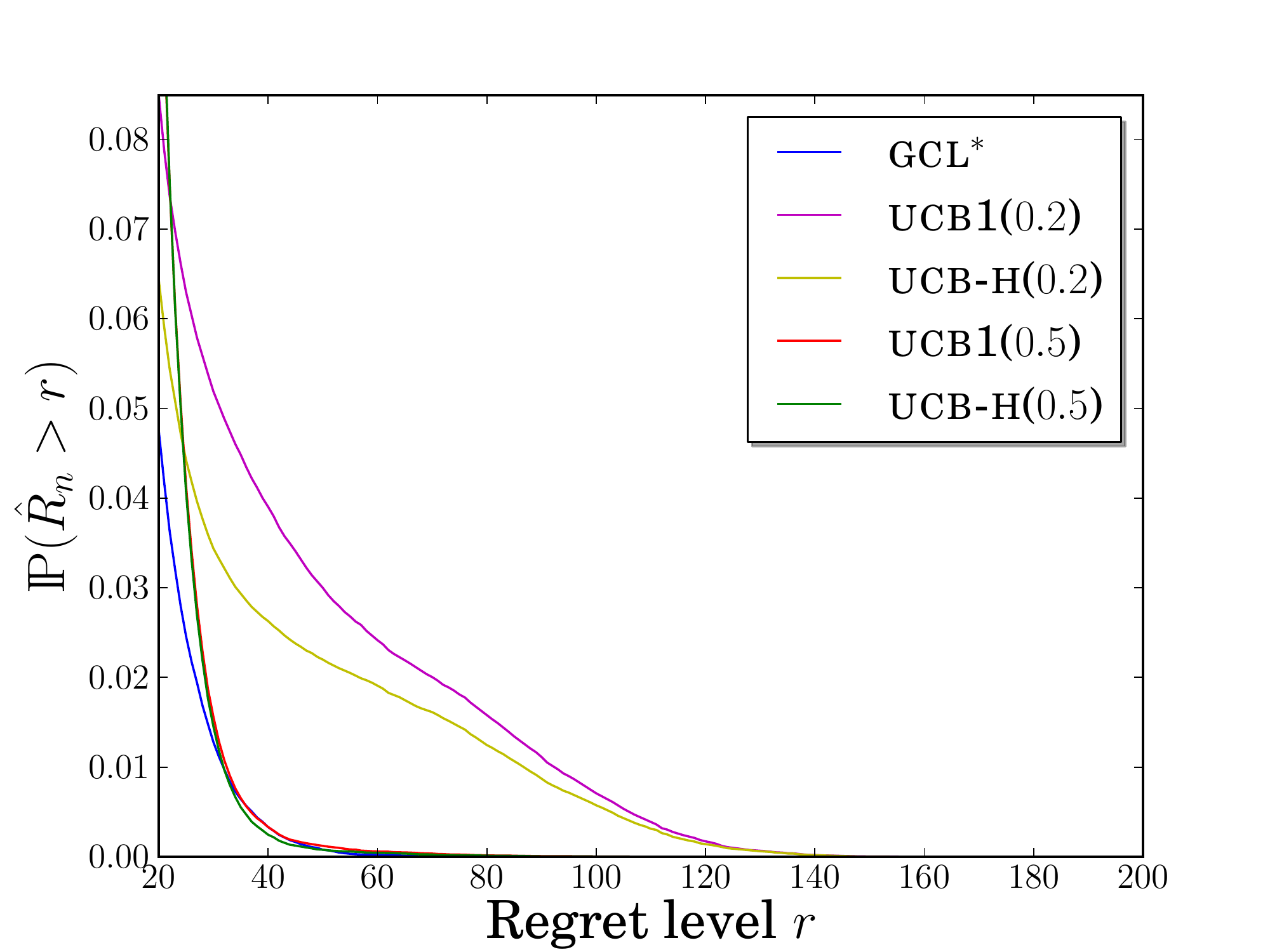}}
\end{center}
\caption{Comparison of policies for n = 1000 and K = 2 arms:
Ber(0.6) and Ber(0.5). Left: smoothed probability mass function. Center and right: tail distribution of the regret.}
\end{figure}

\begin{figure}
\begin{center}
\scalebox{0.18}{\includegraphics{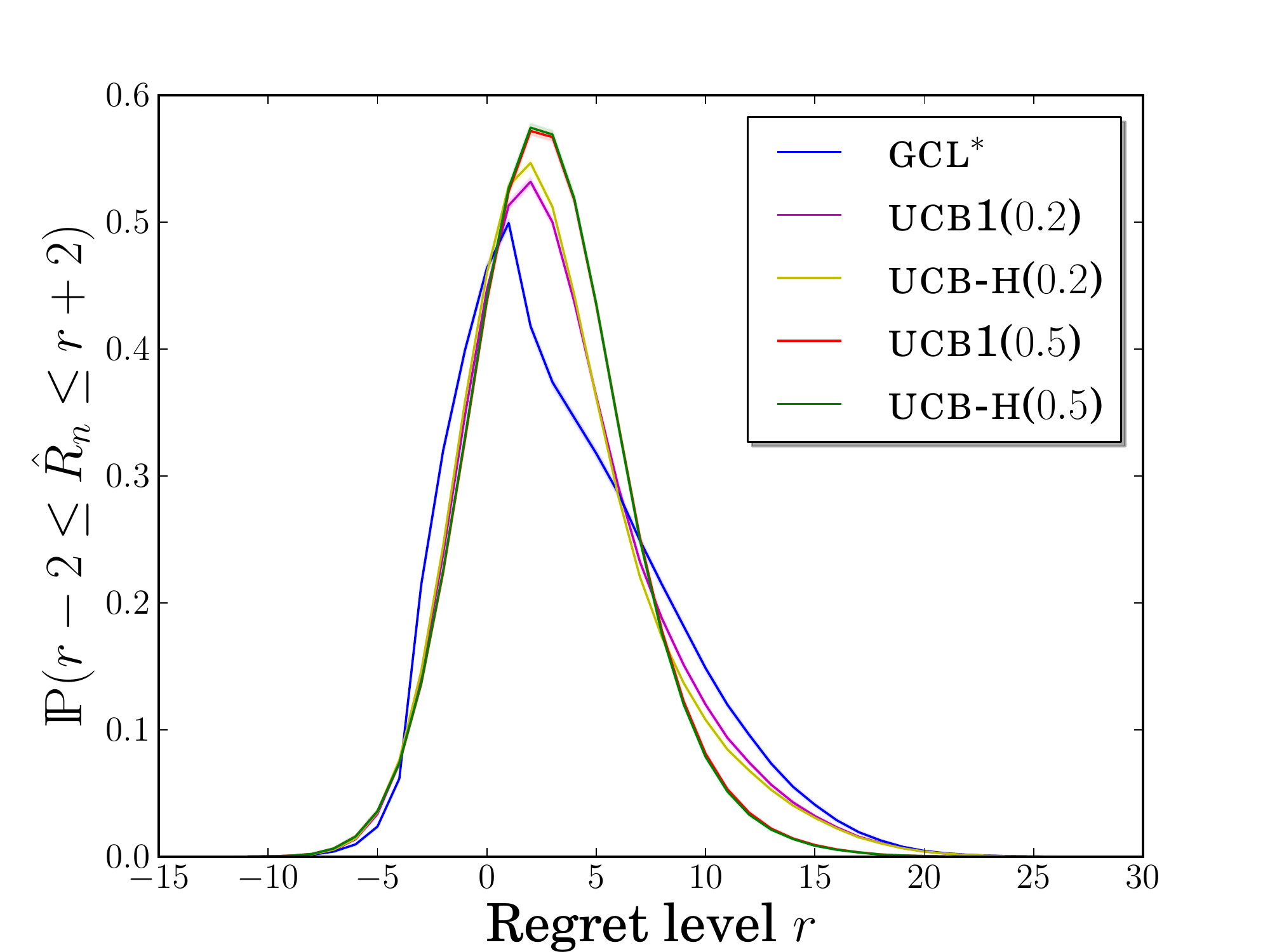}}
\scalebox{0.18}{\includegraphics{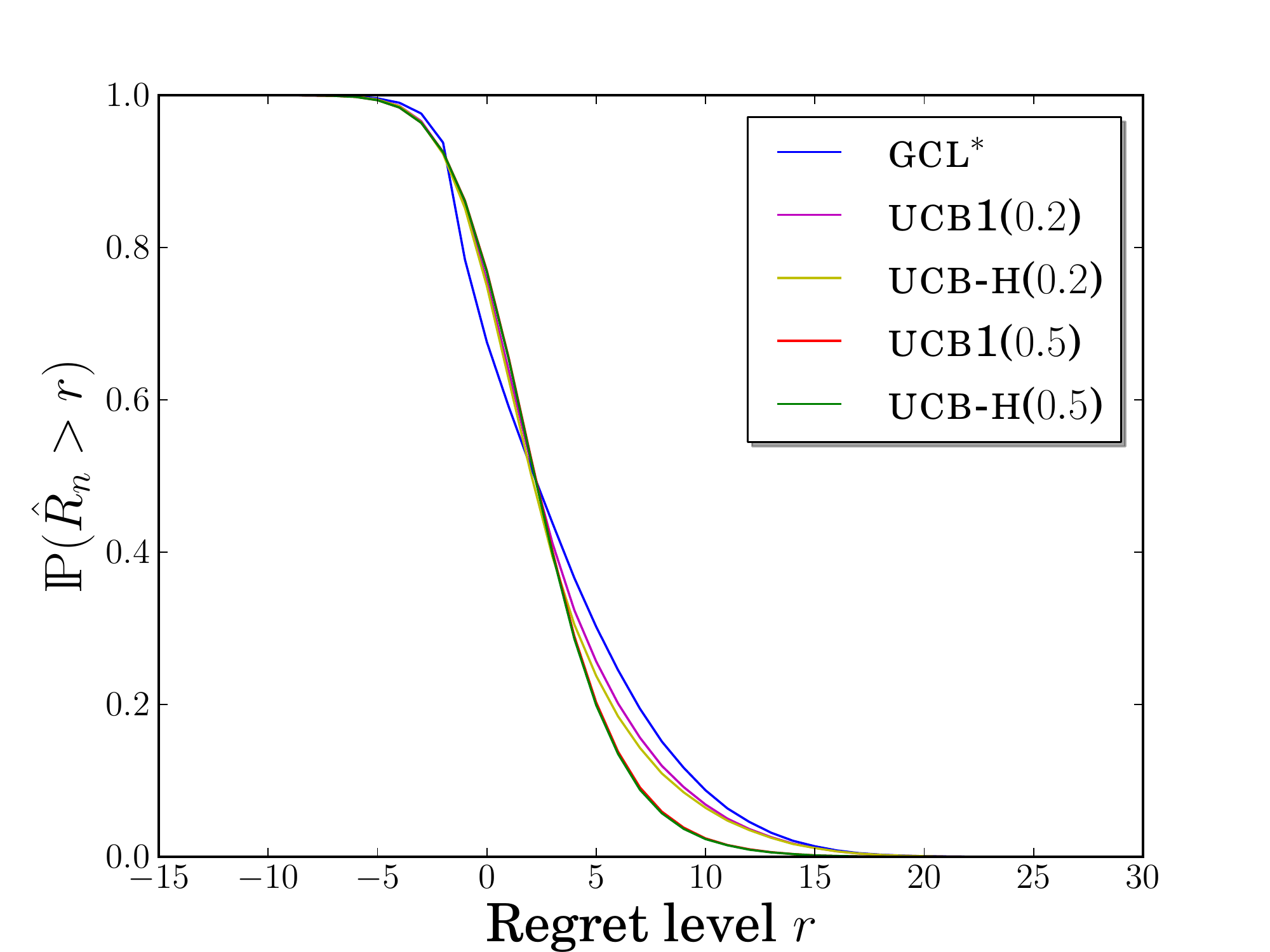}}
\scalebox{0.18}{\includegraphics{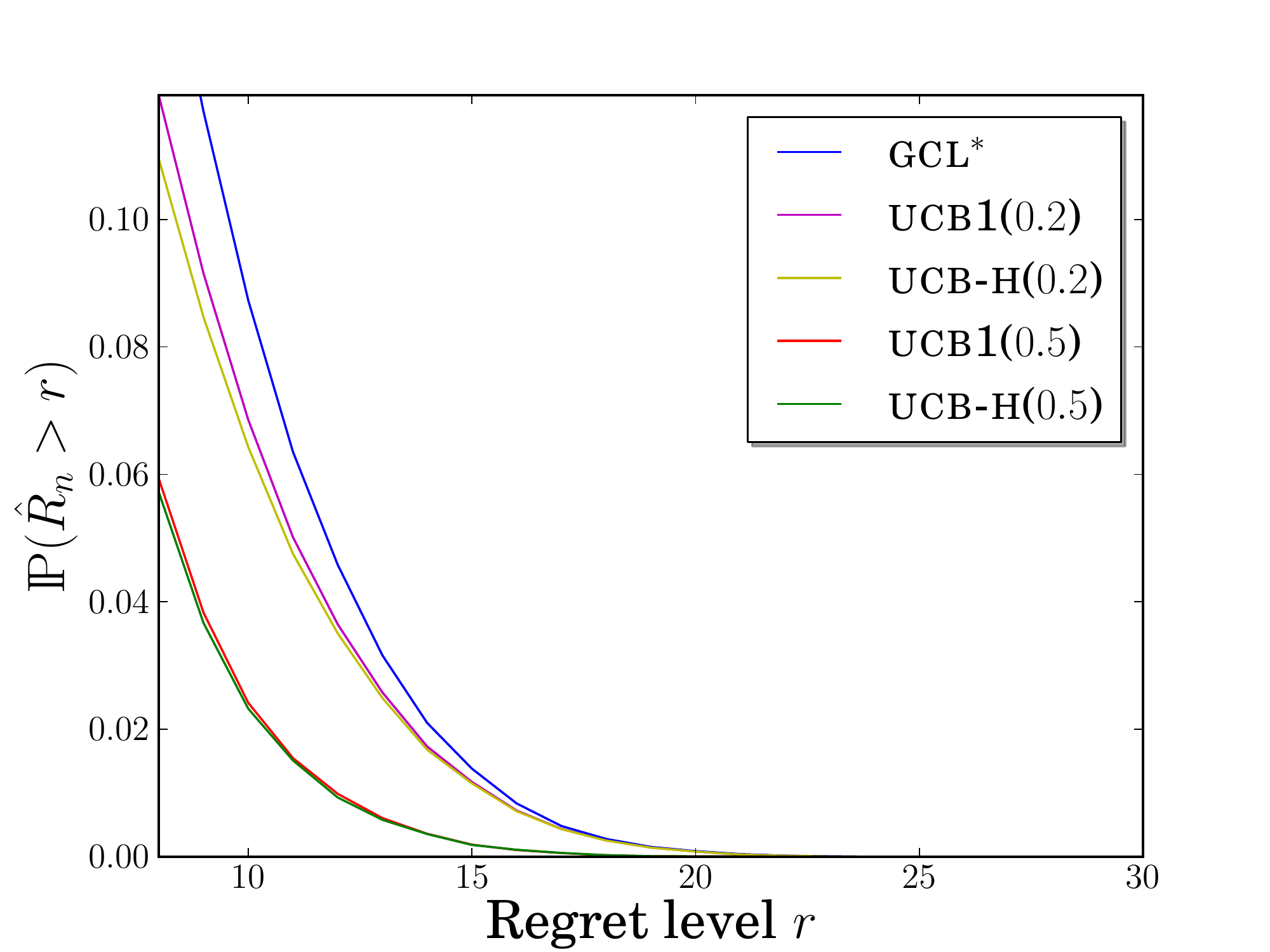}}
\end{center}
\caption{Comparison of policies for n = 100 and K = 2 arms:
Ber(0.6) and Dirac(0.5). Left: smoothed probability mass function. Center and right: tail distribution of the regret.}
\end{figure}

\begin{figure}
\begin{center}
\scalebox{0.18}{\includegraphics{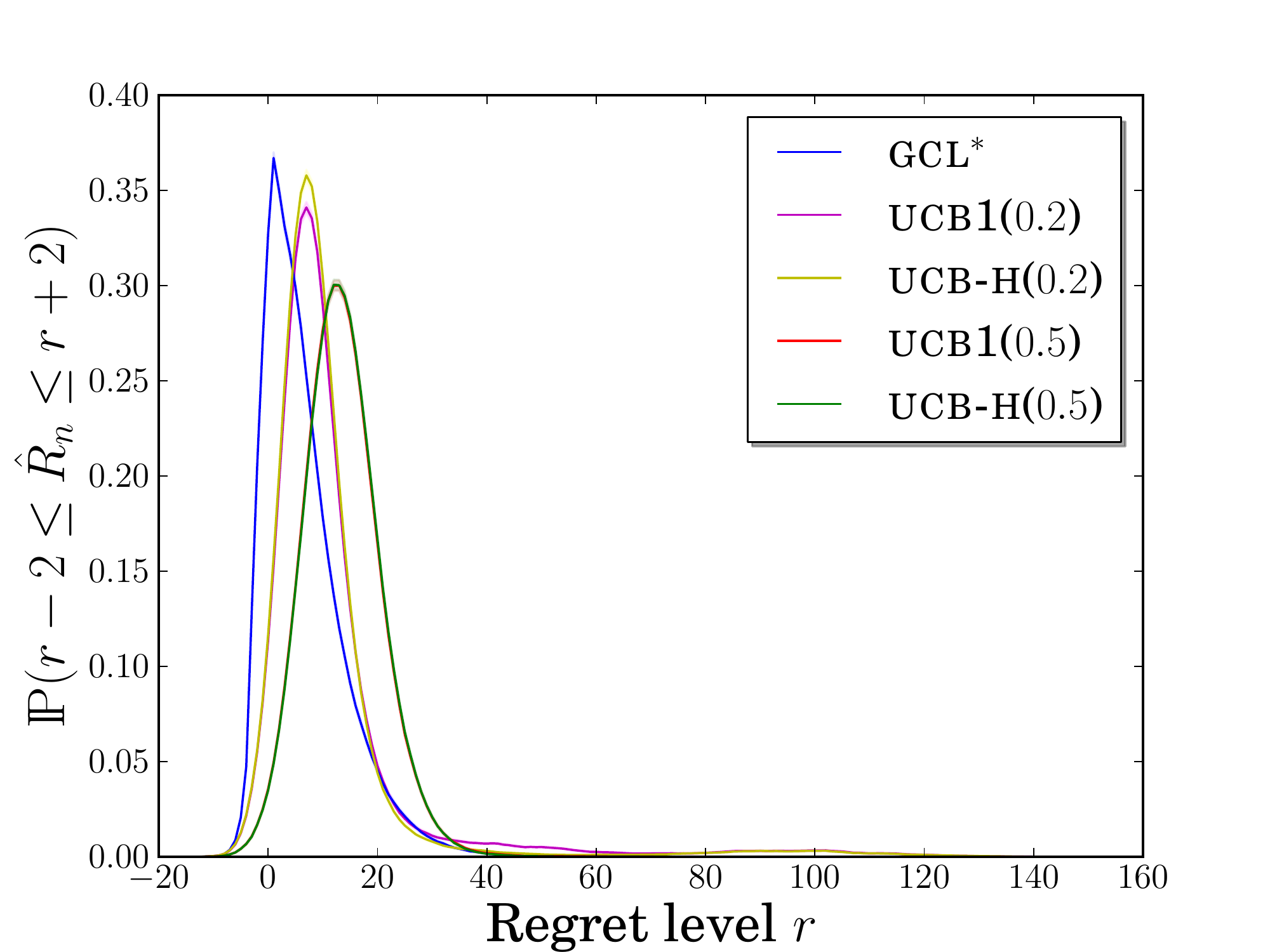}}
\scalebox{0.18}{\includegraphics{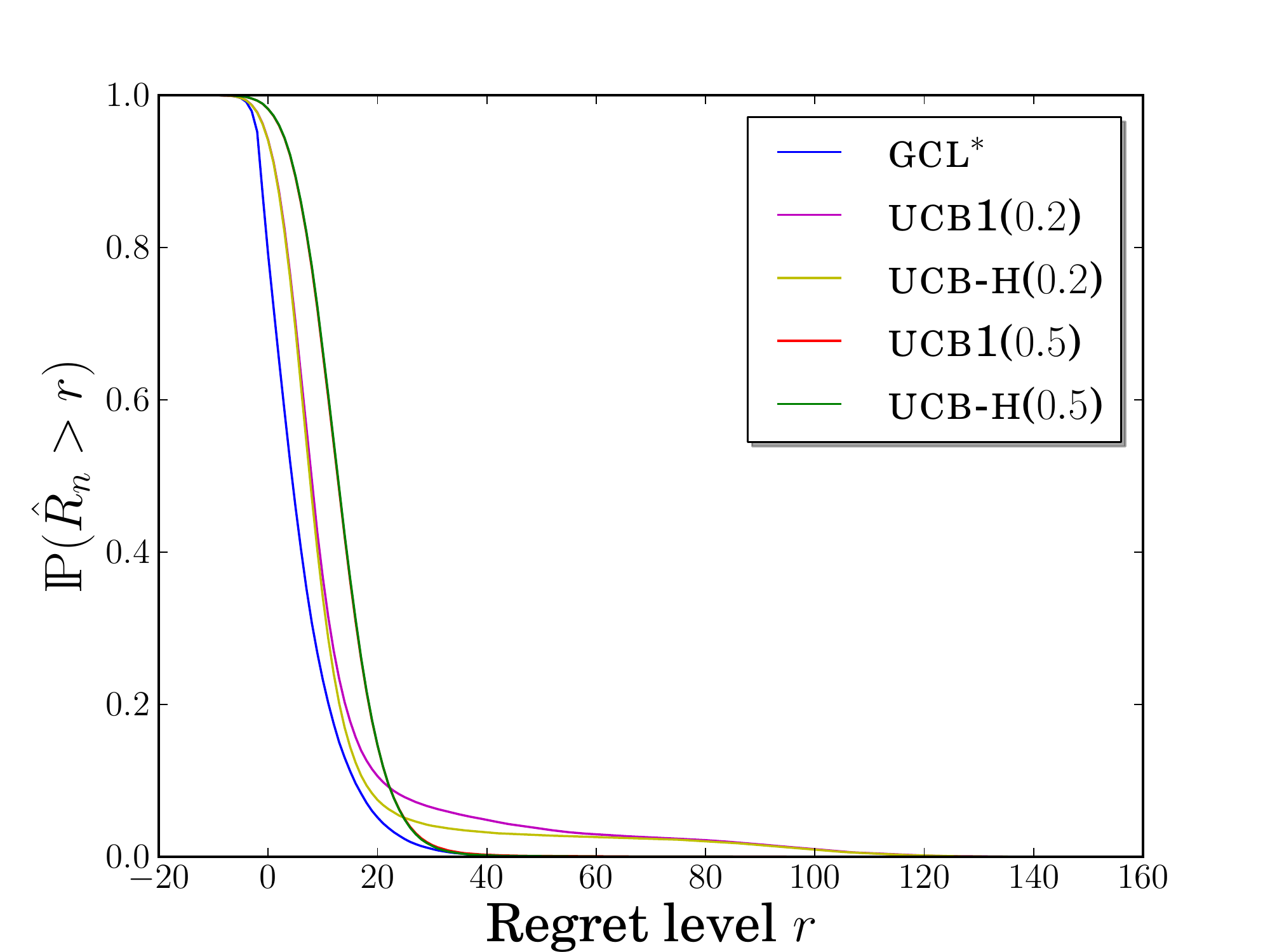}}
\scalebox{0.18}{\includegraphics{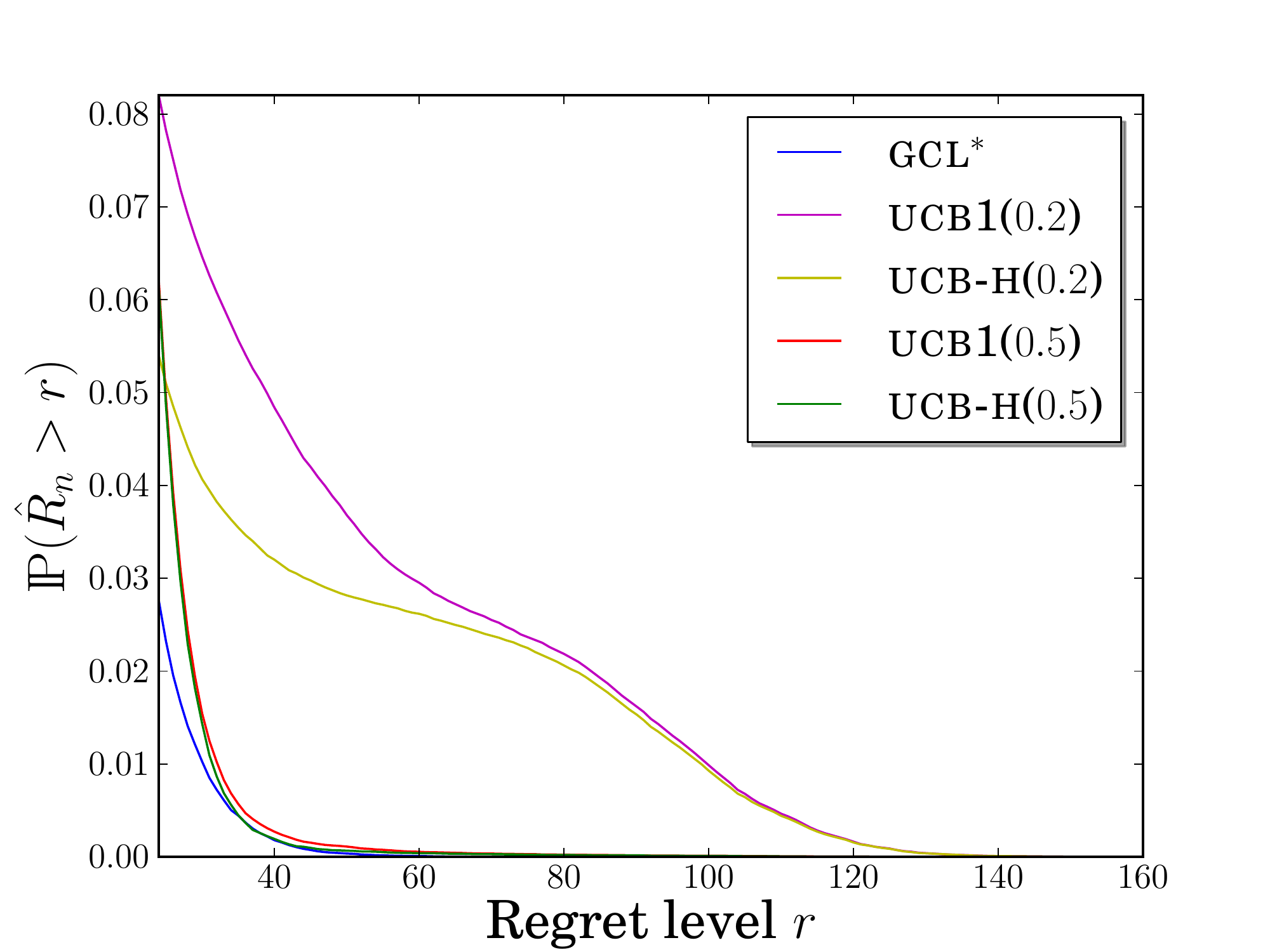}}
\end{center}
\caption{Comparison of policies for n = 1000 and K = 2 arms:
Ber(0.6) and Dirac(0.5). Left: smoothed probability mass function. Center and right: tail distribution of the regret.}
\end{figure}

\begin{figure}
\begin{center}
\scalebox{0.21}{\includegraphics{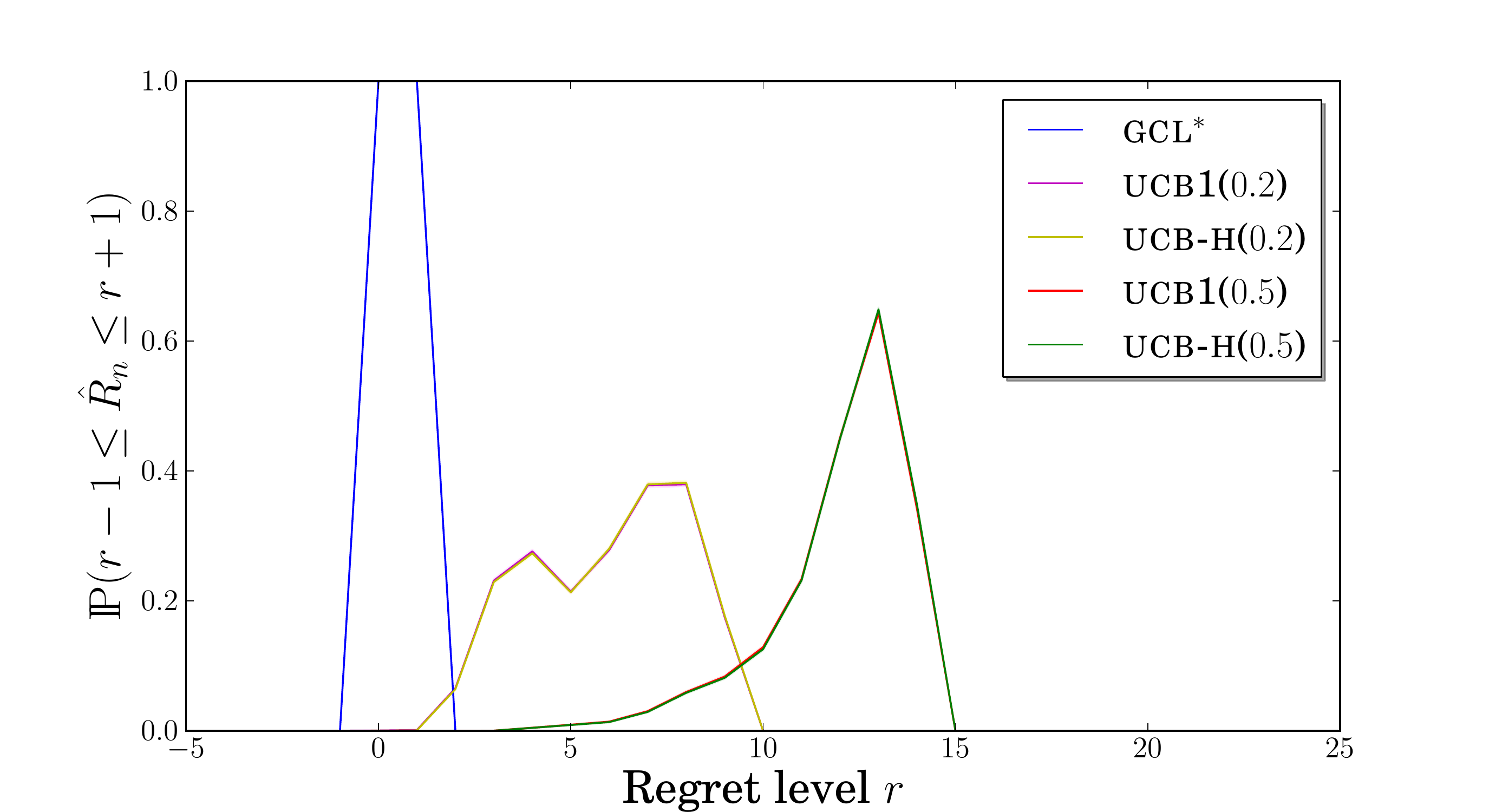}}
\scalebox{0.21}{\includegraphics{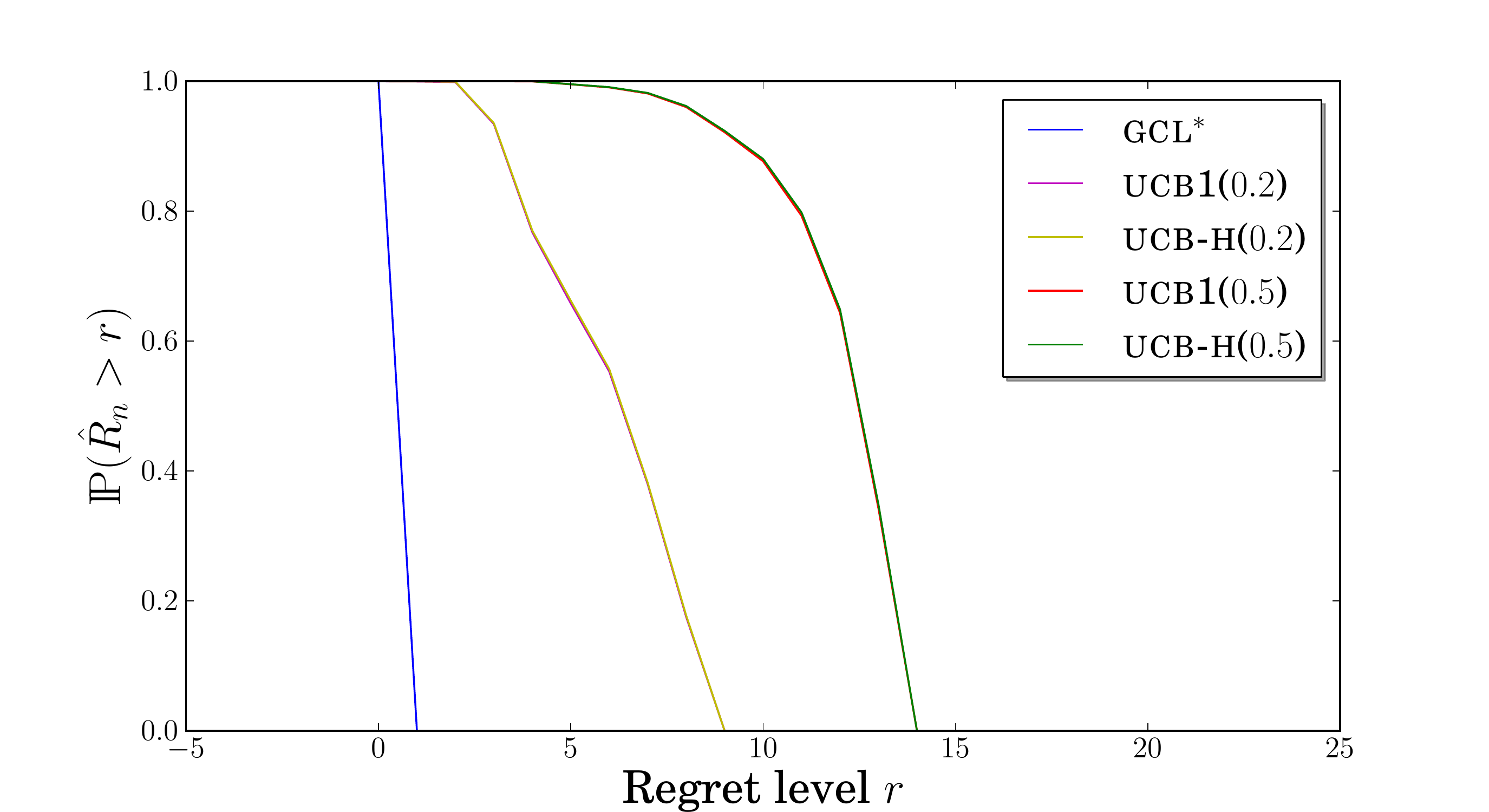}}
\end{center}
\caption{Comparison of policies for n = 1000 and K = 2 arms:
Dirac(0.6) and Ber(0.5). Left: smoothed probability mass function. Right: tail distribution of the regret.
In this simple bandit problem, {\sc ucb1} and {\sc ucb-h} curves are almost identical.}
\end{figure}

\begin{figure}
\begin{center}
\scalebox{0.21}{\includegraphics{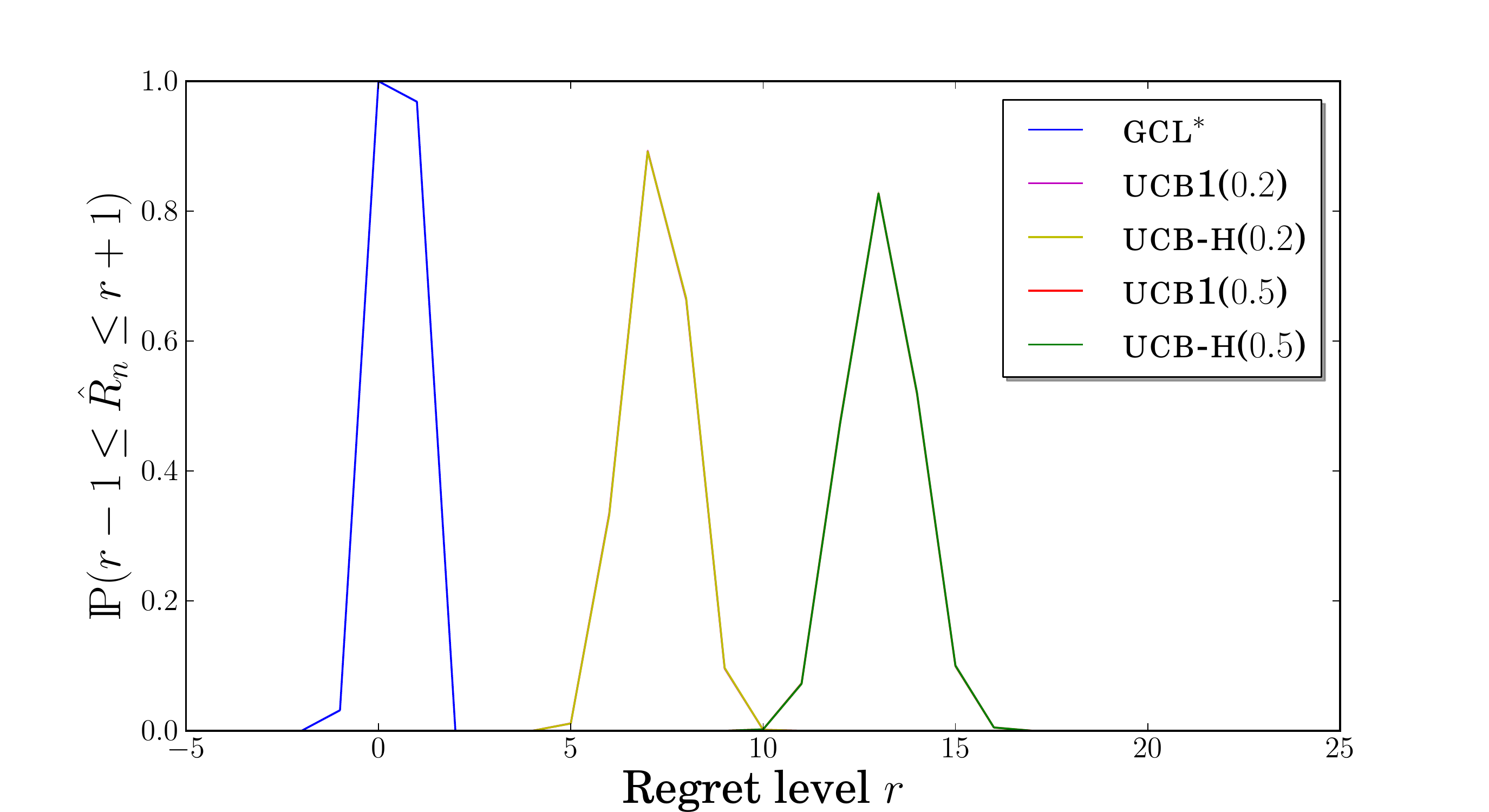}}
\scalebox{0.21}{\includegraphics{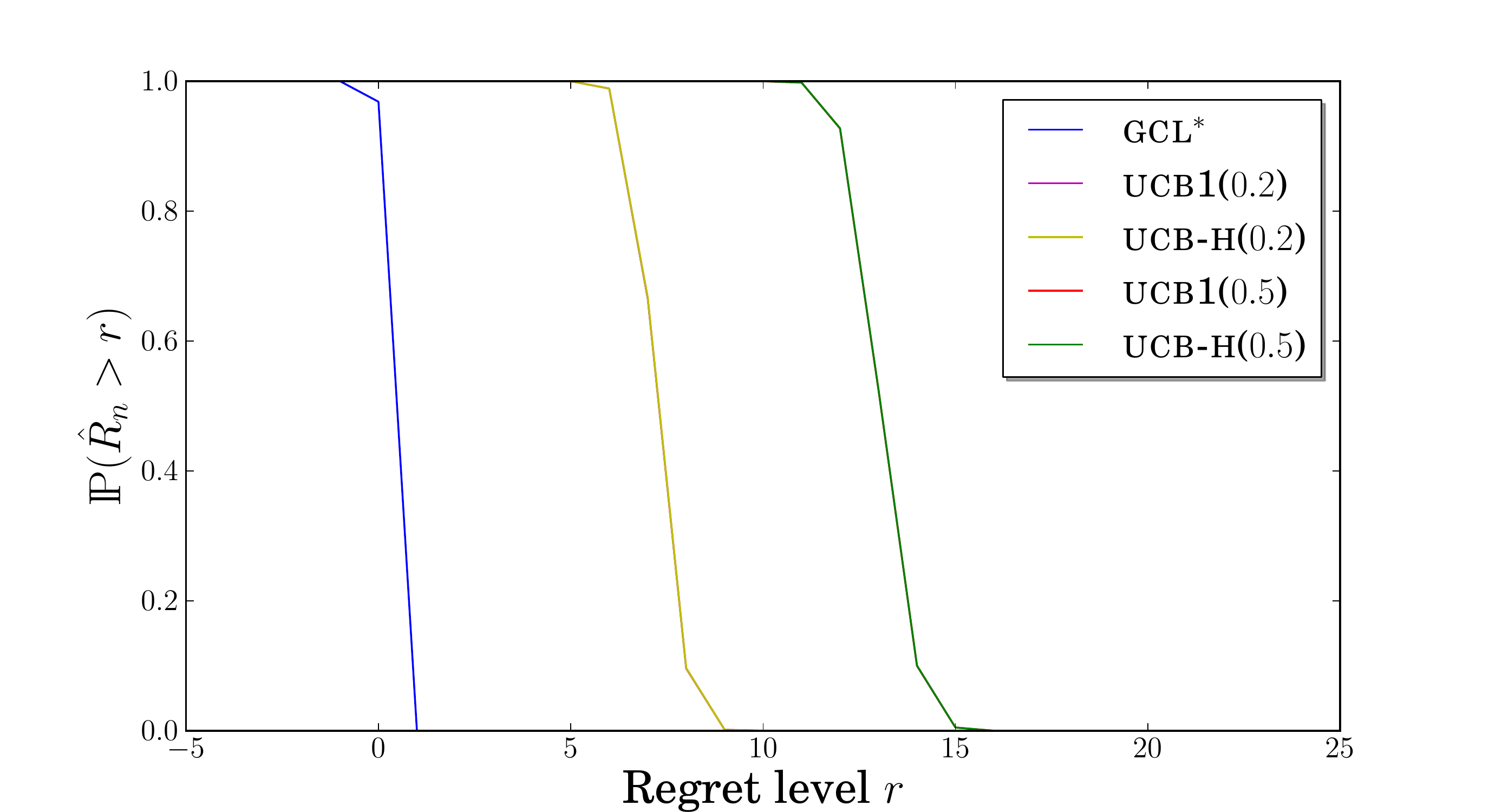}}
\end{center}
\caption{Comparison of policies for n = 1000 and K = 2 arms:
Unif([0.5,0.7]) and Unif([0.4,0.6]).
Left: smoothed probability distribution function. Right: tail distribution of the regret.
In this simple bandit problem, {\sc ucb1} and {\sc ucb-h} curves are almost identical.}
\end{figure}

\begin{figure}
\begin{center}
\scalebox{0.18}{\includegraphics{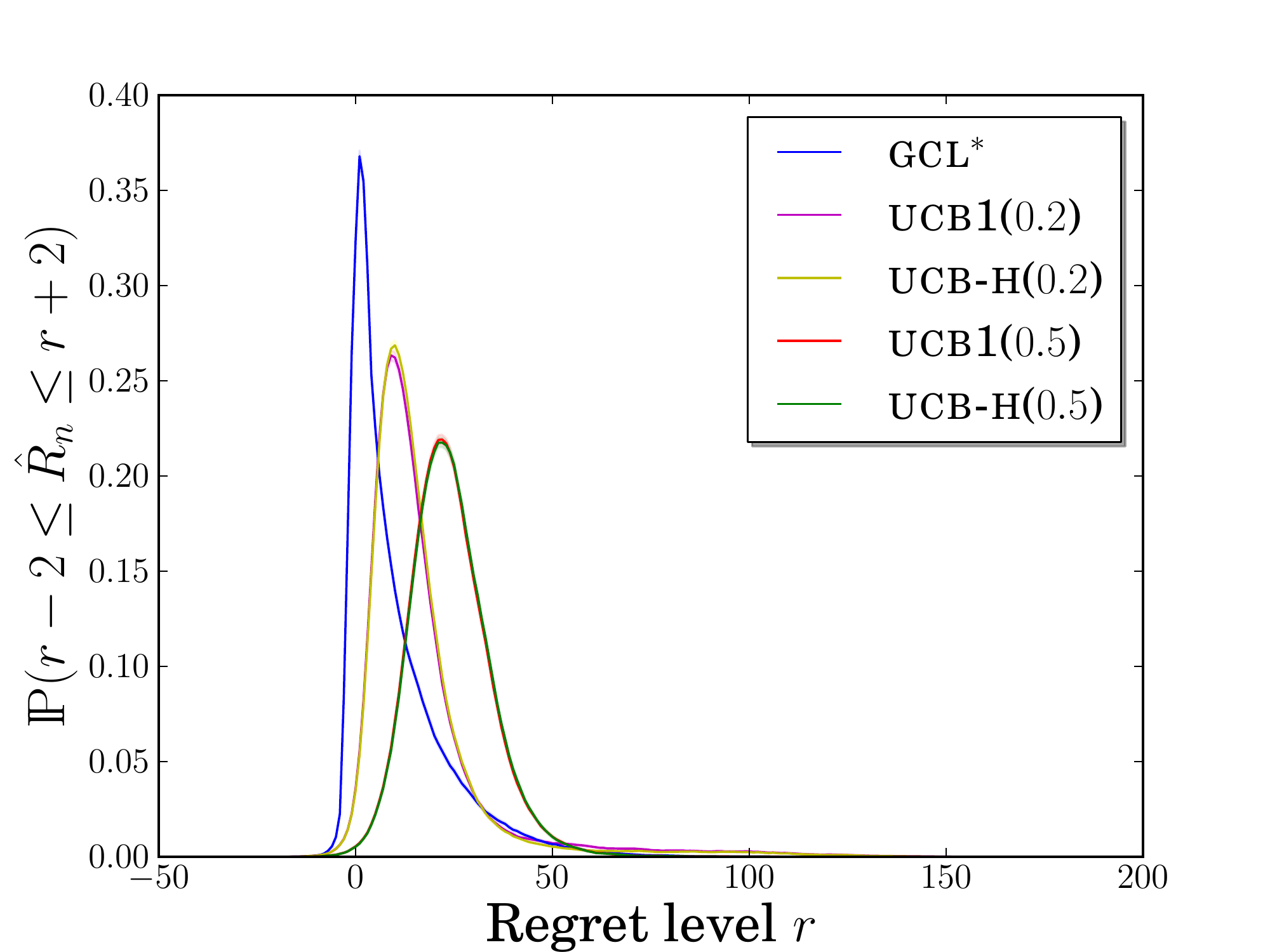}}
\scalebox{0.18}{\includegraphics{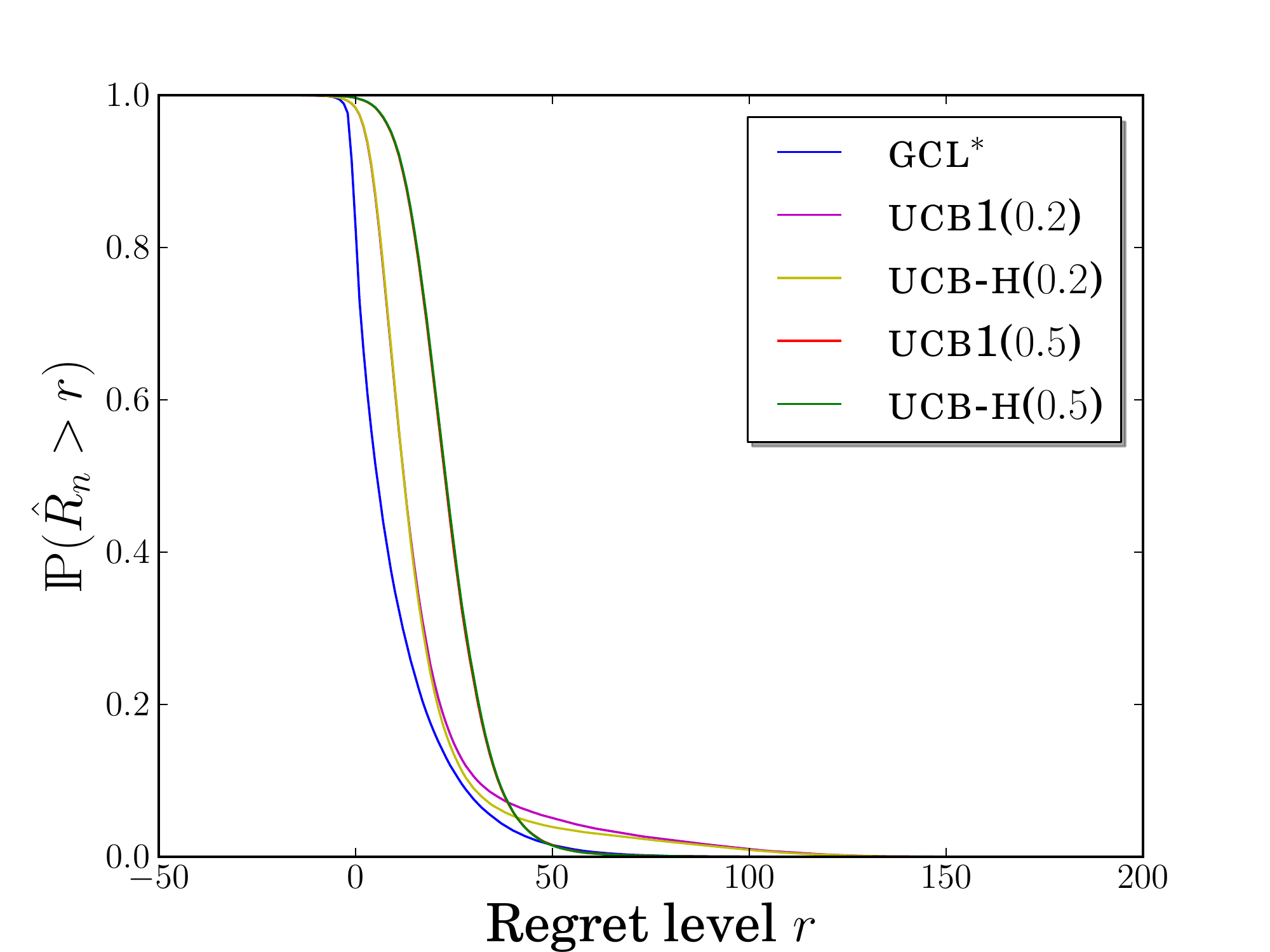}}
\scalebox{0.18}{\includegraphics{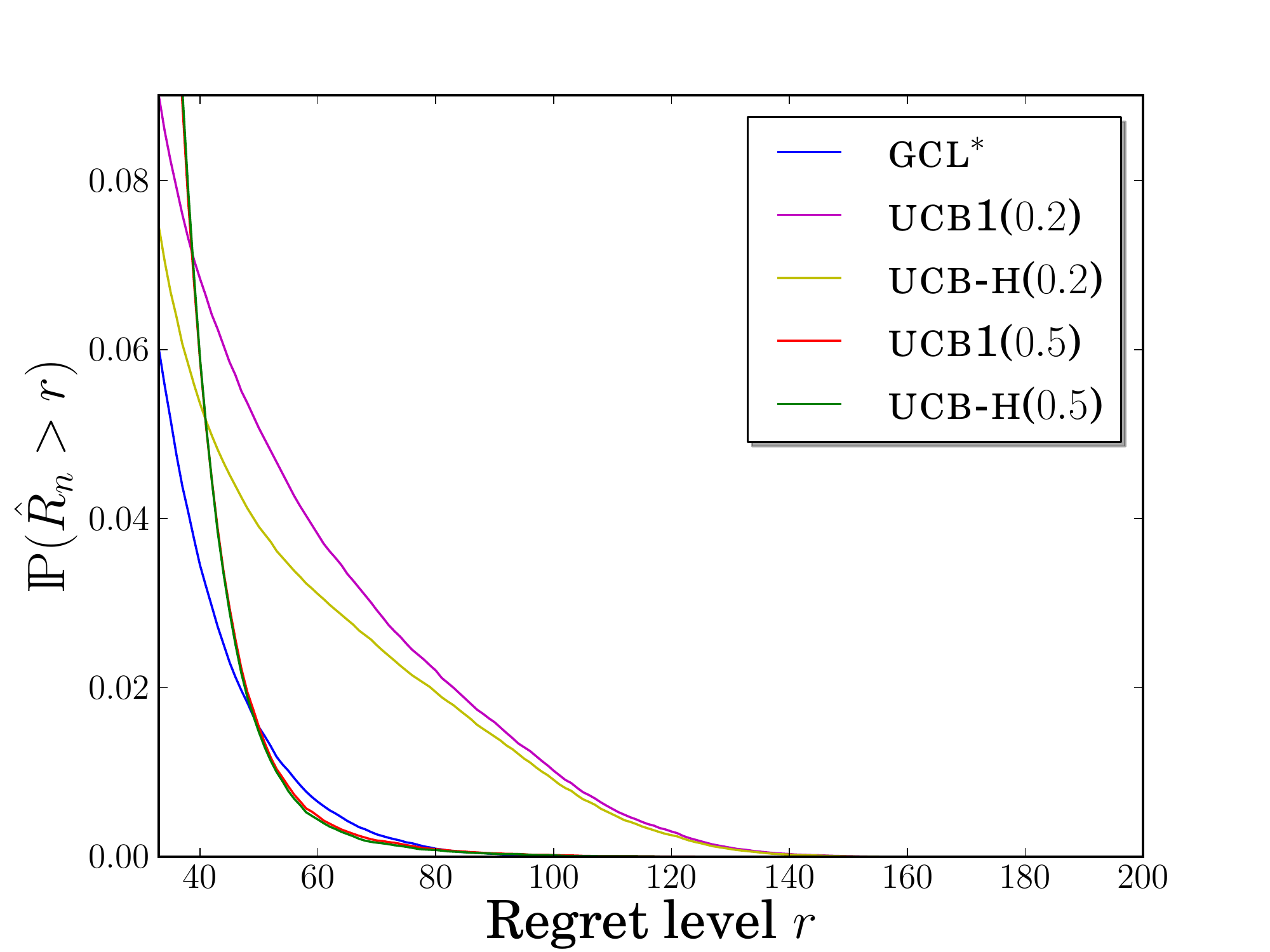}}
\end{center}
\caption{Comparison of policies for n = 1000 and K = 3 arms:
Ber(0.6), Ber(0.5) and Ber(0.5). Left: smoothed probability mass function. Center and right: tail distribution of the regret.}
\end{figure}

\begin{figure}
\begin{center}
\scalebox{0.18}{\includegraphics{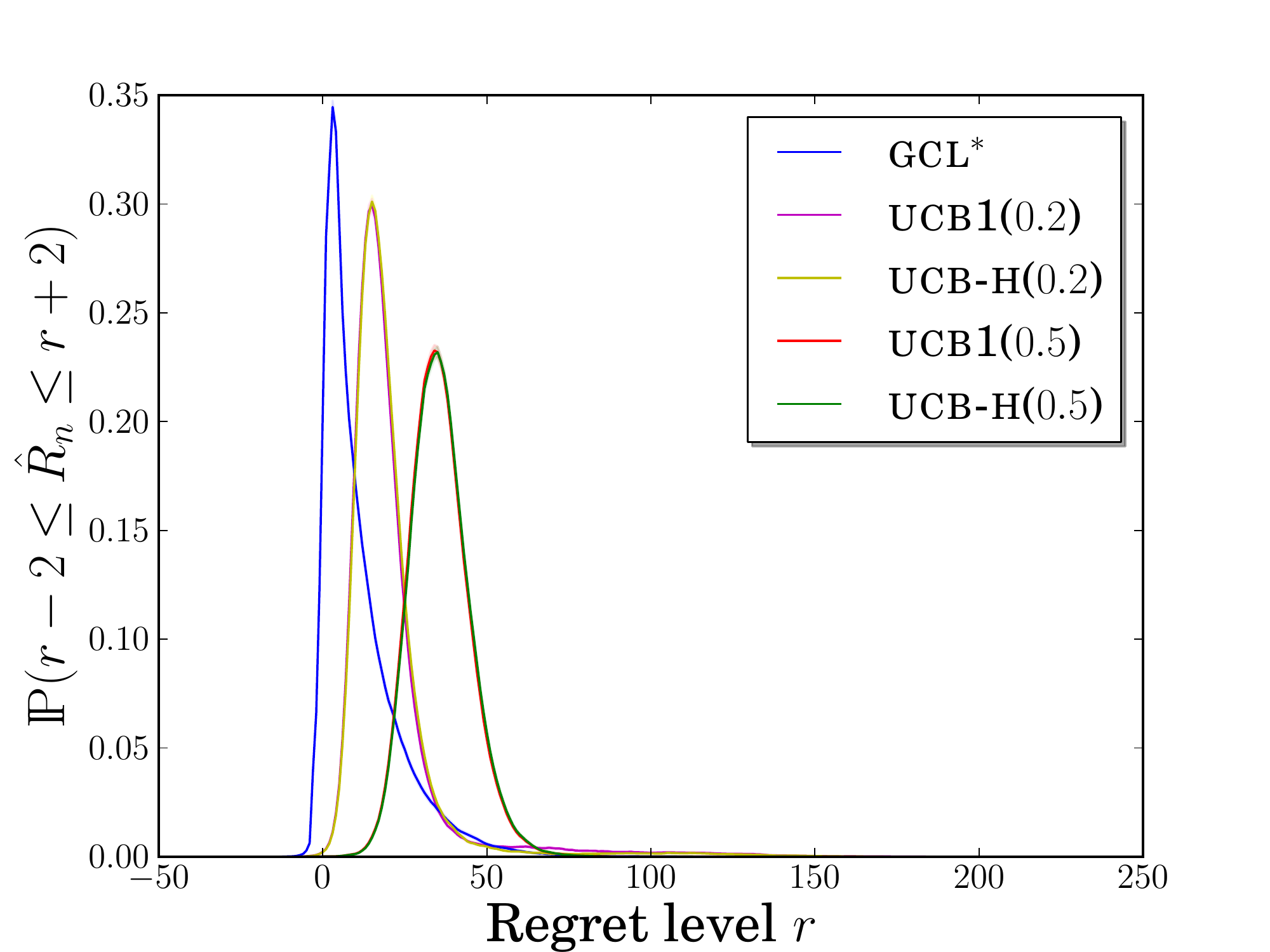}}
\scalebox{0.18}{\includegraphics{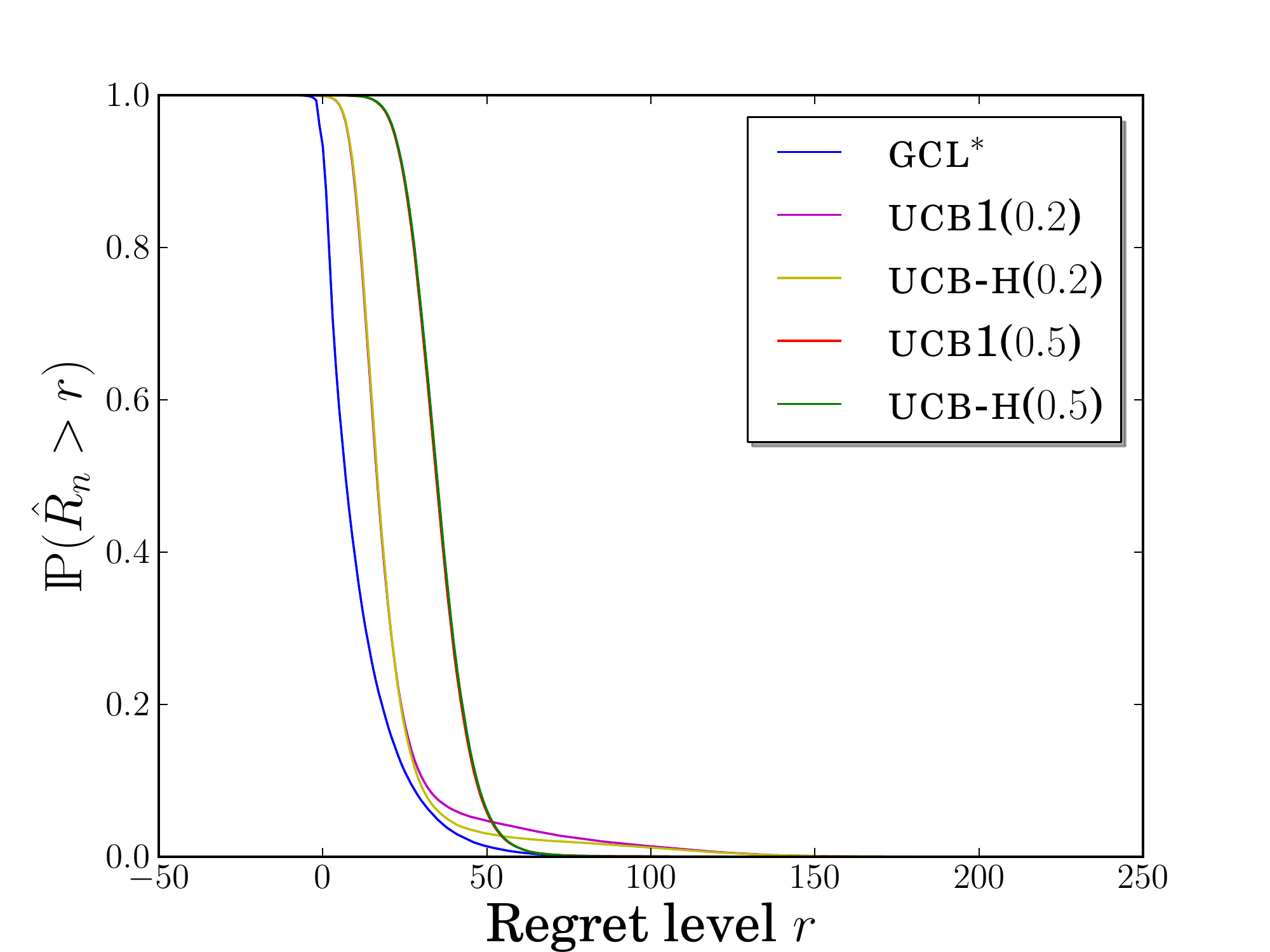}}
\scalebox{0.18}{\includegraphics{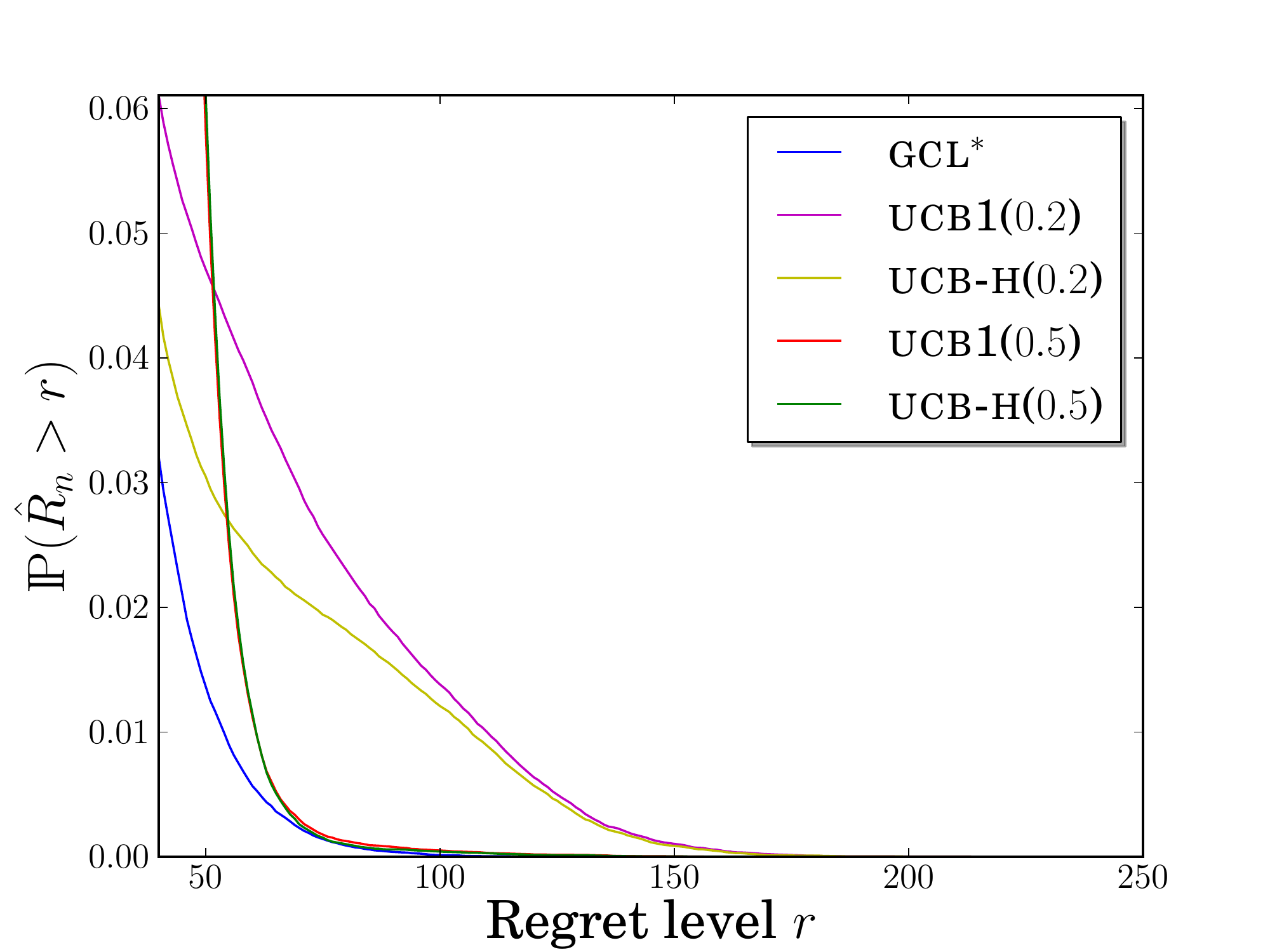}}
\end{center}
\caption{Comparison of policies for n = 1000 and K = 5 arms:
Ber(0.7), Ber(0.6), Ber(0.5), Ber(0.4) and Ber(0.3). Left: smoothed probability mass function. Center and right: tail distribution of the regret.}
\label{fig:devend}
\end{figure}


\section{Proofs}\label{proofs}

\subsection{Proof of Proposition \ref{prop:eqRT}}

$f\text{-}\mathcal T \Rightarrow f\text{-}\mathcal{R}$: 
When a policy is $f$-$\mathcal T$, by a union bound, the event
$$\xi_1=\bigg\{\exists k\in\set{1}K,\ T_k(n) \ge C\frac{\log n}{\Delta_k^2}\bigg\}$$
occurs with probability at most $\frac{K\tilde{C}}{f(n)}$.
Introduce
$
S_{k,s}=\sum_{t=1}^s (X_{k,t}-\mu_k).
$
Since we have
$$
\sum_{t=1}^n X_{I_t,T_{I_t}(t)}
=\sum_{k=1}^K S_{k,T_{k}(n)}+\sum_{k=1}^K T_{k}(n)\mu_k,
$$
we have
    \beglab{eq:creg}
    \hR_n=S_{k^*,n}-S_{k^*,T_{k^*}(n)}-\sum_{k\neq k^*} S_{k,T_{k}(n)} + \sum_{k\neq k^*} \Delta_k T_k(n) .
    \endlab
Let $T=\sum_{k\neq k^*} T_{k}(n)=n-T_{k^*}(n),$ $t^* = \sum_{k\neq k^*} C\frac{\log n}{\Delta_k^2}$, and
$W=\max_{0\le s\le t^*} (S_{k^*,n}-S_{k^*,n-s})$. Since $S_{k^*,n}-S_{k^*,T_{k^*}(n)} \le W$ on the complement $\xi_1^c$ of $\xi_1$,
we have
\beglab{eq:preg}
\hR_n \le n\IND_{\xi_1}+ W -\sum_{k\neq k^*} S_{k,T_{k}(n)} + \sum_{k\neq k^*} \Delta_k T_k(n) .
\endlab
Consider the events
    $$
    \xi_2 = \Bigg\{ W > \sum_{k\neq k^*}\sqrt{\frac{C\beta}{2}}\frac{\log n}{\Delta_k} \Bigg\},
    $$
    $$
    \xi_{3,k} = \Bigg\{ \max_{1\leq s\leq \frac{C\log n}{\Delta_k^2}}
    (-S_{k,s}) > \sqrt{\frac{C\beta}{2}}\frac{\log n}{\Delta_k} \Bigg\},
    $$
and
    $$
    \xi = \xi_1\cup \xi_2 \und{\cup}{k\neq k^*} \xi_{3,k}.
    $$
From Hoeffding's maximal inequality, we have
    \begin{align*}
    \P_{\theta}(\xi_2) & \le \exp\Bigg( - \frac{2 \big(\sum_{k\neq k^*} \sqrt{\fracc{C\beta}{2}}\frac{\log n}{\Delta_k}\big)^2} 
        {\sum_{k\neq k^*} \fracl{C\log n}{\Delta_k^2}}  \Bigg)
 \le\exp\big( - \beta\log n\big) = \frac1{n^\beta} \le \frac{\alpha}{f(n)}.
    \end{align*}
We also use Hoeffding's maximal inequality to control $\P_{\theta}(\xi_{3,k})$:
    $$
    \P_{\theta}(\xi_{3,k}) \le \exp\Bigg( - \frac{2 \big( \sqrt{\fracc{C\beta}{2}}\frac{\log n}{\Delta_k}\big)^2} 
        { \fracl{C\log n}{\Delta_k^2}}  \Bigg) = \frac1{n^\beta} \le \frac{\alpha}{f(n)}.
    $$
By gathering the previous results using a union bound, we have 
    $\P(\xi)\le \frac{2\alpha+\tC}{f(n)}$.
Besides on the complement of $\xi$, by using \eqref{eq:preg}, we have
    $$
    \hR_n < \sum_{k\neq k^*} \sqrt{\frac{C\beta}{2}}\frac{\log n}{\Delta_k}+\sum_{k\neq k^*} \sqrt{\frac{C\beta}{2}}\frac{\log n}{\Delta_k}
        + \sum_{k\neq k^*} \frac{C\log n}{\Delta_k}.
    $$
We have thus proved that 
    $$
    \forall \theta\in\Theta, \forall n\geq 1, \ \Pro_{\theta}\left(\hR_n\geq (C+\sqrt{2C\beta}) \frac{\log n}{\Delta}\right)
        \leq \frac{\tilde{C}+2\alpha}{f(n)},
    $$
hence the policy is $f$-$\mathcal R$.   

$f\text{-w}\mathcal{T} \Rightarrow f\text{-w}\mathcal{R}$: it is exactly the same proof as for $f\text{-}\mathcal{T}\Rightarrow f\text{-}\mathcal{R}$ since
the core of the argument is independent of the position of ``$\forall \theta$'' with respect to ``$\exists C,\tilde{C}$''.

$f\text{-w}\mathcal{R} \Rightarrow f\text{-w}\mathcal{T}$: let us prove the contrapositive. So we assume 
\beglab{eq:contra}
    \exists \theta \in\Theta \text{ such that }\Delta\neq 0, \ \forall C',\tilde{C}'>0, \ \exists n\geq 1, \ \exists k\neq k^*, \ 
    \P_{\theta}\left(T_k(n)\geq C'\frac{\log n}{\Delta_k^2}\right)> \frac{\tilde{C}'}{f(n)}.
\endlab 
It is enough to prove that for this $\theta$, we have
    $$
    \forall C > 9K/\Delta ,\ \forall \tilde{C}>\alpha, \ \exists n\geq 1,  
    \P_{\theta}\left(\hR_n\ge C\frac{\log n}{\Delta}\right)> \frac{\tC}{f(n)}.
    $$
To achieve this, we consider $C'= (\beta+2)C/\Delta$ and $\tC'=\max\big(2\tC,\max_{m\le K} f(m)\big)$ in \eqref{eq:contra} and let $k'\neq k^*$ be such that 
the event
    $$\xi'=\bigg\{ T_{k'}(n)\geq C'\frac{\log n}{\Delta_{k'}^2} \bigg\}$$
holds with probability greater than $\fracc{\tilde{C}'}{f(n)}=\fracc{2\tC}{f(n)}$.
From \eqref{eq:contra} and using $\tC'\ge \max_{m\le K} f(m)$, we necessarily have $n\ge K$.
Let $L=\log\big(\frac{f(n)}{\tC}nK\big)$ and
    \begin{multline*}
    \xi''= \Bigg\{ \forall k \neq k^*, \ \forall s\in\set{1}{n},\ |S_{k,s}|\le \sqrt{\frac{sL}2} \Bigg\}
        \bigcap \Bigg\{ \forall s\in\set{1}{n},\ |S_{k^*,n}-S_{k^*,n-s}|\le \sqrt{\frac{sL}2} \Bigg\}.
    \end{multline*}
By Hoeffding's inequality and a union bound, this event holds with probability at least $1-\fracc{\tC}{f(n)}$.
As a consequence, we have $\P(\xi'\cap\xi'') > \fracc{\tC}{f(n)}$. We now prove that on the event $\xi'\cap\xi''$, we have
    $$
    \hR_n \ge C\frac{\log n}{\Delta}.
    $$
First note that for any $a>0$ the function $s\mapsto a s -\sqrt{2sL}$ is decreasing on $\big[0,\frac{L}{2a^2}\big]$
and increasing on $\big[\frac{L}{2a^2},+\infty\big)$, and that 
    $$
    T_{k'}(n)\geq C'\frac{\log n}{\Delta_{k'}^2} \ge \frac{CL}{\Delta_{k'}^2}, 
    $$
since $\frac{f(n)}{\tC}nK \le \frac{\alpha n^{\beta}}{\alpha} n^2 = n^{\beta+2} \le n^{C'/C}$.
Then, by using \eqref{eq:creg} and $T_{k^*}(n)=n-\sum_{k\neq k^*} T_{k}(n)$, we have
    \begin{align*}
    \hR_n & \ge -|S_{k^*,n}-S_{k^*,T_{k^*}(n)}|-\sum_{k\neq k^*} |S_{k,T_{k}(n)}| + \sum_{k\neq k^*} \Delta_k T_k(n)\\
    & \ge - \sqrt{\frac{L\sum_{k\neq k^*} T_{k}(n)}2}
    - \sum_{k\neq k^*} \sqrt{\frac{LT_{k}(n)}2} + \sum_{k\neq k^*} \Delta_k T_k(n)\\
    & \ge \sum_{k\neq k^*} \Bigg( \Delta_k T_k(n) - \sqrt{2T_{k}(n)L} \Bigg)\\
    & \ge \frac{\Delta_{k'} T_{k'}(n)}2 + \bigg( \frac{\Delta_{k'} T_{k'}(n)}2 - \sqrt{2LT_{k'}(n)} \bigg)
        + \sum_{k\neq k^*,k\neq k'} \und{\min}{s\ge 1}\Bigg( \Delta_k s - \sqrt{2Ls} \Bigg)\\
    & \ge C'\frac{\log n}{2\Delta_{k'}} + \bigg(\frac{C}2 - \sqrt{2C}\bigg)\frac{L}{\Delta_{k'}}- \sum_{k\neq k^*,k\neq k'} \frac{L}{2\Delta_k} \\
    & \ge C'\frac{\log n}{2\Delta_{k'}} + \frac{C}6 \frac{L}{\Delta_{k'}}- \frac{KL}{2\Delta} 
    \ge C'\frac{\log n}{2\Delta_{k'}} \ge C\frac{\log n}{\Delta},
    \end{align*}
which ends the proof of the contrapositive.

\subsection{Proof of Theorem \ref{th:main}}

Let us first notice that we can remove the $\Delta_k^2$ denominator in the the definition of $f$-w$\mathcal T$ without loss of generality. This would not be possible for the $f$-$\mathcal T$ definition owing to the different position of ``$\forall \theta$'' with respect to ``$\exists C,\tilde{C}$''.\\
Thus, a policy is $f$-w$\mathcal T$ if and only if
$$\forall \theta\in\Theta \text{ such that }\Delta\neq 0, \ \exists C,\tilde{C}>0, \ \forall n\geq 2, \ \forall k\neq k^*, \ \Pro_{\theta}\left(T_k(n)\geq C{\log n}\right)\leq \frac{\tilde{C}}{f(n)}.$$

Let us assume that the policy has the $f$-upper tailed property in $\theta$, i.e., there exists $C,\tilde{C}>0$
\begin{equation}\label{eq:uptheta}
\forall N\geq 2, \ \forall \ell\neq k, \ \Pro_{\theta}\big(T_\ell(N)\geq C\log N\big)\leq \frac{\tilde{C}}{f(N)}.
\end{equation}
Let us show that this implies that the policy cannot have also the $f$-upper tailed property in $\tth$.
To prove the latter, it is enough to show that for any $C',\tilde{C}'>0$
\begin{equation}\label{eq:uptth}
\exists n\geq 2, \ \Pro_{\tth}\big(T_k(n)\geq C'\log n\big)> \frac{\tilde{C}'}{f(n)}.
\end{equation}
since $k$ is suboptimal in environment $\tth$. 
Note that proving \eqref{eq:uptth} for $C'= C$ is sufficient. Indeed if \eqref{eq:uptth} holds for $C'=C$, 
it a fortiori holds for $C'<C$. Besides, when $C'>C$, \eqref{eq:uptheta} holds for $C$ replaced by $C'$, and we are thus brought back to the situation when $C=C'$. So we only need to lower bound $\Pro_{\tth}\big(T_k(n)\geq C\log n\big)$.

From Lemma \ref{le:mus}, $\P_{\tth}\big(\frac{d\nu_\ell}{d\tilde{\nu}_\ell}(X_{\ell,1})>0\big)>0$ is equivalent to $\P_{\theta}\big(\frac{d\tnu_\ell}{d\nu_\ell}(X_{\ell,1})>0\big)>0$. 
By independence of $X_{1,1},\dots,X_{K,1}$ under $\P_{\theta}$, condition (c) in the theorem may be written as
$$
\Pro_{\theta}\Bigg(\prod_{\ell\neq k}\frac{d\tnu_\ell}{d\nu_\ell}(X_{\ell,1})>0\Bigg)>0.
$$
Since $\left\{\prod_{\ell\neq k}\frac{d\tnu_\ell}{d\nu_\ell}(X_{\ell,1})>0\right\}=\cup_{m\geq 2}\left\{\prod_{\ell\neq k}\frac{d\tnu_\ell}{d\nu_\ell}(X_{\ell,1})\geq \frac{1}{m} \right\}$, this readily implies that

$$\exists \eta \in (0,1), \ \Pro_{\theta}\Bigg(\prod_{\ell\neq k}\frac{d\tnu_\ell}{d\nu_\ell}(X_{\ell,1})\geq \eta\Bigg)>0.
$$
Let $a=\Pro_{\theta}\left(\prod_{\ell\neq k}\frac{d\tnu_\ell}{d\nu_\ell}(X_{\ell,1})\geq \eta\right)$.

Let us take $n$ large enough such that $N=\left\lfloor 4C\log n \right\rfloor$  satisfies
$N<n$, $C\log N <\frac{N}{2K}$ and
$f(n) \eta^t\big(a^t-\frac{(K-1)\tilde{C}}{f(N)}\big)>\tC'$ for $t=\left\lfloor C\log N\right\rfloor$. 
For any $\tC'$, such a $n$ does exist since $f \gg_{+\infty} \log^\alpha$ for any $\alpha>0$.

The idea is that if until round $N$, arms $\ell\neq k$ have a behaviour that is typical of $\theta$, then the arm $k$ (which is suboptimal in $\tth$) may be pulled about $C\log n$ times at round $N$.
Precisely, we prove that 
$\forall \ell\neq k, \ \Pro_{\theta}\big(T_\ell(N)\geq C\log N\big)\leq \frac{\tilde{C}}{f(N)}$
implies 
$\Pro_{\tth}\big(T_k(n)\geq C'\log n\big)> \frac{\tilde{C}'}{f(n)}$.
Let us denote
$A_t=\cap_{s=1\hdots t}\left\{\prod_{\ell\neq k}\frac{d\tnu_\ell}{d\nu_\ell}(X_{\ell,s})\geq \eta\right\}$. By independence and by definition of $a$, we have $\P_{\theta}(A_t)=a^t$. We also have
\begin{eqnarray}
\Pro_{\tilde{\theta}}\big(T_k(n)\geq C\log n\big)&\geq&\Pro_{\tilde{\theta}}\bigg(T_k(N)\geq \frac{N}{2}\bigg)\nonumber\\
&\geq&\Pro_{\tilde{\theta}}\Bigg(\bigcap_{\ell\neq k}\left\{T_\ell(N) \leq \frac{N}{2K}\right\}\Bigg)\nonumber\\
&\geq&\Pro_{\tilde{\theta}}\Bigg(\bigcap_{\ell\neq k}\bigg\{T_\ell(N) < C\log N\bigg\}\Bigg)\nonumber\\
&\geq& \Pro_{\tilde{\theta}}\Bigg(A_t\cap\Bigg\{\bigcap_{\ell\neq k}\bigg\{T_\ell(N) < C\log N\bigg\}\Bigg\}\Bigg).\nonumber 
\end{eqnarray}
Introduce $B_N=\bigcap_{\ell\neq k}\big\{T_\ell(N) < C\log N\big\}$, 
and the function $q$ such that 
$$
\IND_{A_t\cap B_N}=q\big((X_{\ell,s})_{\ell\neq k, \ s=1..t},(X_{k,s})_{s=1..N}\big).
$$
Since $\tnu_k=\nu_k$, by definition of $A_t$ and by standard properties of density functions $\frac{d\tnu_\ell}{d\nu_\ell}$, we have
\begin{eqnarray}
&&\Pro_{\tilde{\theta}}\Bigg(A_t\cap\Bigg\{\bigcap_{\ell\neq k}\left\{T_\ell(N) < C\log N\right\}\Bigg\}\Bigg)\nonumber\\
&=&\int q\big((x_{\ell,s})_{\ell\neq k, \ s=1..t},(x_{k,s})_{s=1..N}\big)\prod_{\tiny \begin{array}c \ell\neq k \nonumber\\
s=1..t\end{array}}d\tnu_\ell(x_{\ell,s})\prod_{s=1..N}d\tnu_k(x_{k,s})\nonumber\\
&\geq& \eta^t\int q\big((x_{\ell,s})_{\ell\neq k, \ s=1..t},(x_{k,s})_{s=1..N}\big)\prod_{\tiny \begin{array}c \ell\neq k \\ s=1..t\end{array}}d\nu_\ell(x_{\ell,s})\prod_{s=1..N}d\nu_k(x_{k,s})\nonumber\\
&=& \eta^t \Pro_{\theta}\Bigg(A_t\cap\Bigg\{\bigcap_{\ell\neq k}\left\{T_\ell(N) < C\log N\right\}\Bigg\}\Bigg)\nonumber\\
&\geq & \eta^t\bigg(a^t-\frac{(K-1)\tilde{C}}{f(N)} \bigg)\nonumber\\
&> & \frac{\tilde{C}'}{f(n)}, \nonumber
\end{eqnarray}
where the one before last step relies on a union bound with \eqref{eq:uptheta} and $\P_{\theta}(A_t)=a^t$,
and the last inequality uses the definition of $n$.
We have thus proved that \eqref{eq:uptth} holds, and thus the policy cannot have the $f$-upper tailed property simultaneously in environment $\theta$ and $\tth$.

\subsection{Proof of Theorem \ref{th:mainpositive}}\label{proofmainpos}
Let $\theta$ be in $\Theta$. 
Consider the event
$$\xi = \left\{\forall k\in \{1,\hdots,K\}, T \in \{1, \hdots,n\}, T\Vert\hat{F}_{k,T} - F_{\nu_k}\Vert_{\infty}^2 < \frac{\beta+1}{2}\log n \right\}.$$
From Massart's inequality (see \cite{massart1990tight}) applied $nK$ times corresponding to the different times and arms and a union bound to combine the inequalities, we have $$\P_{\theta}(\xi) \geq 1 - nK(2e^{-(\beta+1)\log n})=1-\frac{2K}{n^{\beta}}.$$

We show that on the event $\xi$, inequalities $T_k(n)\leq \frac{2(\beta+1)\log n}{\delta_k^2}+1$ hold for any $k\neq k^*$, where $\delta_k=\min_{\tilde{\theta}\in \Theta_k}\Vert F_{\nu_k}-F_{\tilde{\nu}_k} \Vert_{\infty}$.
 Note that $\delta_k>0$: if not, it would mean that $k$ is suboptimal in $\theta$ and optimal in an other environment $\tilde{\theta}$, with $\nu_k=\tilde{\nu}_k$. In this case, by hypothesis there exists $\ell\neq k$ such that $\frac{d\tilde{\nu}_\ell}{d\nu_\ell}(X_{\ell,1})=0$ $\Pro_{\theta}$-a.s. Thus $\tilde{\theta}$ is almost surely removed during the first rounds of the policy and, as $\Theta$ is finite, all of these problematic $\tilde{\theta}$ are removed almost surely. Note also that $\theta$ cannot be removed: it is readily seen that $\P_{\theta}\left(\frac{d\nu_\ell}{d\tilde{\nu}_\ell}(X_{\ell,1})>0\right)=1$ for all $\tilde{\theta}\in\Theta$ and, still because $\Theta$ is finite, it is almost sure that $\frac{d\nu_\ell}{d\tilde{\nu}_\ell}(X_{\ell,1})>0$ for all $\tilde{\theta}\in\Theta$. A last consequence of the finiteness of $\Theta$ is that terms $\delta_k$ are uniformly bounded away from zero over $\Theta$, and so are the terms $\Delta_k$, so that the inequalities we are going to prove easily lead to the conclusion of the proof.\\

Assume by contradiction that there exists $k\neq k^*$ such that $T_k(n)> \frac{2(\beta+1)\log n}{\delta_k^2}+1$. Then there exists $t\leq n$ such that $I_t=k$ and $T_k(t-1)>\frac{2(\beta+1)\log n}{\delta_k^2}.$\\
As arm $k$ is chosen at round $t$, we have:
$$
T_{k^*}(t-1)\inf_{\tth\in\Theta_{k^*}}\Vert \hat{F}_{k^*,T_k^*(t-1)}-F_{\tnu_{k^*}}\Vert_{\infty}^2\geq T_{k}(t-1)\inf_{\tth\in\Theta_{k}}\Vert \hat{F}_{k,T_k(t-1)}-F_{\tnu_{k}}\Vert_{\infty}^2
$$
On the one hand, we have:
$$
\frac{\beta+1}{2}\log n>T_{k^*}(t-1)\inf_{\tth\in\Theta_{k^*}}\Vert \hat{F}_{k^*,T_k^*(t-1)}-F_{\tnu_{k^*}}\Vert_{\infty}^2,
$$
and on the other hand

%
%

\begin{eqnarray*}
\sqrt{T_{k}(t-1)}\inf_{\tth\in\Theta_{k}}\Vert \hat{F}_{k,T_k(t-1)}-F_{\tnu_{k}}\Vert_{\infty}
&\geq& \sqrt{T_{k}(t-1)}\left(\delta_k-\Vert \hat{F}_{k,T_k(t-1)}-F_{\nu_k}\Vert_{\infty}\right)\\
&\geq& \sqrt{T_{k}(t-1)}\left(\delta_k-\sqrt{\frac{(\beta+1)\log n}{2T_k(t-1)}}\right)\\
&=& \sqrt{T_{k}(t-1)}\delta_k-\sqrt{\frac{\beta+1}{2}\log n}.
\end{eqnarray*}

By combining the former inequalities, we get:
$$
\sqrt{\frac{\beta+1}{2}\log n}> \sqrt{T_{k}(t-1)}\delta_k-\sqrt{\frac{\beta+1}{2}\log n}
$$
and
$$
T_k(t-1)<\frac{2(\beta+1)\log n}{\delta_k^2},
$$
which is the contradiction expected.

\subsection{Proof of Theorem \ref{th:bern}}
The proof of the first part of the theorem is the same as the previous section \ref{proofmainpos}, except that one has to substitute $\delta_k$ by $d_k$ and that the $d_k$ ($k\neq k^*$) are not necessarily non negative. Indeed, the distance $\Vert\hat{F}_{k,T} - F_{\nu_k}\Vert_{\infty}$ equals $|\hat X_{k,T}-\mu_k|$ in the context of Bernoulli laws.\\

The proof of the second part is similar to the one of Theorem \ref{th:main}: we assume by contradiction that there exists a policy such that
\[
\exists C,\tilde{C}>0, \forall \theta\in\Theta, \ \forall n\geq 2, \ \forall k\neq k^*, \ \Pro_{\theta}\left(T_k(n)\geq C\log n\right)\leq \frac{\tilde{C}}{f(n)}.
\]

 The main difference is that we cannot fix $\theta,\tth$ such that $\theta\in\Theta_k$, $\tth\in\Theta\smallsetminus\Theta_k$ and $\mu_k=\tilde \mu_k$. The hypothesis only allows us to take $\mu_k$ and $\tilde \mu_k$ arbitrarily close. This means that we are allowed to consider two sequences $(\theta^n)_{n\geq 1}$ and $(\tth^n)_{n\geq 1}$ such that, for all $n\geq 1$ (with obvious notations):
\begin{itemize}
\item $\theta^n\in\Theta_k,\ \tth^n\in\Theta\smallsetminus\Theta_k,$
\item $\tilde \mu_k^n \geq 2^{-\frac{1}{N}}\mu_k^n,$
\item $1-\tilde \mu_k^n \geq 2^{-\frac{1}{N}}(1-\mu_k^n),$
\end{itemize}
where $N=\left\lfloor 4C\log n \right\rfloor$.\\

On the other hand, the hypothesis readily implies that 
$$
\forall \theta,\tth\in\Theta, \ \forall \ell\in\{1,\cdots,K\}, \ \frac{d\tnu_\ell}{d\nu_\ell}(1)=\frac{\tilde \mu_l}{\mu_l}\geq \gamma
$$
and
\begin{eqnarray*}
\Pro_{\theta}\left(\prod_{\ell\neq k}\frac{d\tnu_\ell}{d\nu_\ell}(X_{\ell,1})\geq \gamma^{K-1}\right)&\geq&
\Pro_{\theta}\left(\bigcap_{\ell\neq k}\left\{\frac{d\tnu_\ell}{d\nu_\ell}(X_{\ell,1})\geq \gamma\right\}\right)=
\prod_{\ell\neq k}\Pro_{\theta}\left(\frac{d\tnu_\ell}{d\nu_\ell}(X_{\ell,1})\geq \gamma\right)\\
&\geq&\prod_{\ell\neq k}\Pro_{\theta}\left(X_{\ell,1}=1\right)=\prod_{\ell\neq k} \mu_l\geq \gamma^{K-1}.
\end{eqnarray*}
Let us denote $a=\gamma^{K-1}$ and $A_t=\bigcap_{s=1}^t\left\{\prod_{\ell\neq k}\frac{d\tnu_\ell}{d\nu_\ell}(X_{\ell,s})\geq a\right\}$. By independence, we have $\Pro_{\theta}(A_t)=a^t$.\\

To find a contradiction, we set $t=\left\lfloor C\log N\right\rfloor$ and we adapt the reasoning of the former proof.\\
If $n$ is chosen large enough, one has $N<n$ and $C\log N <\frac{N}{2K}$, and then:
\begin{eqnarray}
\Pro_{\tilde{\theta}^n}\left(T_k(n)\geq C\log n\right)&\geq&
\Pro_{\tilde{\theta}^n}\left(T_k(N)\geq \frac{N}{2}\right)\nonumber\\
&\geq&\Pro_{\tilde{\theta}^n}\left(\bigcap_{\ell\neq k}\left\{T_\ell(N) \leq \frac{N}{2K}\right\}\right)\nonumber\\
&\geq&\Pro_{\tilde{\theta}^n}\left(\bigcap_{\ell\neq k}\left\{T_\ell(N) < C\log N\right\}\right).\nonumber\\
&\geq& \Pro_{\tilde{\theta}^n}\left(A_t\cap\left\{\bigcap_{\ell\neq k}\left\{T_\ell(N) < C\log N\right\}\right\}\right).\nonumber
\end{eqnarray}
Let us denote $B_N=\bigcap_{\ell\neq k}\left\{T_\ell(N) < C\log N\right\}$. $B_N$ is measurable w.r.t. $X_{k,1},\hdots,X_{k,N}$ and $X_{\ell,1},\hdots,X_{\ell,t}$ ($\ell\neq k$), and $A_t$ is measurable w.r.t. $X_{\ell,1},\hdots,X_{\ell,t}$ ($\ell\neq k$), so that we can write
$$
\IND_{A_t\cap B_N}=c_{t,N}\left((X_{\ell,s})_{\ell\neq k, \ s=1..t},(X_{k,s})_{s=1..N}\right).
$$
By properties of $\tnu_k^n$ and $\nu_k^n$ and by definition of $A_t$ we have
\begin{eqnarray}
&&\Pro_{\tilde{\theta}^n}\left(A_t\cap\left\{\bigcap_{\ell\neq k}\left\{T_\ell(N) < C\log N\right\}\right\}\right)\nonumber\\
&=&\int c_{t,N}\left((x_{\ell,s})_{\ell\neq k, \ s=1..t},(x_{k,s})_{s=1..N}\right)\prod_{\tiny \begin{array}c \ell\neq k \\
s=1..t\end{array}}d\tnu_\ell^n(x_{\ell,s})\prod_{s=1..N}d\tnu_k^n(x_{k,s})\nonumber\\
&\geq& \int c_{t,N}\left((x_{\ell,s})_{\ell\neq k, \ s=1..t},(x_{k,s})_{s=1..N}\right)a^t\prod_{\tiny \begin{array}c \ell\neq k \\ s=1..t\end{array}}d\nu_\ell^n(x_{\ell,s})\prod_{s=1..N}\left(2^{-\frac{1}{N}}d\nu_k^n(x_{k,s})\right)\nonumber\\
&=&\frac{a^t}{2} \Pro_{\theta^n}\left(A_t\cap\left\{\bigcap_{\ell\neq k}\left\{T_\ell(N) < C\log N\right\}\right\}\right)\nonumber\\
&\geq & \frac{a^t}{2}\left(a^t-\frac{(K-1)\tilde{C}}{f(N)} \right).\nonumber
\end{eqnarray}
By straightforward calculations, one can then show that $f(n)\Pro_{\tilde{\theta}^n}\left(T_k(n)\geq C\log n\right)\xrightarrow[N\to+\infty]{}+\infty$, which is the contradiction expected.

\subsection{Proof of Theorem \ref{th:mu*}}
The proof is similar the one of Theorem \ref{th:mainpositive}, except that we use Hoeffding's inequality rather than Massart's one. Consider the event
$$\xi = \left\{\forall k\in \{1,\hdots,K\}, s \in \{1, \hdots,n\}, s(\hX_{k,s} - \mu_k)^2 < \frac{\beta+1}2\log n \right\}.$$
From Hoeffding's inequality applied $2nK$ times corresponding to the different times and arms and a union bound to combine the inequalities, we have
$\P(\xi) \geq 1-2nKe^{-(\beta+1)\log n}=1-\frac{2K}{n^{\beta}}.$ 
We will prove by contradiction that on the event $\xi$, we have $T_k(n) \le 1+\frac{2(\beta+1)\log n}{\Delta_k^2}$ for all $k\neq k^*$.
For this, consider $k\neq k^*$ such that $T_k(n) > \frac{2(\beta+1)\log n}{\Delta_k^2}+1.$
Then there exists $t\le n$ such that $I_t=k$ and $T_k(t-1) > \frac{2(\beta+1)\log n}{\Delta_k^2}$.
Since the arm $k$ is chosen at time $t$, it means that
    \begin{equation}\label{ineq}
    T_k(t-1) \big(\mu^*-\hX_{k,T_k(t-1)}\big)_+^2 \le T_{k^*}(t-1) \big(\mu^*-\hX_{k^*,T_{k^*}(t-1)}\big)_+^2.
    \end{equation}
Let us split the proof into two cases.\\

First case: $\hX_{k,T_k(t-1)}\geq \mu_*$.\\
Then $\hX_{k,T_k(t-1)}-\mu_k\geq \Delta_k$ and $T_k(t-1) \big(\hX_{k,T_k(t-1)}-\mu_k\big)^2\geq T_k(t-1)\Delta_k^2$. The contradiction readily comes from the definition of $\xi$.\\

Second case: $\hX_{k,T_k(t-1)}<\mu_*$.\\
From inequality \eqref{ineq} one has $\hX_{k^*,T_{k^*}(t-1)}<\mu^*$, and \eqref{ineq} can be written as:
$$
    T_k(t-1) \big(\hX_{k,T_k(t-1)}-\mu^*\big)^2 \le T_{k^*}(t-1) \big(\hX_{k^*,T_{k^*}(t-1)}-\mu^*\big)^2.
$$
On the one hand, we have:
$$
\frac{\beta+1}{2}\log n>T_{k^*}(t-1) \big(\hX_{k^*,T_{k^*}(t-1)}-\mu^*\big)^2,
$$
and on the other hand

\begin{eqnarray*}
\sqrt{T_{k}(t-1)}\big|\hX_{k,T_k(t-1)}-\mu^*\big|
&\geq& \sqrt{T_{k}(t-1)}\left(\Delta_k-\big|\hX_{k,T_k(t-1)}-\mu_k\big|\right)\\
&\geq& \sqrt{T_{k}(t-1)}\left(\Delta_k-\sqrt{\frac{(\beta+1)\log n}{2T_k(t-1)}}\right)\\
&=& \sqrt{T_{k}(t-1)}\Delta_k-\sqrt{\frac{\beta+1}{2}\log n}.
\end{eqnarray*}
The former inequalities leads to $$\sqrt{\frac{\beta+1}{2}\log n}>\sqrt{T_{k}(t-1)}\Delta_k-\sqrt{\frac{\beta+1}{2}\log n}\Rightarrow T_k(t-1)<\frac{2(\beta+1)\log n}{\Delta_k^2}.$$
%
%
Thus there is a contradiction, meaning that there is no
$k$ such that $T_k(n) > \frac{2(\beta+1)\log n}{\Delta_k^2}+1.$

\subsection{Proof of Theorem \ref{th:hor}}

We assume by contradiction that there exists a $f$-w$\mathcal T$ policy. As in the proof of Theorem \ref{th:main}, on can remove the $\Delta_k^2$ denominator, so that we have:
\begin{eqnarray*}
\exists C,\tilde{C}>0, \ \forall n\geq 2, \ \forall \ell\neq k, \ \Pro_{\theta}\big(T_\ell(n)\geq C\log n\big)\leq \frac{\tilde{C}}{f(n)}.
\end{eqnarray*}
Let us show that this implies that the policy cannot have also the $f$-upper tailed property in $\tth$.
To prove the latter, it is enough to show that for any $C',\tilde{C}'>0$
\begin{equation}\label{eq:uptthhor}
\exists n\geq 2, \ \Pro_{\tth}\big(T_k(n)\geq C'\log n\big)> \frac{\tilde{C}'}{f(n)},
\end{equation}
since $k$ is suboptimal in environment $\tth$.\\

Similarly to the proof of theorem \ref{th:main}, proving \eqref{eq:uptthhor} for $C'= C$ is sufficient. Moreover, there exists $\eta\in(0,1)$ such that the event $A=\Big\{\prod_{\ell=1}^K\frac{d\tnu_\ell}{d\nu_\ell}(X_{\ell,1})\geq \eta\Big\}$ has probability $a>0$ under $\Pro_{\theta}$. We denote
$A_t=\cap_{s=1\hdots t}\left\{\prod_{\ell=1}^K\frac{d\tnu_\ell}{d\nu_\ell}(X_{\ell,s})\geq \eta\right\}$, and by independence we have $\P_{\theta}(A_t)=a^t$.\\

Let us set $N=\left\lceil KC\log n \right\rceil$, choose $n$ large enough so that $n>N$, and denote $Y$ a r.v. that equals the index of an arm among those that have been pulled the most after time step $N$, e.g.\\ $Y=\min\left(\argmax_{l\in \{1,\ldots,K\}} T_l(N)\right)$. Obviously, such an arm has been pulled at least $C\log n$ at step $N$ (i.e. $T_Y(N)\geq C\log n$ a.s.), so that one has:\\

$
\Pro_{\tth}\left(T_k(n)\geq C\log n\right)\geq \Pro_{\tth}\left(T_k(N)\geq C\log n\right) \geq\Pro_{\tth}\left(Y=k\right)\geq \Pro_{\tth}\left(A_N\cap\{Y=k\}\right).
$

\bigskip

Introduce the function $q$ such that 
$$
\IND_{A_N\cap \{Y=k\}}=q\big((X_{\ell,s})_{1\leq\ell\leq K, \ s=1..N}\big).
$$
One has:

\begin{eqnarray*}
\Pro_{\tth}\left(A_N\cap\{Y=k\}\right)&=&\int q\big((x_{\ell,s})_{1\leq \ell\leq K, \ s=1..N},(x_{k,s})_{s=1..N}\big)\prod_{\tiny \begin{array}c 1\leq \ell\leq K\\
s=1..N\end{array}}d\tnu_\ell(x_{\ell,s})\\
&\geq& \eta^N \int q\big((x_{\ell,s})_{1\leq \ell\leq K, \ s=1..N},(x_{k,s})_{s=1..N}\big)\prod_{\tiny \begin{array}c 1\leq \ell\leq K\\
s=1..N\end{array}}d\nu_\ell(x_{\ell,s})\\
&=& \eta^N \Pro_{\theta}\left(A_N\cap\{Y=k\}\right)\geq \eta^N \left(\Pro_{\theta}(A_N)-\Pro_{\theta}(Y\neq k)\right)\\
&\geq&\eta^Na^N- \eta^N\Pro_{\theta}\left(\exists l\neq k, \ T_l(N)\geq C\log n  \right) \\
&\geq&\eta^Na^N- \eta^N\Pro_{\theta}\left(\exists l\neq k, \ T_l(n)\geq C\log n  \right) \\
&\geq& (\eta a)^N- \eta^N \frac{(K-1)\tilde{C}}{f(n)}.
\end{eqnarray*}

As $N$ is of order $\log n$, it is then readily seen that $f(n)\Pro_{\tth}\left(T_k(n)\geq C\log n\right)\xrightarrow[n\to +\infty]{}+\infty$, hence the result.

\bibliography{bandit_stat}

\begin{thebibliography}{23}
\providecommand{\natexlab}[1]{#1}
\providecommand{\url}[1]{\texttt{#1}}
\expandafter\ifx\csname urlstyle\endcsname\relax
  \providecommand{\doi}[1]{doi: #1}\else
  \providecommand{\doi}{doi: \begingroup \urlstyle{rm}\Url}\fi

\bibitem[Agrawal(1995)]{Agr95}
R.~Agrawal.
\newblock Sample mean based index policies with o(log n) regret for the
  multi-armed bandit problem.
\newblock \emph{Advances in Applied Mathematics}, 27:\penalty0 1054--1078,
  1995.

\bibitem[Audibert et~al.(2009)Audibert, Munos, and
  Szepesv{\'a}ri]{audibert2009exploration}
J.-Y. Audibert, R.~Munos, and C.~Szepesv{\'a}ri.
\newblock {Exploration-exploitation tradeoff using variance estimates in
  multi-armed bandits}.
\newblock \emph{Theoretical Computer Science}, 410\penalty0 (19):\penalty0
  1876--1902, 2009.

\bibitem[Auer et~al.(2002)Auer, Cesa-Bianchi, and Fischer]{AuCeFi02}
P.~Auer, N.~Cesa-Bianchi, and P.~Fischer.
\newblock Finite-time analysis of the multiarmed bandit problem.
\newblock \emph{Mach. Learn.}, 47\penalty0 (2-3):\penalty0 235--256, 2002.

\bibitem[Babaioff et~al.(2009)Babaioff, Sharma, and Slivkins]{Baba09}
M.~Babaioff, Y.~Sharma, and A.~Slivkins.
\newblock {Characterizing truthful multi-armed bandit mechanisms: extended
  abstract}.
\newblock In \emph{Proceedings of the tenth ACM conference on Electronic
  commerce}, pages 79--88. ACM, 2009.

\bibitem[Bergemann and Valimaki(2008)]{BerVal08}
D.~Bergemann and J.~Valimaki.
\newblock Bandit problems.
\newblock 2008.
\newblock {In The New Palgrave Dictionary of Economics, 2nd ed. Macmillan
  Press}.

\bibitem[Bubeck et~al.(2009)Bubeck, Munos, Stoltz, and Szepesvari]{Bub08}
S.~Bubeck, R.~Munos, G.~Stoltz, and C.~Szepesvari.
\newblock Online optimization in {X}-armed bandits.
\newblock In \emph{Advances in Neural Information Processing Systems 21}, pages
  201--208. 2009.

\bibitem[Burnetas and Katehakis(1996)]{BurKat96}
A.N. Burnetas and M.N. Katehakis.
\newblock Optimal adaptive policies for sequential allocation problems.
\newblock \emph{Advances in Applied Mathematics}, 17\penalty0 (2):\penalty0
  122--142, 1996.

\bibitem[Coquelin and Munos(2007)]{Coq07}
P.A. Coquelin and R.~Munos.
\newblock {Bandit algorithms for tree search}.
\newblock In \emph{{U}ncertainty in {A}rtificial {I}ntelligence}, 2007.

\bibitem[Devanur and Kakade(2009)]{Dev09}
N.R. Devanur and S.M. Kakade.
\newblock {The price of truthfulness for pay-per-click auctions}.
\newblock In \emph{Proceedings of the tenth ACM conference on Electronic
  commerce}, pages 99--106. ACM, 2009.

\bibitem[Garivier and Capp{\'e}(2011)]{garivier2011kl}
A.~Garivier and O.~Capp{\'e}.
\newblock The kl-ucb algorithm for bounded stochastic bandits and beyond.
\newblock \emph{Arxiv preprint arXiv:1102.2490}, 2011.

\bibitem[Gelly and Wang(2006)]{Gel06}
S.~Gelly and Y.~Wang.
\newblock {Exploration exploitation in go: UCT for Monte-Carlo go}.
\newblock In \emph{Online trading between exploration and exploitation
  Workshop, Twentieth Annual Conference on Neural Information Processing
  Systems (NIPS 2006)}, 2006.

\bibitem[Holland(1992)]{Hol92}
J.H. Holland.
\newblock \emph{{Adaptation in natural and artificial systems}}.
\newblock MIT press Cambridge, MA, 1992.

\bibitem[Honda and Takemura(2010)]{HonTak10}
J.~Honda and A.~Takemura.
\newblock An asymptotically optimal bandit algorithm for bounded support
  models.
\newblock In \emph{Proceedings of the Twenty-Third Annual Conference on
  Learning Theory (COLT)}, 2010.

\bibitem[Kleinberg et~al.(2008)Kleinberg, Slivkins, and Upfal]{Kle08}
R.~Kleinberg, A.~Slivkins, and E.~Upfal.
\newblock {Multi-armed bandits in metric spaces}.
\newblock In \emph{Proceedings of the 40th annual ACM symposium on Theory of
  computing}, pages 681--690, 2008.

\bibitem[{Kleinberg}(2005)]{Kle05}
R.~D. {Kleinberg}.
\newblock Nearly tight bounds for the continuum-armed bandit problem.
\newblock In \emph{Advances in Neural Information Processing Systems 17}, pages
  697--704. 2005.

\bibitem[Kocsis and Szepesv\'ari(2006)]{KocSze06}
L.~Kocsis and {Cs.} Szepesv\'ari.
\newblock Bandit based {M}onte-{C}arlo planning.
\newblock In \emph{Proceedings of the 17th European Conference on Machine
  Learning ({ECML}-2006)}, pages 282--293, 2006.

\bibitem[Lai and Robbins(1985)]{LaiRo85}
T.~L. Lai and H.~Robbins.
\newblock Asymptotically efficient adaptive allocation rules.
\newblock \emph{Advances in Applied Mathematics}, 6:\penalty0 4--22, 1985.

\bibitem[Lamberton et~al.(2004)Lamberton, Pag{\`e}s, and
  Tarr{\`e}s]{LamPagTar04}
D.~Lamberton, G.~Pag{\`e}s, and P.~Tarr{\`e}s.
\newblock {When can the two-armed bandit algorithm be trusted?}
\newblock \emph{Annals of Applied Probability}, 14\penalty0 (3):\penalty0
  1424--1454, 2004.

\bibitem[Maillard et~al.(2011)Maillard, Munos, and Stoltz]{maillard2011finite}
O.A. Maillard, R.~Munos, and G.~Stoltz.
\newblock A finite-time analysis of multi-armed bandits problems with
  kullback-leibler divergences.
\newblock \emph{Arxiv preprint arXiv:1105.5820}, 2011.

\bibitem[Massart(1990)]{massart1990tight}
P.~Massart.
\newblock {The tight constant in the Dvoretzky-Kiefer-Wolfowitz inequality}.
\newblock \emph{The Annals of Probability}, 18\penalty0 (3):\penalty0
  1269--1283, 1990.

\bibitem[Robbins(1952)]{ro52}
H.~Robbins.
\newblock Some aspects of the sequential design of experiments.
\newblock \emph{Bulletin of the American Mathematics Society}, 58:\penalty0
  527--535, 1952.

\bibitem[Rudin(1986)]{rudin}
W.~Rudin.
\newblock \emph{Real and complex analysis (3rd)}.
\newblock New York: McGraw-Hill Inc, 1986.

\bibitem[Sutton and Barto(1998)]{Sut98}
R.~S. Sutton and A.~G. Barto.
\newblock \emph{Reinforcement Learning: An Introduction}.
\newblock MIT Press, 1998.

\end{thebibliography}

\end{document}